\documentclass{article}
\ifdefined\pdfminorversion\pdfminorversion=7\fi

\PassOptionsToPackage{numbers, sort&compress}{natbib} 
\usepackage[main,final]{neurips_2026}

\usepackage[utf8]{inputenc} 
\usepackage[T1]{fontenc}    
\usepackage{hyperref}       
\hypersetup{hypertexnames=false}
\usepackage{url}            
\usepackage{booktabs}       
\usepackage{amsfonts}       
\usepackage{nicefrac}       
\usepackage{microtype}      
\usepackage{xcolor}         
\usepackage{tikz}

\usepackage{booktabs}
\usepackage{amsfonts}       
\usepackage{amsmath, amssymb, amsthm} 
\usepackage{graphicx}       
\usepackage{subcaption}     
\usepackage{enumitem}

\usepackage{placeins}
\usepackage{longtable,array}

\newtheorem{theorem}{Theorem}     
\newtheorem{lemma}{Lemma} 

\usepackage[most]{tcolorbox}
\usepackage{enumitem}  
\definecolor{contribbar}{RGB}{15, 76, 92}       
\definecolor{contribbg}{RGB}{240, 247, 249}     
\newtcolorbox{contributionbox}{
    enhanced,
    breakable,
    colback=contribbg,
    boxrule=0pt,
    sharp corners,
    borderline west={2.5pt}{0pt}{contribbar},
    left=12pt, right=12pt, top=8pt, bottom=8pt,
    before skip=6pt, after skip=6pt,
}

\title{Phase Kernel Lifts Capacity of\\Dense Associative Memory}

\author{
  \makebox[0pt][c]{Yifei Zhao$^{1}$
  \quad Ying Tang$^{1,2,3,4}$\thanks{Corresponding author: \texttt{jamestang23@gmail.com}}} \\
  \\
  \makebox[0pt][c]{$^1$Institute of Fundamental and Frontier Sciences,} \\
  \makebox[0pt][c]{University of Electronic Science and Technology of China, Chengdu 611731, China} \\
  \makebox[0pt][c]{$^2$School of Physics, University of Electronic Science and Technology of China, Chengdu 611731, China} \\
  \makebox[0pt][c]{$^3$Key Laboratory of Quantum Physics and Photonic Quantum Information, Ministry of Education,} \\
  \makebox[0pt][c]{University of Electronic Science and Technology of China, Chengdu 611731, China} \\
  \makebox[0pt][c]{$^4$Non-classical Information Science Basic Discipline Research Center of Sichuan Province,} \\
  \makebox[0pt][c]{University of Electronic Science and Technology of China, Chengdu 611731, China}
}

\begin{document}

\addtocontents{toc}{\protect\setcounter{tocdepth}{-1}}

\maketitle

\begin{abstract}
\looseness=-1 Dense associative memory underlies modern attention mechanisms; however, Modern Hopfield Networks (MHNs) with linear kernels yield overlapping attractor basins and spurious states under high memory load, limiting their theoretically exponential storage capacity in practice. We construct a class of nonlinear phase kernels that modify the energy landscape and retrieval dynamics, and prove that these kernels preserve exponential storage capacity and local fixed-point convergence under given conditions on kernel smoothness, memory separation, and memory distributions. The framework is instantiated by the Sigmoidal Phase Hopfield Network (S-Hop), which guarantees monotone energy descent and mitigates practical capacity loss. S-Hop has $8.3\times$ and $4\times$ gains in one-step noisy retrieval load over Uniform Hopfield Networks (U-Hop) on CIFAR-10 and TinyImageNet respectively, and reaches a $50\times$ gain in first-failure load for ImageNet. The results demonstrate that phase kernels provide a principled route to robust dense associative memory
\footnote{Code is available at \url{https://github.com/Machine-learning-and-complex-systems/S-Hop}.}.
\end{abstract}

\section{Introduction}

The MHN framework~\cite{ramsauer2021hopfield} provides a rigorous energy framework
with provable capacity~\cite{lucibello2024exponential} and a direct
mathematical equivalence to self-attention, yet most existing variants
remain bound by linear or polynomial kernels that produce a
dot-product energy surface in which increasing memory load can reduce
pattern separation~\cite{hu2023sparse,jin2005pattern, santos2026sparse}.
This geometric limitation causes narrow, overlapping attractor basins
and spurious states under high memory load, motivating kernel designs
that improve separation while retaining the MHN energy structure. 

Continuous oscillator
networks~\cite{nishikawa2004capacity}
offer a parallel paradigm in which stable retrieval emerges from smooth
phase dynamics, and their success serves as the motivational starting
point for our construction.
Any modified kernel attempting to escape the geometric bottleneck of
linear MHN must simultaneously preserve Lyapunov stability and
exponential capacity, neither of which follows from a nonlinear feature substitution alone.
Closing this gap therefore requires a structural condition on the kernel
that is verifiable in advance and supports convergence and capacity analysis for a family of nonlinear embeddings, with energy descent verified for each instantiation.

Here, we propose a general theoretical framework that introduces a
family of nonlinear phase kernels into MHN through separable strong convexity of the phase
potential and a kernel contraction condition. Under the stated assumptions, kernels in this family preserve exponential storage capacity and local fixed-point convergence. The induced Hessian metric controls local conditioning, while feature scale and separation also enter the capacity bound. S-Hop instantiates this framework with a coupled hyperbolic mapping and admits a concave-convex procedure (CCCP) with monotone energy descent.
Figure~\ref{fig:energy_landscape} illustrates the energy-shaping motivation; the analysis establishes disjoint local attraction regions under a margin condition.

We evaluate S-Hop on MNIST, CIFAR-10, TinyImageNet, and ImageNet using one noisy retrieval update and a mean squared error (MSE) threshold of $0.1$.
On CIFAR-10 at $N=500$, S-Hop achieves $92.07\%$ mean success compared with $48.40\%$ for U-Hop's public \emph{avg} objective, a 43.67-percentage-point gain over three seeds.
On ImageNet, S-Hop passes the $90\%$ success criterion through $N=500$ at both $64\times64$ and $224\times224$ resolution (Figure~\ref{fig:capacity_bar}).
Scale and component controls identify contributions from feature scale and the energy-derived state update (Appendices~\ref{subsec:cr-scale} and~\ref{subsec:cr-components}).
Figure~\ref{fig:image-retrieval-examples} provides a visual illustration of image retrieval from noisy queries.




\begin{figure}[htbp]
    \centering
    \makebox[\linewidth][c]{%
    \begin{tikzpicture}
        \node[anchor=north west, inner sep=0] (img) at (0,0)
            {\includegraphics[width=\linewidth]{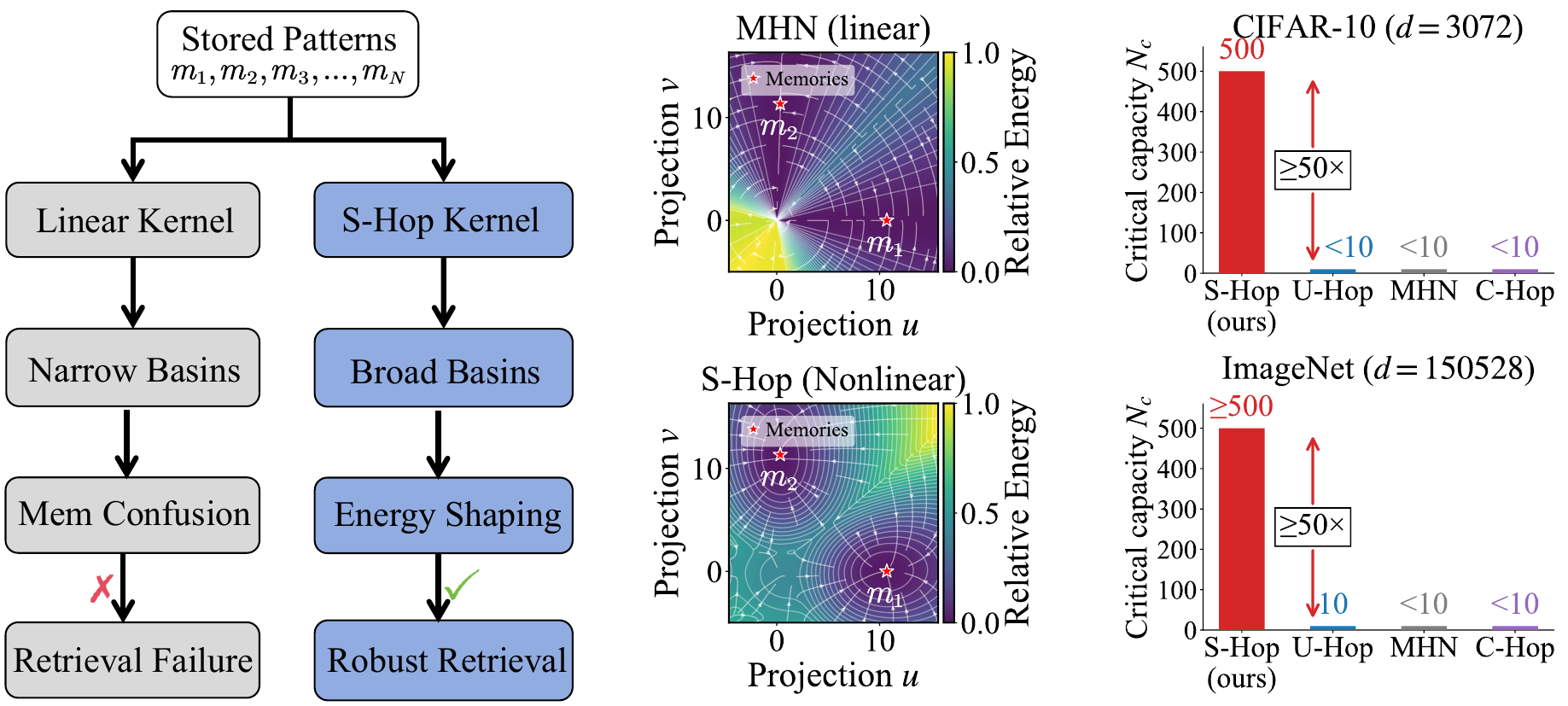}};
        \begin{scope}[x={(img.north east)}, y={(img.south west)}]
            \node[font=\small] at (0.01, 0.01) {(a)};
            \node[font=\small] at (0.39, 0.01) {(b)};
            \node[font=\small] at (0.68, 0.01) {(c)};
        \end{scope}
    \end{tikzpicture}
    }
    \caption{Framework, energy landscape, and capacity gains of S-Hop.
    (a)~Schematic of the energy-shaping motivation for stored patterns. 
    (b)~Illustrative 2D energy contours around $m_1, m_2$, independently scaled to $[0,1]$ for each kernel.
    (c)~Retrieval-load summaries on CIFAR-10 ($d=3072$) and ImageNet ($224\times224$, $d=150528$) for S-Hop, U-Hop~\cite{pmlr-v235-wu24i}, MHN~\cite{ramsauer2021hopfield}, and Cosine Hopfield Network (C-Hop). S-Hop achieves passing-load ratios of over $50\times$ than the previous methods. 
    }
    \label{fig:energy_landscape}
\end{figure}

Our main contributions are:
\begin{contributionbox}
\begin{itemize}[leftmargin=*, itemsep=4pt, topsep=0pt]
    \item We introduce a class of \textbf{nonlinear phase kernels} for MHN. Their phase potential yields a component-wise inverse local field and an energy-derived retrieval update.
    \item \looseness=-1 We establish exponential storage capacity and local fixed-point convergence under explicit kernel and memory conditions. S-Hop also has \textbf{monotone energy descent} via CCCP.
    \item We quantify retrieval gains on MNIST, CIFAR-10, TinyImageNet, and ImageNet (Table~\ref{tab:main}). S-Hop supports \textbf{more than 500 memories} on MNIST, CIFAR-10, and ImageNet. 
\end{itemize}
\end{contributionbox}

\begin{figure}[!t]
\centering
\includegraphics[width=\linewidth]{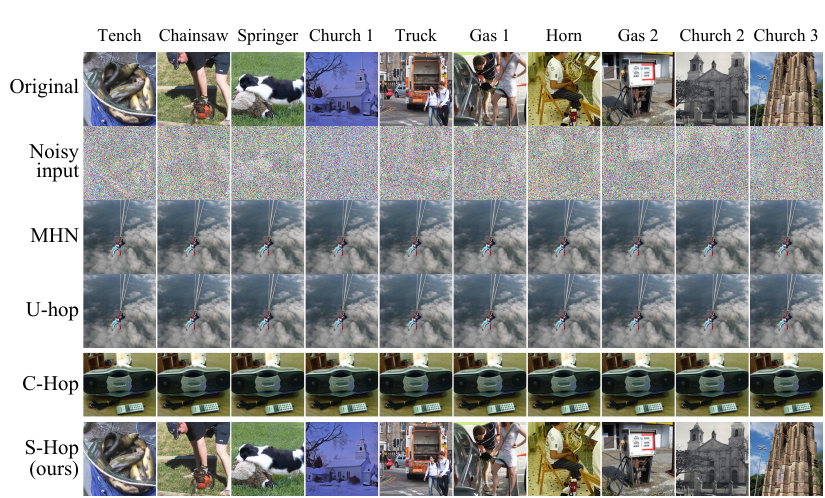}
\caption{Illustrative image retrieval from a shared memory bank of 500 ImageNet images. 
Each column shows a target image, its noisy query, and retrieval results after 20 updates ($\beta=1$) without PCA or normalization. 
Queries are generated as $q=\operatorname{clip}(|x+\epsilon|,0,1)$ with $\epsilon\sim\mathcal{N}(0,I)$. 
S-Hop successfully retrieves target images, while MHN, U-Hop and C-Hop converge to other stored images.}
\label{fig:image-retrieval-examples}
\end{figure}

\section{Related Work}
\label{sec:related_work}

\medskip \noindent\textbf{Modern Hopfield Networks and Attention} Associative memory models have long served as a foundational paradigm for understanding how
complex systems store and retrieve structured information \cite{hu2015associative, hopfield1982neural, krotov2025modern, noguer2024hopfield, krotov2023frontier}.
The introduction of MHN significantly advanced this field by demonstrating exponential memory
capacity through polynomial and exponential interaction functions
\cite{krotov2016dense, krotov2021large, demircigil2017model, hoover2025epanechnikov, chaudhry2023long},
with Lucibello and M\'{e}zard \cite{lucibello2024exponential, negri2023storage} subsequently establishing tight
asymptotic capacity bounds for this family \cite{hu2024provably}.
Notably, continuous-state MHN exhibits a direct mathematical equivalence with the self-attention
mechanism in Transformer architectures \cite{ramsauer2021hopfield, hoover2023energy, smart2025incontext, hoier2026training, gladstone2026energy} and with generative diffusion models via shared energy formulations \cite{ambrogioni2024dispersed, rodriguez2026generative, pham2025memorization}, a correspondence further deepened by \citet{rende2024mapping}, who show that a single self-attention layer learns the conditionals of a generalized Potts model—a direct statistical-mechanics analogue of Hopfield energies. Santos et al.\ \cite{santos2025hopfield} further generalized this connection by formulating
Hopfield-Fenchel-Young networks, revealing that sparsity and loss margins jointly govern
exact single-pattern retrieval.
Together, these results confirm that capacity and convergence in MHN are highly sensitive to
the choice of interaction function---a sensitivity our nonlinear phase kernel framework
directly exploits. Context-based retrieval \cite{choraria2026context} complements this approach by conditioning recall on auxiliary query signals.

\medskip \noindent\textbf{Limitations of Linear Kernels} Despite the widespread success of continuous MHN, most existing variants rely on linear or
polynomial kernel functions for pattern embedding.
As demonstrated by Jin and Zhao \cite{jin2005pattern} and Hu et al.\ \cite{hu2023sparse},
linear mappings can exhibit overlapping attractor basins at high memory loads
in high-dimensional memory spaces, causing memory confusion and spurious states under load. The role of low-dimensional representations in shaping these energy landscapes is studied in \cite{lipshutz2026low}, providing theoretical grounding for this failure mode.
Recent biologically plausible approaches \cite{shafiei2026biologically, kozachkov2025neuron} escape this bottleneck through a two-layer architecture with threshold nonlinearity,
achieving capacity exponential in the number of hidden units via distributed representations.
These models operate in a discrete two-layer setting without a continuous-state energy
landscape or continuous-state convergence guarantees---the regime targeted by our framework.

\medskip \noindent\textbf{Uniform Retrieval, Oscillatory Dynamics, and Distance-Based Variants} U-Hop~\cite{pmlr-v235-wu24i} improves the capacity ratio $\alpha$
by repositioning memory embeddings in a kernelized feature space, yet
uses a learned linear projection, whereas S-Hop changes both the embedding
and the phase-space retrieval update.
Oscillator-based and input-driven dynamical
models~\cite{betteti2025input,hoover2024dense,miyato2025artificial,nishikawa2004oscillatory,nishikawa2004capacity,wang2024training}
leverage continuous phase dynamics for stable retrieval, and recent
work~\cite{guo2025oscillatory,ogranovich2026oscillator} proves exponential
capacity in honeycomb Kuramoto networks, but these paradigms operate outside
the attention-equivalent energy framework.
A parallel line of work~\cite{hoover2024dense,hoover2025epanechnikov}
reshapes basin geometry through distance-based overlaps rather than feature
embeddings, achieving strong capacity gains via a memory-norm penalty
that operates on an axis orthogonal to ours. Curved neural networks with higher-order interactions have also been introduced to enhance memory capacity and accelerate memory retrieval ~\cite{aguilera2025explosive}. 
Our framework develops this complementary direction through separable, strongly convex phase potentials and explicit contraction conditions. These conditions support a family of nonlinear embeddings with exponential capacity and local convergence, while S-Hop additionally admits monotone energy descent.

\section{Method}
\label{sec:method}

In this section, we formalize the proposed \emph{family} of nonlinear phase kernels. We construct a surrogate Hopfield energy defined in phase space and derive tractable CCCP fixed-point dynamics. Here, $d$ is the phase-space dimension and $N$ the number of stored patterns.
\subsection{Nonlinear Phase Kernel and Memory Patterns}
\label{subsec:kernel}

We work with phase variables $\varphi \in \mathbb{R}^d$, motivated by continuous oscillator dynamics. Each memory pattern is represented by a phase vector $\psi_i \in \mathbb{R}^d$ for $i = 1, \dots, N$. We embed phases into an MHN feature space through a nonlinear map:
\begin{equation}
  F : \mathbb{R}^d \to \mathbb{R}^D,
  \qquad
  \xi = F(\varphi),
  \label{eq:F-def}
\end{equation}
and denote the embedded memories by
\begin{equation}
  x_i := F(\psi_i) \in \mathbb{R}^D,
  \qquad
  X := [x_1, \dots, x_N] \in \mathbb{R}^{D \times N}.
\label{eq:embedded-memories}
\end{equation}

Our theoretical framework applies to a class of kernels $F$ satisfying the smoothness and derivative bounds (Appendix~\ref{sec:appendix-theory}) and a structural condition that enables a tractable implicit step in phase space (Section~\ref{subsec:surrogate-energy}). As a concrete instantiation, we formulate the S-Hop architecture using the kernel in \eqref{eq:sqrt-kernel}.
\begin{equation}
\scalebox{0.9}{$
  F(\varphi)
  =
  \Big[
    \sqrt{\varphi_1^2 + \tanh \varphi_1 + 1},
    \sqrt{\varphi_1^2 + 1},
    \dots,
    \sqrt{\varphi_d^2 + \tanh \varphi_d + 1},
    \sqrt{\varphi_d^2 + 1}
  \Big]^\top
  \in \mathbb{R}^D.$}
  \label{eq:sqrt-kernel}
\end{equation}
Here, $D = 2d$, and the stored patterns are
\begin{equation}
\scalebox{0.9}{$
  x_i
  =
  F(\psi_i)
  =
  \Big[
    \sqrt{\psi_{i,1}^2 + \tanh \psi_{i,1} + 1},
    \sqrt{\psi_{i,1}^2 + 1},
    \dots,
    \sqrt{\psi_{i,d}^2 + \tanh \psi_{i,d} + 1},
    \sqrt{\psi_{i,d}^2 + 1}
  \Big]^\top.$}
  \label{eq:x-i-sqrt}
\end{equation}
This map is smooth and locally Lipschitz on any compact phase region, with a Jacobian sharing the same square-root structure used in our theoretical derivations. The analysis in Appendix~\ref{sec:appendix-theory} applies to kernels satisfying the stated structural and regularity conditions, of which \eqref{eq:sqrt-kernel} is one instance.

\subsection{Hopfield Energy in Feature Space}
\label{subsec:feature-energy}

Given a feature vector $\xi \in \mathbb{R}^D$ and memory matrix $X$, the standard MHN energy in feature space is
\begin{equation}
  E_{\mathrm{MHN}}(\xi)
  :=
  \frac{1}{2} \|\xi\|_2^2
  +
  \frac{1}{\beta}
  \Bigl(
    \log N + \frac{D}{2}
  \Bigr)
  -
  \frac{1}{\beta}\operatorname{lse}(\beta X^\top \xi),
  \label{eq:EMHN}
\end{equation}
where $\beta > 0$ is an inverse-temperature parameter and $\operatorname{lse}(z) = \log \sum_i e^{z_i}$ denotes the log-sum-exp \cite{ramsauer2021hopfield}. The logits $X^\top \xi \in \mathbb{R}^N$ induce softmax weights
\begin{equation}
  s(\xi) := \mathrm{softmax}(\beta X^\top \xi)
  \in \Delta_{N-1},
  \label{eq:softmax}
\end{equation}
where $\Delta_{N-1}$ is the $(N{-}1)$-dimensional probability simplex. The constant term in \eqref{eq:EMHN} does not affect the dynamics but simplifies parts of the theoretical analysis. The pullback through $\xi=F(\varphi)$ introduces the Jacobian $J_F$ into the phase-space update; the kernel structure below enables a component-wise implicit update.
\subsection{Surrogate Energy in Phase Space}
\label{subsec:surrogate-energy}

Directly pulling back $E_{\mathrm{MHN}}(\xi)$ through $\xi = F(\varphi)$ yields
\begin{equation}
  \tilde{E}(\varphi)
  :=
  \frac{1}{2} F(\varphi)^\top F(\varphi)
  -
  \frac{1}{\beta}\operatorname{lse}\!\left(\beta X^\top F(\varphi)\right).
  \label{eq:pullback-energy}
\end{equation}
For a general embedding, its regularizer gradient $J_F(\varphi)^\top F(\varphi)$ can couple phase coordinates. Our kernel family admits the component-wise representation:
\begin{equation}
  \frac{1}{2} F(\varphi)^\top F(\varphi)
  =
  \sum_{k=1}^d \int_0^{\varphi_k} h(u)\,\mathrm{d}u
  + C,
  \label{eq:quadratic-equals-integral}
\end{equation}
for some constant $C$ and a scalar local field $h:\mathbb{R}\to\mathbb{R}$ acting component-wise. This identity decouples the implicit step across coordinates and makes the phase-space update transparent. For the S-Hop kernel, a direct calculation yields
\begin{equation}
  h(\varphi_k)
  =
  2\varphi_k + \frac{1}{2}\operatorname{sech}^2(\varphi_k),
  \qquad k = 1,\dots,d,
  \label{eq:h-def}
\end{equation}
which is smooth and globally strictly monotone on $\mathbb{R}$, hence admits a component-wise inverse $h^{-1}$. The local regularizer is
\begin{equation}
  E_1(\varphi) := \sum_{k=1}^d \int_0^{\varphi_k} h(u)\,\mathrm{d}u,
  \qquad
  \nabla_\varphi E_1(\varphi) = h(\varphi),
  \label{eq:E1-grad}
\end{equation}
with $h(\varphi)$ understood component-wise. Appendix~\ref{sec:appendix-theory} (Lemma~\ref{lemma:strong-convexity}) shows that the quadratic kernel term
\begin{equation}
Q(\varphi) \;=\;
\tfrac12 \,\|F(\varphi)\|_2^2 
\end{equation}
is globally $\mu$-strongly convex, so $h = \nabla Q$ is globally $\mu$-strongly monotone and $L_Q$-Lipschitz, supporting the contraction analysis in Appendix~\ref{app:capacity-analysis}. The MHN log-sum-exp interaction is
\begin{equation}
  E_2(\varphi)
  := \frac{1}{\beta}
     \log \sum_{i=1}^N \exp\!\bigl(\beta x_i^\top \xi\bigr),
  \qquad \xi = F(\varphi).
\label{eq:E2-def}
\end{equation}
Applying the chain rule gives \cite{santos2025hopfield}
\begin{equation}
  \nabla_\varphi E_2(\varphi)
  = R(\varphi)\, X s(\xi),
  \label{eq:E2-grad}
  \qquad
  s(\xi) = \mathrm{softmax}\!\bigl(\beta X^\top \xi\bigr),
\end{equation}
where $R(\varphi) = J_F(\varphi)^\top$. The surrogate phase-space energy is $E(\varphi) := E_1(\varphi) - E_2(\varphi)$. 

For the S-Hop kernel, both $E_1$ and $E_2$ are globally strictly convex on $\mathbb{R}^d$ 
(Appendix A), so $E$ admits a strict difference-of-convex decomposition and the iteration in the next subsection is a CCCP with monotone energy descent.\footnote{Geometrically, 
the separable strong convexity of $Q$ induces a flat Hessian manifold on phase space with a 
Legendre isometry to its dual; the condition number $\kappa = L_Q/\mu$ controls the 
metric distortion and bounds the inverse-field sensitivity. See Appendix~A.2 
for the Hessian-geometric viewpoint.}


\subsection{CCCP Fixed-Point Dynamics}

The gradient of the surrogate energy has the form
\begin{equation}
  \nabla_\varphi E(\varphi)
  = h(\varphi) - R(\varphi) X s(\xi),
  \qquad \xi = F(\varphi),
  \label{eq:E-grad}
\end{equation}
so stationary points satisfy
\begin{equation}
  h(\varphi^\star) = R(\varphi^\star) X s(\xi^\star),
  \qquad \xi^\star = F(\varphi^\star).
\label{eq:stationary}
\end{equation}
Treating $E_1$ implicitly and $E_2$ explicitly, the CCCP update for the current phase $\varphi_t$ with $\xi_t = F(\varphi_t)$ is defined by
\begin{equation}
  \nabla_\varphi E_1(\varphi_{t+1})
  =  \nabla_\varphi E_2(\varphi_t)
  = R(\varphi_t) X s(\xi_t).
\label{eq:cccp-update-phi}
\end{equation}
Since $h$ is globally strictly monotone, the implicit relation can be solved independently for each coordinate:
\begin{equation}
  \varphi_{t+1}
  = h^{-1}\!\bigl(R(\varphi_t) X s(\xi_t)\bigr),
  \qquad
  \xi_{t+1} = F(\varphi_{t+1}).
\label{eq:phi-update}
\end{equation}
For S-Hop, whose $F$ is injective, this induces a feature-space dynamical system,
\begin{equation}
  \xi_{t+1} = T(\xi_t)
  := (F \circ h^{-1})\!\bigl(R(\varphi_t) X s(\xi_t)\bigr),
  \label{eq:T-def}
\end{equation}
which preserves the MHN interaction structure and, for S-Hop, satisfies the monotone descent $E(\varphi_{t+1}) \le E(\varphi_t)$ with equality only at stationary points (Appendix~\ref{sec:appendix-theory}).
\subsection{Theoretical Guarantees }
\label{subsec:theory-summary}

Appendices~\ref{sec:appendix-theory} and~\ref{app:capacity-analysis} analyze the dynamics \eqref{eq:T-def}. Under the kernel smoothness and contraction conditions, and sufficient separation of the embedded patterns $\{x_i\}_{i=1}^N$, we establish:

\begin{itemize}
  \item \textbf{Local convergence.} Around each stored phase $\psi_i$ there exists an invariant ball $S_i$ on which the phase-space update is a strict contraction. Each ball contains a unique fixed point; composing the convergent phase iterates with $F$ gives convergence in feature space.
  \item \textbf{Correct retrieval and disjoint basins.} If the similarity gaps $\Delta_i = \min_{j \ne i} (x_i^\top x_i - x_i^\top x_j)$ between each memory and its closest competitor satisfy the deterministic lower bound \eqref{eq:Delta-min-cond}, the balls $\{S_i\}$ are pairwise disjoint, so each stored pattern $x_i$ has a dedicated basin.
  \item \textbf{Exponential storage capacity.} For independent phase coordinates uniformly distributed on $[-1,1]$, the embedded memories $x_i=F(\psi_i)$ satisfy the storage condition with probability at least $1-\delta$ for a failure bound $\delta\in(0,1)$ whenever
  \begin{equation}
    2\le N\le \exp(\gamma d)=\exp(\gamma D/2),
    \label{eq:capacity-main}
  \end{equation}
  for sufficiently large $d$, where $\gamma>0$ depends on the kernel and $\beta$ (Theorem~\ref{thm:capacity}). Storage thus has the same exponential order in $D$ as standard MHN~\cite{lucibello2024exponential}.
\end{itemize}

\section{Experiments}
\label{sec:results}

We evaluate the proposed nonlinear phase kernel framework, concretely instantiated
as the S-Hop architecture, along five axes:
(i) memory retrieval accuracy under increasing load,
(ii) critical capacity measured by an MSE criterion,
(iii) ablation of kernel and preprocessing,
(iv) computational efficiency of the proposed kernel, and
(v) generalization in a downstream CIFAR-10 classification task; full training curves are provided in Appendix ~\ref{sec:appendix-experiment}. 

\looseness=-1 The implementation builds on the U-Hop codebase~\cite{pmlr-v235-wu24i}. Unless otherwise stated, retrieval uses one update with $\beta=1$, Gaussian query noise with standard deviation $\sigma=0.3$, and success criterion $\mathrm{MSE}<0.1$. ``Perfect retrieval'' in the plots denotes success under this MSE threshold. We define the noisy critical retrieval load $N_c$ as the largest tested $N$ with over $90\%$ success; a passing sweep endpoint gives a lower bound. Table~\ref{tab:main} summarizes the loads. Our primary U-Hop comparison uses its public \emph{avg} objective (Section~\ref{cos-kernel}); Figures~\ref{fig:capacity_main}--\ref{fig:capacity_bar} use this objective, with the \emph{max} comparison retained in Table~\ref{tab:main}. Appendix~\ref{sec:camera-ready-controls} provides additional controls, and Appendix~\ref{sec:appendix-experiment} evaluates a learned gating wrapper~\cite{qiu2025gated}.

\begin{table}[t]
\centering
\caption{Retrieval loads at $90\%$ MSE${}<0.1$ success: one update, $\sigma=0.3$. Public \emph{avg} and ImageNet entries use three-seed means. U-Hop uses a block projection at $224\times224$. Times amortize MNIST training over $256$ queries at $N=300$ ($\dagger$: RTX 5090). TinyImageNet: Appendix~\ref{subsec:cr-datasets}.}
\label{tab:main}
\normalsize
\setlength{\tabcolsep}{6pt}
\begin{tabular}{l c c c c c}
\toprule
 & MNIST & CIFAR-10 & \multicolumn{2}{c}{ImageNet} & Time \\
\cmidrule(lr){4-5}
Method & $28\times28$ & $32\times32$ & $64\times64$ & $224\times224$ & (ms/q) \\
\midrule
MHN~\cite{ramsauer2021hopfield} & ${}<10$ & ${}<10$ & ${}<10$ & ${}<10$ & ${\approx}0.1$ \\
U-Hop (\emph{max})~\cite{pmlr-v235-wu24i} & $40$ & ${}<10$ & ${}<10$ & ${}<10$ & $10$--$40$ \\
U-Hop (public \emph{avg}) & $\geq500$ & $60$ & ${}<10$ & ${}<10$ & $1.24^{\dagger}$ \\
C-Hop & ${}<10$ & ${}<10$ & ${}<10$ & ${}<10$ & ${\approx}0.4$ \\
S-Hop (ours) & $\mathbf{\geq500}$ & $\mathbf{\geq500}$ & $\mathbf{\geq500}$ & $\mathbf{\geq500}$ & ${\approx}1.5$ \\
\bottomrule
\end{tabular}
\end{table}

\subsection{Memory Retrieval Performance and Capacity}
\label{cos-kernel}
With U-Hop's public \emph{avg} objective, a three-seed CIFAR-10 comparison at $N=500$ gives success rates of $92.07\pm0.42\%$ for S-Hop, $48.40\pm3.83\%$ for U-Hop, and $3.33\pm0.12\%$ for MHN (mean $\pm$ sample standard deviation). The mean gain over U-Hop is $43.67$ percentage points. All methods share memories, noisy queries, and one update, without principal component analysis (PCA) preprocessing~\cite{le1991eigenvalues}. U-Hop last passes at $60$; S-Hop passes through $500$, giving a tested-load ratio of at least $8.3$. Figure~\ref{fig:cr-public-avg} gives raw, normalized, and matched-logit curves; Appendix~\ref{subsec:cr-public-avg} specifies the objective and protocol.

Figure~\ref{fig:capacity_main} extends the public \emph{avg} comparison to ImageNet at $224\times224$, averaging three seeds over $N=10,20,\ldots,500$.
S-Hop maintains at least $90\%$ success throughout the sweep, reaching $98.53\pm0.12\%$ at $N=500$; U-Hop reaches $4.67\pm0.12\%$ at this load.
The error curves show lower mean reconstruction error for S-Hop as the memory load increases.
These results extend the retrieval gain to high-resolution images; normalization and matched-scale controls on CIFAR-10 isolate the contributions of feature scale and retrieval dynamics (Appendix~\ref{subsec:cr-scale}).

Further tests examine retrieval across datasets and kernel instances. MNIST ($d=784$) also gives $N_c\geq500$, and Phase Hopfield Network (P-Hop) supplies a second admissible kernel (Appendices~\ref{subsec:appendix-mnist} and~\ref{subsec:appendix-phop}). A centered dot-product control compares our embedding with a variant without the norm-penalty mechanism of~\cite{hoover2024dense} (Appendix~\ref{subsec:appendix-orthogonality}). These comparisons connect retrieval performance to kernel choice and memory geometry.

\begin{figure}[htbp]
    \centering
    \begin{tikzpicture}
        \node[anchor=north west, inner sep=0] (imga) at (0,0)
            {\includegraphics[width=0.48\linewidth]{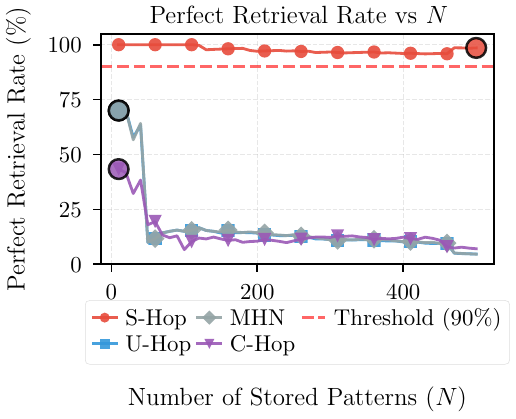}};
        \node[anchor=north west, inner sep=0] (imgb) 
            at ([xshift=4pt]imga.north east)
            {\includegraphics[width=0.47\linewidth]{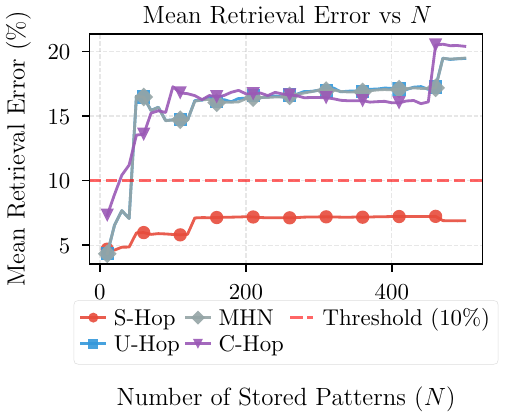}};
        \node[font=\small] at ([xshift=1pt, yshift=1pt]imga.north west) {(a)};
        \node[font=\small] at ([xshift=1pt, yshift=1pt]imgb.north west) {(b)};
    \end{tikzpicture}

    \begin{subfigure}[b]{0.48\linewidth}
        \phantomsubcaption\label{fig:capacity_a}
    \end{subfigure}\hfill
    \begin{subfigure}[b]{0.48\linewidth}
        \phantomsubcaption\label{fig:capacity_b}
    \end{subfigure}
    \caption{\looseness=-1 ImageNet retrieval at $224\times224$, averaged over three seeds with U-Hop's public \emph{avg} objective.
    (a) MSE${}<0.1$ success; outlined points mark first failures below $90\%$ or passing endpoints.
    (b) Mean retrieval error. Shared memories and queries; one update, $\beta=1$, $\sigma=0.3$, and no PCA.}
    \label{fig:capacity_main}
\end{figure}

\subsection{Critical Capacity Across Datasets}

\begin{figure}[htbp]
    \centering
    \begin{tikzpicture}
        \node[anchor=north west, inner sep=0] (imga) at (0,0)
            {\includegraphics[width=0.48\linewidth]{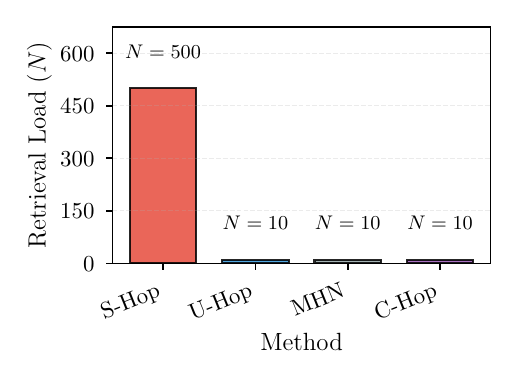}};
        \node[anchor=north west, inner sep=0] (imgb) 
            at ([xshift=4pt]imga.north east)
            {\includegraphics[width=0.48\linewidth]{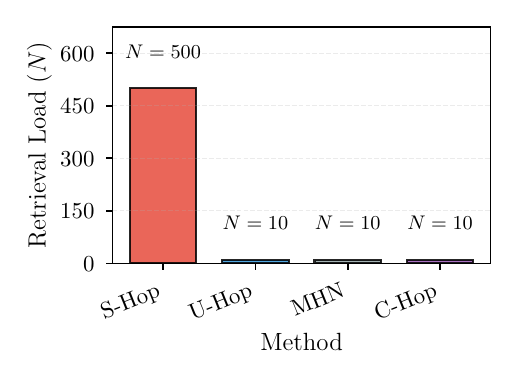}};
        \node[font=\small] at ([xshift=3pt, yshift=-3pt]imga.north west) {(a)};
        \node[font=\small] at ([xshift=3pt, yshift=-3pt]imgb.north west) {(b)};
    \end{tikzpicture}

    \begin{subfigure}[b]{0.48\linewidth}
        \phantomsubcaption\label{fig:capacity_bar_a}
    \end{subfigure}\hfill
    \begin{subfigure}[b]{0.48\linewidth}
        \phantomsubcaption\label{fig:capacity_bar_b}
    \end{subfigure}
    \caption{ImageNet retrieval-load summaries at (a) $64\times64$ and (b) $224\times224$, using three-seed mean success and U-Hop's public \emph{avg} objective. Bars mark passing endpoints for S-Hop ($N_c\geq500$) and first failures for the baselines ($N_c<10$), under the protocol of Figure~\ref{fig:capacity_main}.}
    \label{fig:capacity_bar}
\end{figure}


Table~\ref{tab:main} summarizes retrieval loads on MNIST, CIFAR-10, and ImageNet.
On MNIST, both S-Hop and U-Hop (public \emph{avg}) pass through $N=500$, with U-Hop achieving $92.93\pm1.17\%$ success at this load.
On CIFAR-10, the normalized-logit sweep extends S-Hop's passing load to $5000$, compared with $870$ for MHN and $880$ for U-Hop (\emph{max}), giving tested-load ratios of approximately $5.75$ and $5.68$.
With matched logit scale, S-Hop achieves $91.47\%$ success at $N=500$, compared with $3.40\%$ for MHN and $4.40\%$ for U-Hop (Appendix~\ref{subsec:cr-scale}).
The aligned sum-of-squared-differences (SSD) criterion gives the same pattern-level success decisions as MSE${}<0.1$ (Figure~\ref{fig:appendix_ssd_curves}).

On ImageNet, S-Hop achieves $91.47\pm0.12\%$ and $98.53\pm0.12\%$ success at $N=500$ for $64\times64$ and $224\times224$ inputs, respectively (Figure~\ref{fig:capacity_bar}).
Both sweeps give $N_c\geq500$, versus $N_c<10$ for U-Hop (public \emph{avg}), a load advantage of over $50\times$.
TinyImageNet's $4\times$ gain and selected $500$-memory result are reported together in Appendix~\ref{subsec:cr-datasets}.



\FloatBarrier

\subsection{Ablation of Kernel and Preprocessing}

\begin{table}[htbp]
\centering
\caption{Ablation study and protocol checks on CIFAR-10 ($d=3072$). Loads are passing endpoints or first failures, identified in the final column. ``$\downarrow$linear'' denotes the linear-kernel ablation. U-Hop uses \emph{max} unless indicated.}
\label{tab:ablation}
\normalsize
\setlength{\tabcolsep}{5pt}
\begin{tabular}{l c l}
\toprule
Variant & Load & Key observation \\
\midrule
S-Hop (full)                 & $\mathbf{\geq500}$ & Passing sweep endpoint \\
w/o PCA                     & --- & Retrieval error increases under noise \\
P-Hop (2nd admissible kernel) & 50 & Below $90\%$ success at this load \\
$\downarrow$linear (w/o nonlinear kernel) & 10  & Below threshold at the smallest load \\
C-Hop (cosine kernel)       & 10  & Violates strong convexity \\
\midrule
\multicolumn{3}{l}{\emph{Cross-framework \& metric robustness}} \\
\midrule
S-Hop (SSD metric)          & $\geq500$ & Equivalent threshold, SSD${}<0.1d$ \\
S-Hop (legacy \emph{avg}) & $\geq500$ & Passing endpoint in this comparison \\
S-Hop + Gating              & $\geq500$ & Passing endpoint; heuristic update \\
U-Hop (legacy \emph{avg}) & 60 & $81.7\%$ success; below threshold \\
H-Hop$^{-}$ (dot-product degenerate of~\cite{hoover2024dense})
                             & 40  & Centered overlap; norm penalty removed \\
\bottomrule
\end{tabular}
\end{table}


Table~\ref{tab:ablation} summarizes the kernel and preprocessing controls. A similarity-only S-Hop variant fails the $90\%$ criterion at $N=10$, while restoring the joint Jacobian/inverse-field update recovers a curve close to full S-Hop through $N=500$. S-Hop also outperforms the learned $2d$ U-Hop embedding (Figure~\ref{fig:cr-components}). Figure~\ref{fig:cr-pca} presents the preprocessing and noise controls, and Figures~\ref{fig:appendix_phop_capacity}--\ref{fig:appendix_phop_curves} compare the two admissible kernels.

\FloatBarrier

\subsection{Computational Efficiency Trade-Offs}
\label{app:efficiency}

\begin{figure}[!htbp]
  \centering
  \includegraphics[width=0.6\linewidth]{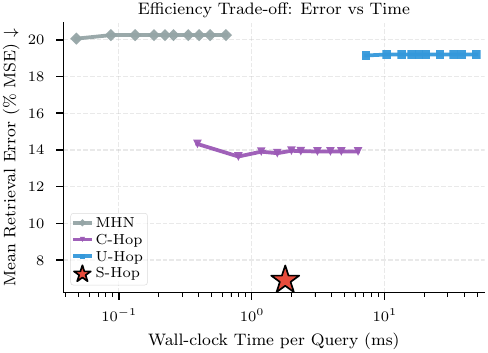} %
  \caption{Efficiency trade-off between mean retrieval error and amortized wall-clock time per query, exemplified on CIFAR-10 under the original \emph{max} protocol. U-Hop training is performed once per memory set and amortized over $256$ queries at N=300; S-Hop uses a training-free retrieval map.}
  \label{fig:efficiency}
\end{figure}

\looseness=-1 On MNIST at $N=300$, S-Hop requires approximately $1.5$ ms per query, compared with $0.1$ ms for MHN and $10$--$40$ ms for U-Hop (\emph{max}; Figure~\ref{fig:efficiency}).
S-Hop has the lowest retrieval error in this comparison and is approximately $7$--$27$ times faster than U-Hop when projection training is amortized over $256$ queries.
U-Hop trains its projection once per memory set and fixes it during retrieval; its amortized cost therefore decreases as more queries share that memory set.
These measurements place S-Hop between MHN's lower latency and U-Hop's training overhead while improving retrieval accuracy.

\FloatBarrier

\subsection{Energy-Shaped Classification on CIFAR-10}


We integrate S-Hop layers into a CIFAR-10 classifier by replacing
attention-based aggregation modules; a hybrid of early S-Hop and later MHN
layers achieves ${\approx}90\%$ test accuracy, approximately $0.2$ percentage points
above the best baseline peak in the $100$-epoch run
(Figure ~\ref{fig:cifar_main}, Appendix ~\ref{sec:appendix-experiment}).
These results extend the evaluation to supervised classification. Supplementary retrieval, gating, and multiple-instance learning (MIL) results are reported in Appendix~\ref{sec:appendix-experiment}.


\section{Discussion}
\label{sec:discussion}
The retrieval protocol specifies the inverse temperature $\beta$ and evaluation noise level $\sigma$. The original capacity experiments use $\beta=1$, $\sigma=0.3$, and one update, with success measured by MSE${}<0.1$. One update gives a shared evaluation budget, while repeated updates connect the experiments to the convergence analysis. Appendix~\ref{subsec:cr-metrics} examines how error thresholds, identification criteria, and additional updates affect retrieval.

PCA denoising reduces input variance before storage through spectral shrinkage with regularization parameter $\sigma^2_{\mathrm{PCA}}=0.01$. Our ablation shows improved retrieval at higher noise ($\sigma=0.5$) without changing the phase-kernel update (Figure~\ref{fig:appendix_ablation}). The load--error curves in Figure~\ref{fig:cr-pca} support preprocessing as a complementary source of robustness.

Memory geometry affects retrieval alongside input dimension. The TinyImageNet selection and target-removal controls link performance to memory separation and stored-memory information (Appendix~\ref{subsec:cr-datasets}). The ImageNet resolution comparison extends this observation to two input dimensions. These results motivate representations with well-separated memories and comparable feature norms.

The scale and component controls identify contributions from both feature scale and the complete retrieval map (Figures~\ref{fig:cr-normalization} and~\ref{fig:cr-components}). S-Hop and P-Hop share admissibility conditions but differ in their Jacobians, inverse local fields, and finite-dimensional performance (Appendix~\ref{subsec:phop-shop-comparison}). Extending these conditions to learned phase kernels could adapt retrieval to correlated data.

\section{Conclusion}

We introduced a nonlinear phase-kernel framework that reshapes both the MHN energy landscape and its retrieval dynamics. Separable strong convexity makes the inverse local field tractable, while explicit contraction and memory-separation conditions establish local convergence. The framework preserves exponential storage capacity for independent random phases, and S-Hop additionally guarantees monotone energy descent.

The experiments connect this construction to larger one-step noisy retrieval loads. S-Hop passes through $500$ memories on MNIST, CIFAR-10, and ImageNet at both resolutions, while U-Hop (public \emph{avg}) passes through $500$ on MNIST, last passes at $60$ on CIFAR-10, and falls below threshold at $10$ on ImageNet. TinyImageNet, scale, and component controls connect these gains to memory separation, feature scale, and the joint Jacobian/inverse-field update. These results establish phase-kernel design as a practical means of improving dense associative retrieval with an analyzable energy model.

\subsection{Limitations}
\label{sec:limitations}
The capacity theorem assumes independent, uniformly distributed phases, whereas natural images have correlated pixels and uneven feature norms. The original TinyImageNet subset contains visually similar goldfish images with lower average brightness than the selected subset and has a retrieval-load threshold of $40$. Selecting better-separated memories with comparable feature norms raises the passing load to over $500$ at the same resolution (Appendix~\ref{subsec:cr-datasets}). This improvement motivates extending the theory to correlated image distributions and evaluating independently sampled memory sets. Stricter retrieval criteria, separate Jacobian and inverse-field ablations, energy-based gating, and faster kernel evaluation provide further routes to robust and efficient associative memory (Appendix~\ref{subsec:cr-limitations}).

\begin{ack}
We thank Hong Zhao and Tailin Wu for discussions. This work is supported by Projects 12322501, 12575035 of the National Natural Science Foundation of China, and 2026NSFSCZY0124 of the Natural Science Foundation of Sichuan Province. 
\end{ack}

\bibliographystyle{unsrtnat}
\renewcommand{\bibsection}{\section*{References}}
\bibliography{references_1}

\clearpage
\appendix

\section*{Appendix}

\addtocontents{toc}{\protect\setcounter{tocdepth}{2}}

{\renewcommand{\contentsname}{}
\vspace{-2em} 
\tableofcontents}

\clearpage
\section{Theoretical Framework: Regularity, Convexity, and Energy
Descent}
\label{sec:appendix-theory}

This part develops the structural foundations of the nonlinear phase
kernel family. We introduce the notation (Appendix~\ref{sec:notation})
and the Hessian-geometric viewpoint (Appendix~\ref{subsec:kernel-family}),
verify the CCCP decoupling identity
(Appendix~\ref{subsec:cccp-verification}), and establish the regularity and
global strong convexity of the local phase potential
(Appendices~\ref{subsec:regularity}--\ref{subsec:E2-gradient}). For the
S-Hop instantiation, we additionally prove global strict convexity of
the interaction term (Appendix~\ref{subsec:E2-convexity}) and derive the
monotone CCCP energy descent guarantee (Appendix~\ref{subsec:cccp-descent},
Theorem~\ref{thm:energy-descent}). These structural results underpin the
contraction and storage-capacity analysis in
Appendix~\ref{app:capacity-analysis}.

\subsection{Notation Table}
\label{sec:notation}

Tables~\ref{tab:notation} and~\ref{tab:notation-extra} summarize the model, proof, and experimental notation.

\begingroup
\footnotesize
\renewcommand{\arraystretch}{1.12}
\setlength{\tabcolsep}{5pt}
\begin{longtable}{@{}>{\raggedright\arraybackslash}p{0.28\linewidth}>{\raggedright\arraybackslash}p{\dimexpr0.72\linewidth-10pt\relax}@{}}
\caption{Key mathematical notation.}\label{tab:notation}\\
\toprule
\textbf{Notation} & \textbf{Description} \\
\midrule
\endfirsthead
\multicolumn{2}{@{}l}{Table~\thetable\ (continued)}\\
\toprule
\textbf{Notation} & \textbf{Description} \\
\midrule
\endhead
\bottomrule
\endfoot
\multicolumn{2}{@{}l}{\textit{Dimensions and Constants}}\\*
\midrule
$d$ & Phase-space dimension. \\
$\alpha$ & Capacity ratio, $N_c/d$. \\
$D$ & Feature-space dimension; $D=2d$ for S-Hop. \\
$N$ & Number of stored memories. \\
$N_c$ & Largest tested load with $\geq90\%$ success at MSE${}<0.1$. \\
$\beta$ & Softmax inverse temperature. \\
\midrule
\multicolumn{2}{@{}l}{\textit{State Variables and Memory Patterns}}\\*
\midrule
$\varphi,\psi_i$ & Phase state and phase of memory $i$, in $\mathbb R^d$. \\
$\xi$ & Embedded state in $\mathbb R^D$. \\
$x_i$ & Embedded memory, $x_i=F(\psi_i)\in\mathbb R^D$. \\
$X$ & Memory matrix, $[x_1,\ldots,x_N]\in\mathbb R^{D\times N}$. \\
$s(\xi)$ & Attention weights, $\mathrm{softmax}(\beta X^\top\xi)$. \\
\midrule
\multicolumn{2}{@{}l}{\textit{Functions and Energy Terms}}\\*
\midrule
$F$ & Phase kernel, $F:\mathbb R^d\to\mathbb R^D$. \\
$J_F(\varphi)$ & Jacobian of $F$. \\
$R(\varphi)$ & Transposed Jacobian, $J_F(\varphi)^\top$. \\
$E_{\mathrm{MHN}}(\xi)$ & Standard MHN energy in feature space. \\
$\tilde E(\varphi)$ & Surrogate Hopfield energy in phase space. \\
$E_1(\varphi),E_2(\varphi)$ & Local regularizer and interaction energy. \\
$E(\varphi)$ & Dynamical energy, $E_1(\varphi)-E_2(\varphi)$. \\
$Q(\varphi)$ & Phase potential, $\frac12\|F(\varphi)\|_2^2$. \\
$h(\varphi)$ & Component-wise local field, $\nabla Q(\varphi)$. \\
$T$ & Feature-space fixed-point update. \\
\midrule
\multicolumn{2}{@{}l}{\textit{Theoretical Bounds and Geometry}}\\*
\midrule
$S_i$ & Invariant phase-space ball centered at $\psi_i$. \\
$\Delta_i,\Delta_{\min}$ & Memory margin, $\min_{j\ne i}x_i^\top(x_i-x_j)$; $\min_i\Delta_i$. \\
$\Delta_i^{\mathrm{eff}}(\xi)$ & State-dependent margin, $\min_{j\ne i}\xi^\top(x_i-x_j)$. \\
$\Delta_{\mathrm{store}}(N)$ & Margin threshold for correct storage. \\
$\mu,L_Q$ & Strong-convexity and gradient-Lipschitz constants of $Q$. \\
$\kappa$ & Hessian-metric condition number, $L_Q/\mu$. \\
$\rho_0$ & Kernel structural bound; admissibility requires $\rho_0<1$. \\
$L_{q''}$ & Lipschitz constant of $q''$. \\
$L_2$ & Bound on $|f_1''|$ and $|f_2''|$ over the phase region. \\
$\delta$ & Storage-failure probability bound, $0<\delta<1$. \\
$M$ & Common state and memory norm bound, Eq.~\eqref{eq:M-bound}. \\
$L_F,L_{\max}$ & Lipschitz bound for $h^{-1},F\circ h^{-1}$; uniform derivative bound. \\
$v,B$ & Scalar-feature variance; coordinate-margin bound, Eqs.~\eqref{eq:phase-variance}--\eqref{eq:phase-increment-bound}. \\
$\gamma$ & Capacity exponent in $N\leq\exp(\gamma d)$, Eq.~\eqref{eq:capacity-rate}. \\
\end{longtable}
\endgroup

\begingroup
\footnotesize
\renewcommand{\arraystretch}{1.12}
\setlength{\tabcolsep}{5pt}
\begin{longtable}{@{}>{\raggedright\arraybackslash}p{0.28\linewidth}>{\raggedright\arraybackslash}p{\dimexpr0.72\linewidth-10pt\relax}@{}}
\caption{Additional notation for geometry, proofs, and experiments.}\label{tab:notation-extra}\\
\toprule
\textbf{Notation} & \textbf{Description} \\
\midrule
\endfirsthead
\multicolumn{2}{@{}l}{Table~\thetable\ (continued)}\\
\toprule
\textbf{Notation} & \textbf{Description} \\
\midrule
\endhead
\bottomrule
\endfoot
\multicolumn{2}{@{}l}{\textit{Indices, Operators, and Kernel Geometry}}\\*
\midrule
$i,j$; $k,a,b$; $\ell$ & Memory, coordinate, and scalar-feature indices ($\ell=1,2$). \\
$t$; $\varphi_t,\xi_t$ & Iteration index and iterates; $t$ also denotes a scalar argument. \\
$I,I_d$; $\mathrm{id}$ & Identity matrix (size $d$ for $I_d$); identity map. \\
$\|\cdot\|_2,\|\cdot\|_\infty$ & Euclidean/spectral norm; maximum absolute value or supremum. \\
$\langle\cdot,\cdot\rangle$, $\preceq$ & Euclidean inner product; positive-semidefinite matrix order. \\
$\nabla,\nabla^2,J$ & Gradient, Hessian, and Jacobian; primes denote scalar derivatives. \\
$\operatorname{lse},z,\Delta_{N-1}$ & Log-sum-exp, logit vector, and probability simplex. \\
$\tanh,\operatorname{sech},\operatorname{arctanh}$ & Hyperbolic tangent, hyperbolic secant, and inverse hyperbolic tangent. \\
$f_1,f_2$; $q,C$ & Scalar features; potential and constant in $Q=\sum_kq(\varphi_k)+C$. \\
$g_1,g_2$ & S-Hop radicands, $t^2+\tanh t+1$ and $t^2+1$. \\
$g(t)$ & Auxiliary function, $\tanh t\,\operatorname{sech}^2t$. \\
$N(t)$ & Scalar convexity numerator, $2g_1g_1''-(g_1')^2$. \\
$t^*$; $T=-t$ & Stationary point $\operatorname{arctanh}(1/\sqrt3)$; negative-tail variable. \\
$c_k$ & Weighted feature coordinate, $\sum_i x_{i,k}s_i(\xi)$. \\
$g_Q,g_{Q^*}$; $Q^*$ & Hessian metrics; Fenchel conjugate of $Q$. \\
$T_\varphi\mathbb R^d$ & Tangent space at $\varphi$. \\
$u,v,w$ & Local tangent/dual vectors or scalar integration variables. \\
$\mathrm{d}h$; $D_Q$ & Differential of $h$; Bregman divergence generated by $Q$. \\
\midrule
\multicolumn{2}{@{}l}{\textit{Local Dynamics and Capacity Proofs}}\\*
\midrule
$U$ & Compact convex region containing stored phases and invariant balls. \\
$\Phi,J_\Phi$ & Phase update $h^{-1}(RXs(F(\varphi)))$ and its Jacobian. \\
$\varphi^\star,\xi^\star$ & Stationary phase and feature states. \\
$\varphi_i^\star,\xi_i^\star$ & Attracting fixed point in $S_i$ and its image $F(\varphi_i^\star)$. \\
$r,\tilde r$ & Distances from the current or intermediate phase point to $\psi_i$. \\
$\tilde\varphi,\tilde\xi$ & Intermediate phase point and its feature image. \\
$\varepsilon_0$ & Stored-pattern softmax leakage bound, $(N-1)e^{-\beta\Delta_i}$. \\
$\varepsilon,\hat\varepsilon$ & Upper bounds on softmax leakage $1-s_i$. \\
$\rho$ & Uniform contraction factor on an invariant ball, $0<\rho<1$. \\
$C_\varepsilon,C_r$ & Leakage and displacement coefficients, Eq.~\eqref{eq:jacobian-constants}. \\
$C_\varepsilon',C_r'$; $C_\varepsilon^{(\mathrm{soft})}$ & Perturbation bounds; softmax contribution to $C_\varepsilon$. \\
$N_0,N_0^{\mathrm{inv}},N_0^{\mathrm{store}}$ & Load thresholds for contraction, invariance, and correct storage. \\
$y_j$ & Transformed memory, $R(\psi_i)x_j$. \\
$w(\varphi),m$ & Local-field input $R(\varphi)Xs(F(\varphi))$; weighted memory $Xs(\xi)$. \\
$A,J_s$ & Softmax covariance $\operatorname{diag}(s)-ss^\top$ and Jacobian $\beta A$. \\
$\Delta R$ & Derivative of $R(\varphi)Xs$ with $s$ held fixed. \\
$\eta$ & Feature residual, $Xs(F(\varphi))-F(\varphi)$. \\
$\Delta_{ij},Y_{ij,k}$ & Directed gap $x_i^\top(x_i-x_j)$ and coordinate-$k$ contribution. \\
$m$ in $M=m\sqrt d$; $C_0$ & Norm coefficient; storage-threshold growth constant. \\
$\mathbb P,\mathbb E,\operatorname{Var}$ & Probability, expectation, and variance. \\
$O(\cdot),o(\cdot)$ & Asymptotic upper-order and vanishing relative-order notation. \\
\pagebreak
\multicolumn{2}{@{}l}{\textit{Retrieval, Preprocessing, and Diagnostic Quantities}}\\*
\midrule
$m_i$; $u,v$ (Figure~\ref{fig:energy_landscape}) & Illustrative memories; orthonormal memory-plane directions. \\*
$\sigma$ & Standard deviation of Gaussian query noise. \\*
$\psi,\hat\psi$; $\psi_k,\hat\psi_k$ & Target, retrieved output, and their pixel coordinates. \\*
$\mathrm{MSE},\mathrm{SSD}$ & Mean squared error and squared-difference sum; SSD${}=d\,\mathrm{MSE}$. \\*
$\operatorname{clamp}$; $\cos$ & Coordinate-wise clipping; cosine similarity. \\*
$\sigma^2_{\mathrm{PCA}},\lambda_k$ & PCA shrinkage regularizer and memory-covariance eigenvalue. \\*
$\Psi,\varphi^+$ & Raw memory matrix $[\psi_1,\ldots,\psi_N]$ and next phase state. \\*
$\bar\psi,\overline{F(\psi)}$ & Mean raw and embedded memories in centering controls. \\*
$x,\xi_j,\bar\xi$ (Appendix~\ref{subsec:appendix-orthogonality}) & Query, memory, and memory mean; tildes denote centered values. \\*
$W,u_i$ & U-Hop projection and normalized memory $W\psi_i/\|W\psi_i\|_2$. \\*
$m,q$; $x_{m,i},q_m$ & Method, query, and their indexed features in scale controls. \\*
$z_{m,i},z_{m,i}^{\mathrm{norm}}$ & Raw and normalized scores for method $m$ and memory $i$. \\*
$H,\sigma_z$ & Softmax entropy (nats); mean within-query logit standard deviation. \\*
$\sigma_m,a_m$ & Mean logit spread; scale multiplier $\sigma_{\mathrm{MHN}}/\sigma_m$. \\*
$\Delta_{\min,m}$ & Minimum directed feature gap for method $m$. \\*
$\operatorname{Std},\operatorname{mean}$ & Standard deviation and arithmetic mean. \\
\end{longtable}
\endgroup

\FloatBarrier

We analyze the feature-space fixed-point iteration from the main text:
\begin{equation}
  \xi_{t+1}
  =
  T(\xi_t)
  :=
  (F \circ h^{-1})\!\bigl(R(\varphi_t)\, X\, s(\xi_t)\bigr),
  \qquad
  s(\xi)=\mathrm{softmax}(\beta X^\top \xi),
  \label{eq:feature-iteration}
\end{equation}
with $\xi_t = F(\varphi_t)$ and $R(\varphi)=J_F(\varphi)^\top$.

The nonlinear phase kernel is
\begin{equation}
\scalebox{1}{$
  \label{eq:kernel-F}
  F(\varphi)
  =
  \Bigl[
    \sqrt{\varphi_1^{2} + \tanh\varphi_1 + 1},\;
    \sqrt{\varphi_1^{2} + 1},\;
    \ldots,\;
    \sqrt{\varphi_d^{2} + \tanh\varphi_d + 1},\;
    \sqrt{\varphi_d^{2} + 1}
  \Bigr]^{\top}$},
\end{equation}
so that the feature-space dimension is $D = 2d$.
Stored patterns are $x_i = F(\psi_i)$, and $R(\varphi)=J_F(\varphi)^\top$ has the
same square-root structure.  All boundedness and Lipschitz conditions required in
the analysis will follow from the explicit kernel above together with the global
properties of the local field in Lemma~\ref{lemma:strong-convexity} below.

\subsection{Kernel Family Viewpoint}
\label{subsec:kernel-family}
Although we work throughout with the explicit kernel \eqref{eq:kernel-F} for
concreteness, the arguments in this appendix are not specific to the hyperbolic-tangent
instance. More generally, the phase-space analysis applies to smooth
coordinate-wise kernels $F:\mathbb{R}^d\to\mathbb{R}^{2d}$ with
$F_{2k-1}(\varphi)=f_1(\varphi_k)$ and $F_{2k}(\varphi)=f_2(\varphi_k)$.
We require that the associated phase potential
\begin{equation}
  Q(\varphi) := \frac12 \,\|F(\varphi)\|_2^2
  \label{eq:Q-family}
\end{equation}
is (up to an additive constant) separable and strongly convex in the phase
coordinates. Concretely, we require that $Q$ can be written as
\begin{equation}
  Q(\varphi)
  =
  \sum_{k=1}^d q(\varphi_k) + C
  \label{eq:Q-family-separable}
\end{equation}
for some scalar $q:\mathbb{R}\to\mathbb{R}$ and constant $C$, so that the local field
$h=\nabla Q$ acts component-wise and admits a component-wise inverse $h^{-1}$ on its
range.  In addition, we assume uniform Hessian bounds
\begin{equation}
  \mu I_d \;\preccurlyeq\; \nabla^2 Q(\varphi) \;\preccurlyeq\; L_Q I_d
  \qquad \text{for all } \varphi\in\mathbb{R}^d,
  \label{eq:Q-family-hessian}
\end{equation}
for some constants $0<\mu\le L_Q<\infty$, which ensure that $h$ is globally
$\mu$-strongly monotone and $L_Q$-Lipschitz and that $h^{-1}$ is globally Lipschitz
on $\mathrm{range}(h)$.  
We further assume that $q''$ is globally Lipschitz, equivalently
that the third derivative is uniformly bounded:
\begin{equation}
  L_{q''} \;:=\; \sup_{t\in\mathbb{R}} |q'''(t)| \;<\; \infty.
  \label{eq:Q-family-third}
\end{equation}
This is a mild regularity condition on the scalar potential $q$ and is
used only in the Jacobian-perturbation step of
Lemma~\ref{lemma:Jacobian-bound} (Step~4).
Together with smoothness of $F$ on compact phase regions,
these conditions give the regularity bounds in Lemma~\ref{lemma:regularity}.
The contraction argument also uses the condition below.  The separable form~\eqref{eq:Q-family-separable}, Hessian bounds~\eqref{eq:Q-family-hessian}, and third-derivative bound~\eqref{eq:Q-family-third} can be checked directly for S-Hop and P-Hop.

In addition, we require the following \emph{contraction} condition
on the component functions $f_1, f_2$ of $F$ (i.e., $F_{2k-1}(\varphi)
= f_1(\varphi_k)$ and $F_{2k}(\varphi) = f_2(\varphi_k)$):
\begin{equation}
  \rho_0
  \;:=\;
  \sup_{t \in \mathbb{R}}
  \left|\frac{f_1(t)f_1''(t) + f_2(t)f_2''(t)}{q''(t)}\right|
  \;<\; 1.
  \label{eq:rho0-family}
\end{equation}
The identity
$q''=(f_1')^2+(f_2')^2+f_1f_1''+f_2f_2''$ then gives
\[
  (1-\rho_0)\mu
  \le [f_1'(t)]^2+[f_2'(t)]^2
  \le (1+\rho_0)L_Q.
\]
In particular, the scalar feature pair is nonconstant and has a uniformly
bounded derivative. S-Hop satisfies~\eqref{eq:rho0-family} by
Lemma~\ref{lemma:Jacobian-bound}, Step~1; the P-Hop bound is verified in
Appendix~\ref{subsec:phop-shop-comparison}.

\medskip \noindent\textbf{Remark (Hessian-geometric interpretation of the kernel family).}
The structural conditions \eqref{eq:Q-family-separable}--\eqref{eq:rho0-family} admit a clean geometric reading in terms of Hessian geometry \citep{shima2007geometry, amari2000methods}, which we state here as an interpretive summary. No new assumptions are introduced; all claims below follow from Lemma~\ref{lemma:strong-convexity} by standard results in convex analysis \citep{rockafellar1970convex}.

\emph{Induced Hessian metric.}
The global Hessian bounds $\mu I \preceq \nabla^2 Q(\varphi) \preceq L_Q I$ from Lemma~\ref{lemma:strong-convexity} imply that $Q$ induces a flat Hessian metric on phase space,
\begin{equation}
  g_Q(\varphi)[u, v] := u^\top \nabla^2 Q(\varphi)\, v, \qquad u, v \in T_\varphi \mathbb{R}^d,
  \label{eq:hessian-metric}
\end{equation}
turning $(\mathbb{R}^d, g_Q)$ into a smooth Riemannian manifold. By separability of $Q$ (Eq.~\eqref{eq:Q-family-separable}), the metric $g_Q$ is diagonal with entries $q''(\varphi_k)$.

\emph{Dual metric and Legendre isometry.}
Since $Q$ is globally strongly convex and essentially smooth, it is of Legendre type in the sense of \citet[\S26]{rockafellar1970convex}, so its Fenchel conjugate $Q^*(u) := \sup_\varphi\{\langle \varphi, u\rangle - Q(\varphi)\}$ is also smooth and strongly convex, with
\begin{equation}
  \nabla Q^* = h^{-1}, \qquad \nabla^2 Q^*(u) = \bigl(\nabla^2 Q(\varphi)\bigr)^{-1} \text{ whenever } u = h(\varphi).
  \label{eq:dual-hessian}
\end{equation}
The dual manifold $((\mathbb{R}^d)^*, g_{Q^*})$ carries the dual Hessian metric $g_{Q^*}(u) := \nabla^2 Q^*(u)$, satisfying $(1/L_Q) I \preceq g_{Q^*}(u) \preceq (1/\mu) I$. The Legendre map $h = \nabla Q: (\mathbb{R}^d, g_Q) \to ((\mathbb{R}^d)^*, g_{Q^*})$ is a Riemannian isometry: for any tangent vectors $v, w$,
\begin{equation}
  g_{Q^*}\bigl(h(\varphi)\bigr)\bigl[\mathrm{d}h(\varphi)[v], \mathrm{d}h(\varphi)[w]\bigr] = g_Q(\varphi)[v, w],
  \label{eq:isometry}
\end{equation}
which follows directly from $\mathrm{d}h = \nabla^2 Q$ and \eqref{eq:dual-hessian}.

\emph{CCCP update as primal--dual--primal trajectory.}
Using $h^{-1} = \nabla Q^*$, the update rule \eqref{eq:phi-update} rewrites as
\begin{equation}
  \varphi_{t+1} = \nabla Q^*\bigl(\nabla E_2(\varphi_t)\bigr),
  \label{eq:cccp-dual-form}
\end{equation}
which admits the geometric reading: (i) compute $\nabla E_2(\varphi_t)$ in the dual space $(\mathbb{R}^d)^*$; (ii) pull it back to primal phase space via the dual gradient $\nabla Q^*$. The separability of $Q$ preserves separability of $Q^*$ \citep[Theorem 16.4]{rockafellar1970convex}, making $\nabla Q^*$ component-wise tractable---this is precisely the structural property exploited in Lemma~\ref{lemma:cccp-identity}.

\emph{Condition number as a unified geometric constant.}
The condition number $\kappa := L_Q/\mu$ from Lemma~\ref{lemma:strong-convexity} admits three interrelated geometric interpretations in our framework, each following from Lemma~\ref{lemma:strong-convexity} together with the isometry \eqref{eq:isometry}:
\begin{enumerate}
  \item[\textbf{(i)}] \emph{Ellipsoidal distortion.} The metric $g_Q$ relative to the Euclidean metric has maximal-to-minimal eigenvalue ratio at most $\kappa$ uniformly in $\varphi$; equivalently, the metric ellipsoid $\{v : \|v\|_{g_Q(\varphi)} \leq 1\}$ has semi-axis ratio at most $\sqrt{\kappa}$.
  \item[\textbf{(ii)}] \emph{Bregman asymmetry.} The induced Bregman divergence $D_Q(\varphi_1, \varphi_2) := Q(\varphi_1) - Q(\varphi_2) - \langle h(\varphi_2), \varphi_1 - \varphi_2\rangle$ satisfies
  \begin{equation}
    \frac{1}{\kappa} \leq \frac{D_Q(\varphi_1, \varphi_2)}{D_Q(\varphi_2, \varphi_1)} \leq \kappa,
    \label{eq:bregman-asymmetry}
  \end{equation}
  for $\varphi_1\ne\varphi_2$, quantifying its asymmetry.
  \item[\textbf{(iii)}] \emph{Inverse-field sensitivity.} The inverse local field has Lipschitz constant at most $1/\mu$, and the softmax contribution to the Jacobian bound is $8\beta M^2L_{\max}^2\varepsilon/\mu$. Appendix~\ref{subsec:phop-shop-comparison} gives the leading capacity-scale dependence and the lower-order regularity terms.

\end{enumerate}

For S-Hop, $\kappa\approx1.477$, whereas P-Hop has $\kappa\approx2.09$ (Appendix~\ref{subsec:phop-shop-comparison}). These bounds characterize the local-potential geometry. The linear embedding $F=\mathrm{id}$ has $\kappa=1$ and recovers the standard MHN, which retains exponential-order capacity under its usual assumptions.

The strong-convexity bounds enter the convergence and capacity proofs together with the kernel contraction and memory-margin conditions. The geometric interpretation identifies separable strong convexity with a flat Hessian manifold and well-behaved Legendre duality \citep{shima2007geometry}.

\subsection{Verification of the CCCP Structural Identity}
\label{subsec:cccp-verification}

The surrogate-energy construction in Section~\ref{subsec:surrogate-energy}
requires that the phase potential satisfies the decoupling identity
\begin{equation}
  \frac{1}{2}\,\|F(\varphi)\|_2^2
  \;=\;
  \sum_{k=1}^{d}
  \int_0^{\varphi_k} h(u)\,\mathrm{d}u
  \;+\; C
  \tag{\ref{eq:quadratic-equals-integral}}
\end{equation}
for a component-wise local field $h$ and some additive constant $C$.
This identity is the structural prerequisite for the CCCP
update \eqref{eq:cccp-update-phi}--\eqref{eq:phi-update}: it ensures
that the implicit step $\nabla_\varphi E_1(\varphi_{t+1}) =
\nabla_\varphi E_2(\varphi_t)$ can be inverted component-wise via
$h^{-1}$.
For a general kernel in our family the convexity of $E_2$ is not
guaranteed, and the update has the form of a CCCP-inspired iteration.  However, for the S-Hop kernel \eqref{eq:kernel-F} we prove
in Section~\ref{subsec:E2-convexity} that $E_2(\varphi)$ is globally
strictly convex, upgrading the iteration to a genuine CCCP with
monotone energy descent (Theorem~\ref{thm:energy-descent},
Section~\ref{subsec:cccp-descent}).
We now verify this identity explicitly for the S-Hop kernel
\eqref{eq:kernel-F}.

\begin{lemma}[CCCP structural identity for S-Hop]
\label{lemma:cccp-identity}
For the S-Hop kernel $F$ defined in \eqref{eq:kernel-F}, the identity
\eqref{eq:quadratic-equals-integral} holds with the local field
\begin{equation}
  h(\varphi_k)
  \;=\;
  2\varphi_k + \tfrac{1}{2}\operatorname{sech}^{2}(\varphi_k),
  \qquad k = 1,\dots,d,
\end{equation}
and the constant $C = d$.
Consequently, the local regularizer is
$E_1(\varphi) = \sum_{k=1}^{d}\int_0^{\varphi_k} h(u)\,\mathrm{d}u$
with gradient $\nabla_\varphi E_1 = h$ (cf.\ Eq.~\eqref{eq:E1-grad}),
and the CCCP update
\eqref{eq:cccp-update-phi}--\eqref{eq:phi-update} is well-defined,
yielding the feature-space iteration
$\xi_{t+1} = T(\xi_t)$ in \eqref{eq:T-def}.
\end{lemma}

\begin{proof}
Expanding the left-hand side of \eqref{eq:quadratic-equals-integral}
using the explicit form of $F$ in \eqref{eq:kernel-F}:
\begin{align}
  \frac{1}{2}\,\|F(\varphi)\|_2^2
  &=
  \frac{1}{2}
  \sum_{k=1}^{d}
  \Bigl[
    \bigl(\varphi_k^{2} + \tanh\varphi_k + 1\bigr)
    +
    \bigl(\varphi_k^{2} + 1\bigr)
  \Bigr]
  \notag\\
  &=
  \sum_{k=1}^{d}
  \Bigl[
    \varphi_k^{2}
    + \tfrac{1}{2}\tanh\varphi_k
    + 1
  \Bigr].
  \label{eq:LHS-cccp}
\end{align}
For the right-hand side, we compute the antiderivative of $h$.
Using the standard integral
$\int_0^{t} \operatorname{sech}^{2}(u)\,\mathrm{d}u
= \tanh(t) - \tanh(0) = \tanh(t)$,
we obtain
\begin{align}
  \int_0^{\varphi_k}
  h(u)\,\mathrm{d}u
  &=
  \int_0^{\varphi_k}
  \!\Bigl[2u + \tfrac{1}{2}\operatorname{sech}^{2}(u)\Bigr]\,\mathrm{d}u
  =
  \varphi_k^{2}
  + \tfrac{1}{2}\tanh\varphi_k.
  \label{eq:h-antideriv}
\end{align}
Summing over $k$ and adding $C = d$:
\begin{equation}
  \sum_{k=1}^{d}
  \int_0^{\varphi_k} h(u)\,\mathrm{d}u
  \;+\; d
  \;=\;
  \sum_{k=1}^{d}
  \Bigl[
    \varphi_k^{2}
    + \tfrac{1}{2}\tanh\varphi_k
    + 1
  \Bigr].
  \label{eq:RHS-cccp}
\end{equation}
Comparing \eqref{eq:LHS-cccp} and \eqref{eq:RHS-cccp} term by term
establishes \eqref{eq:quadratic-equals-integral} with $C = d$.

Since $h(\varphi_k) = q'(\varphi_k)$ and $q''(t) \ge \mu > 0$ for all
$t \in \mathbb{R}$ (Lemma~\ref{lemma:strong-convexity} below), $h$ is
globally strictly monotone and hence admits a component-wise inverse
$h^{-1}$ on $\mathbb{R}$.
This guarantees that the implicit CCCP relation
$h(\varphi_{t+1}) = R(\varphi_t) X s(\xi_t)$
(Eq.~\eqref{eq:cccp-update-phi}) uniquely determines $\varphi_{t+1}$,
yielding the component-wise implicit update
$\varphi_{t+1} = h^{-1}(R(\varphi_t) X s(\xi_t))$
(Eq.~\eqref{eq:phi-update}) and the feature-space iteration
$T$ in \eqref{eq:T-def}.
\end{proof}

\medskip \noindent\textbf{Comparison with the P-Hop identity.} For the P-Hop kernel, the analogous identity holds with
$h_P(\varphi_k) = 2\varphi_k + \frac{1}{2}(\cos\varphi_k -
\sin\varphi_k)$ and $C_P = d$, following from
$\int_0^{t}\frac{1}{2}(\cos u - \sin u)\,\mathrm{d}u =
\frac{1}{2}(\sin t + \cos t - 1)$.
Both kernels thus satisfy the decoupling condition with the same constant
$C = d$; their local fields and regularity properties are analyzed in the following subsections.

\subsection{Regularity of the Kernel and Dynamics}
\label{subsec:regularity}

\begin{lemma}[Regularity of the kernel and dynamics]
\label{lemma:regularity}
Let $U\subset\mathbb{R}^d$ be a compact convex phase region containing
the stored phases and the balls used below
(note that $F$ is injective on $U$, since $\tanh$ is strictly monotone
and the map $\varphi_k \mapsto (\sqrt{\varphi_k^2+\tanh\varphi_k+1},\,
\sqrt{\varphi_k^2+1})$ is injective on $\mathbb{R}$).
For the kernel $F$ in \eqref{eq:kernel-F} and the component-wise local field
\begin{equation}
  h(\varphi_k)
  =
  2\varphi_k + \frac{1}{2}\operatorname{sech}^{2}(\varphi_k),
  \qquad k=1,\dots,d,
  \label{eq:local-field}
\end{equation}
the following hold.
\begin{enumerate}
  \item There exist finite constants $M,L_F,L_{\max}>0$ such that, for all
  $\varphi\in U$, $\xi=F(\varphi)$ and all stored patterns $x_i = F(\psi_i)$ with
  $\psi_i\in U$,
  \begin{equation}
    \|x_i\|_2 \le M,\qquad
    \|\xi\|_2 \le M,\qquad
    \|R(\varphi)\, x_i\|_2 \le M.
    \label{eq:M-bound}
  \end{equation}
  \item Both $h^{-1}$ and $F\circ h^{-1}$ are globally Lipschitz, with
  constants bounded by $L_F$.
  \item The Jacobians of $F$ and $h^{-1}$ are globally bounded, and the
  derivative of $R$ is bounded on $U$. Their operator-norm bounds are
  absorbed into $L_{\max}\ge1$.
\end{enumerate}
In particular, the map $T$ in \eqref{eq:feature-iteration} is smooth on
$F(U)$, and all subsequent bounds involving $M$, $L_F$, and $L_{\max}$
are understood to hold uniformly on this compact phase region.
\end{lemma}

\begin{proof}
For each $k$, the radicands in \eqref{eq:kernel-F} are
\[
  g_1(\varphi_k) := \varphi_k^2 + \tanh\varphi_k + 1,
  \qquad
  g_2(\varphi_k) := \varphi_k^2 + 1.
\]

\medskip \noindent\textbf{Radicand $g_2$.} Since $g_2(\varphi_k) = \varphi_k^2 + 1 \ge 1 > 0$ for all
$\varphi_k \in \mathbb{R}$, the second component
$\sqrt{g_2(\varphi_k)}$ is smooth everywhere with no clamping required.

\medskip \noindent\textbf{Radicand $g_1$.} Using $\tanh\varphi_k \in (-1,1)$ for all $\varphi_k \in \mathbb{R}$,
we have $g_1(\varphi_k) > \varphi_k^2 - 1 + 1 = \varphi_k^2 \ge 0$.
For an explicit bound, $\tanh t\ge t$ when $t\le0$, giving
$g_1(t)\ge t^2+t+1\ge3/4$; for $t\ge0$, $g_1(t)\ge1$.
Thus $g_1(t)\ge3/4$ globally.
Hence each component of $F$ in \eqref{eq:kernel-F} is smooth and has
bounded derivative on $U$, and so does $R(\varphi) = J_F(\varphi)^\top$.

Since $U$ is compact and $F$ is continuous, $F(U)$ is compact and the norms of
$x_i = F(\psi_i)$ and $\xi = F(\varphi)$ are uniformly bounded; the same holds
for $R(\varphi)x_i$ because $R$ is continuous and $x_i$ are bounded.  This
yields \eqref{eq:M-bound} for a suitable $M>0$.

From Lemma~\ref{lemma:strong-convexity} below we know that $h=\nabla Q$ is
globally Lipschitz and $\mu$-strongly monotone on $\mathbb{R}^d$ for some
$\mu>0$, so $h$ is a diffeomorphism onto its range and $h^{-1}$ is globally
Lipschitz with constant at most $1/\mu$.  The derivative bound following~\eqref{eq:rho0-family} makes $F$ globally
Lipschitz, so $F\circ h^{-1}$ is globally Lipschitz as well.  Finally, since $U$
is compact and all maps under consideration are smooth, their Jacobians are
uniformly bounded on $U$, and we absorb those bounds into $L_{\max}$.
\end{proof}

\subsection{Strong Convexity of the Local Phase Potential}
\label{subsec:local-potential}

We next record a simple global convexity property of the quadratic term
associated with the nonlinear phase kernel. Recall that the kernel map
$F : \mathbb{R}^d \to \mathbb{R}^D$ with $D = 2d$ is given componentwise by
\eqref{eq:kernel-F}, and that the local phase potential in the main text is
defined as
\begin{equation}
  Q(\varphi)
  :=
  \frac12 \,\|F(\varphi)\|_2^2.
  \label{eq:Q-def}
\end{equation}
\looseness=-1 In the main-text energy decomposition, $Q$ equals $E_1(\varphi)$ up to a constant.

\begin{lemma}[Strong convexity of the local phase potential]
\label{lemma:strong-convexity}
There exist global constants $0 < \mu \le L_Q < \infty$ such that
\begin{equation}
  \mu I_d
  \;\preccurlyeq\;
  \nabla^2 Q(\varphi)
  \;\preccurlyeq\;
  L_Q I_d
  \qquad
  \text{for all } \varphi \in \mathbb{R}^d.
  \label{eq:Q-Hessian-bounds}
\end{equation}
In particular, $Q$ (and hence $E_1$) is $\mu$-strongly convex on
$\mathbb{R}^d$, and its gradient $h(\varphi) = \nabla Q(\varphi)$ is
$\mu$-strongly monotone and $L_Q$-Lipschitz on $\mathbb{R}^d$.
\end{lemma}

\begin{proof}
By the explicit form of $F$, we obtain
\begin{equation}
  Q(\varphi)
  =
  \frac12 \sum_{k=1}^d
  \Bigl(
    \varphi_k^2 + \tanh \varphi_k + 1
    +
    \varphi_k^2 + 1
  \Bigr)
  =
  \sum_{k=1}^d q(\varphi_k),
  \label{eq:Q-separable}
\end{equation}
where the one-dimensional potential $q : \mathbb{R} \to \mathbb{R}$ is
\begin{equation}
  q(t)
  :=
  t^2 + \frac12\bigl(\tanh t + 2\bigr)
  =
  t^2 + \frac12 \tanh t + 1.
  \label{eq:q-def}
\end{equation}
In particular, $Q$ is separable in the coordinates $\{\varphi_k\}$ and its
Hessian is diagonal:
\begin{equation}
  \nabla^2 Q(\varphi)
  =
  \mathrm{diag}
  \bigl(
    q''(\varphi_1), \dots, q''(\varphi_d)
  \bigr).
  \label{eq:Q-Hessian-diag}
\end{equation}
Differentiating $q$ twice, we obtain
\begin{equation}
  q'(t)
  =
  2t + \frac12\operatorname{sech}^{2}(t),
  \qquad
  q''(t)
  =
  2 - \tanh(t)\operatorname{sech}^{2}(t).
  \label{eq:q-derivatives}
\end{equation}
Here we used the standard identity
$\frac{d}{dt}\operatorname{sech}^{2}(t)
= -2\tanh(t)\operatorname{sech}^{2}(t)$.

We now bound the perturbation term
$g(t) := \tanh(t)\operatorname{sech}^{2}(t)$.
Setting $g'(t) = 0$:
\begin{align}
  g'(t)
  &=
  \operatorname{sech}^{4}(t)
  -
  2\tanh^{2}(t)\,\operatorname{sech}^{2}(t)
  \notag\\
  &=
  \operatorname{sech}^{2}(t)
  \bigl[
    \operatorname{sech}^{2}(t) - 2\tanh^{2}(t)
  \bigr]
  =
  0.
\end{align}
Since $\operatorname{sech}^{2}(t) > 0$ for all $t$, the critical points
satisfy $\operatorname{sech}^{2}(t) = 2\tanh^{2}(t)$.
Using the identity $\operatorname{sech}^{2}(t) = 1 - \tanh^{2}(t)$,
this gives $1 - \tanh^{2}(t) = 2\tanh^{2}(t)$, hence
$\tanh^{2}(t) = 1/3$ and $\tanh(t) = \pm 1/\sqrt{3}$.

At these critical points, $\operatorname{sech}^{2}(t) = 1 - 1/3 = 2/3$, so
\[
  |g(t)|
  \le
  \frac{1}{\sqrt{3}} \cdot \frac{2}{3}
  =
  \frac{2}{3\sqrt{3}}
  =
  \frac{2\sqrt{3}}{9}
  \approx 0.385.
\]
Since $g(0) = 0$ and $g(t) \to 0$ as $t \to \pm\infty$, the global
extrema of $g$ are indeed attained at $\tanh(t) = \pm 1/\sqrt{3}$,
confirming $\|g\|_\infty = \frac{2}{3\sqrt{3}}$.
Therefore
\begin{equation}
  2 - \frac{2}{3\sqrt{3}}
  \;\le\;
  q''(t)
  \;\le\;
  2 + \frac{2}{3\sqrt{3}}
  \qquad
  \text{for all } t \in \mathbb{R}.
  \label{eq:q-second-bounds}
\end{equation}
This immediately yields the uniform bounds
\begin{equation}
  \Bigl(2 - \frac{2}{3\sqrt{3}}\Bigr) I_d
  \;\preccurlyeq\;
  \nabla^2 Q(\varphi)
  \;\preccurlyeq\;
  \Bigl(2 + \frac{2}{3\sqrt{3}}\Bigr) I_d
  \qquad
  \text{for all } \varphi \in \mathbb{R}^d,
\end{equation}
so we can take
\begin{equation}
  \mu := 2 - \frac{2}{3\sqrt{3}} \approx 1.615
  \qquad\text{and}\qquad
  L_Q := 2 + \frac{2}{3\sqrt{3}} \approx 2.385
\end{equation}
in \eqref{eq:Q-Hessian-bounds}.
The corresponding condition number is
$\kappa := L_Q / \mu \approx 1.477$.

By construction we also have $h(\varphi) = \nabla Q(\varphi)$, and
component-wise
\begin{equation}
  h(\varphi_k) = q'(\varphi_k)
  = 2\varphi_k + \tfrac12\operatorname{sech}^{2}(\varphi_k),
\end{equation}
which coincides with the local field used in the main text and in
\eqref{eq:local-field}.  The eigenvalue bounds
\eqref{eq:Q-Hessian-bounds} imply that $Q$ is $\mu$-strongly convex and that its
gradient $h = \nabla Q$ is $\mu$-strongly monotone and $L_Q$-Lipschitz on
$\mathbb{R}^d$: for all $\varphi,\psi \in \mathbb{R}^d$,
\begin{equation}
  \langle h(\varphi) - h(\psi), \, \varphi - \psi \rangle
  \;\ge\; \mu \,\|\varphi - \psi\|_2^2,
  \qquad
  \|h(\varphi) - h(\psi)\|_2
  \;\le\; L_Q \,\|\varphi - \psi\|_2.
  \label{eq:h-strong-monotone}
\end{equation}
In particular, strong monotonicity yields a global Lipschitz bound on
the inverse of $h$ on its range: whenever $u = h(\varphi)$ and $v = h(\psi)$,
\begin{equation}
  \|\varphi - \psi\|_2
  \;\le\; \frac{1}{\mu} \,\|h(\varphi) - h(\psi)\|_2
  \;=\; \frac{1}{\mu} \,\|u - v\|_2,
\end{equation}
so that $h$ is injective and $h^{-1}$ is globally Lipschitz with constant
at most $1/\mu$ on $\mathrm{range}(h)$.  Since $E_1(\varphi)$ differs from
$Q(\varphi)$ only by an additive constant, it inherits the same strong
convexity and smoothness properties.  These global bounds on $h$ and $h^{-1}$
will be used in the contraction analysis of the fixed-point map $T$ below.
\end{proof}

\subsection{Gradient of the Log-Sum-Exp Term}
\label{subsec:E2-gradient}

We derive $\nabla_\varphi E_2(\varphi)$ explicitly to recover Eq.~\eqref{eq:E2-grad}. From \eqref{eq:E2-def},
\begin{equation}
  E_2(\varphi)
  =
  \frac{1}{\beta}
  \log
  \sum_{i=1}^N
  \exp
  \bigl(
    \beta x_i^\top \xi
  \bigr),
  \qquad
  \xi = F(\varphi).
  \label{eq:appendix-E2}
\end{equation}
Taking the gradient with respect to $\varphi$ yields
\begin{align}
  \nabla_\varphi E_2(\varphi)
  &=
  \nabla_\varphi
  \left[
    \frac{1}{\beta}
    \log
    \sum_{i=1}^N
    \exp
    \bigl(
      \beta x_i^\top \xi
    \bigr)
  \right]
  \label{eq:E2-grad-1-appendix}
  \\
  &=
  \sum_{i=1}^N
  \frac{
    \exp(\beta x_i^\top \xi)
  }{
    \sum_{j=1}^N \exp(\beta x_j^\top \xi)
  }
  \nabla_\varphi (x_i^\top \xi)
  \label{eq:E2-grad-2-appendix}
  \\
  &=
  \sum_{i=1}^N
  s_i(\xi)\,
  \nabla_\varphi (x_i^\top \xi),
  \label{eq:E2-grad-3-appendix}
\end{align}
where $s_i(\xi)$ are the softmax weights defined in
\eqref{eq:softmax}. By the chain rule,
\begin{equation}
  \nabla_\varphi (x_i^\top \xi)
  = J_F(\varphi)^\top x_i
  = R(\varphi)x_i.
  \label{eq:JF-chain}
\end{equation}
Collecting terms, we obtain
\begin{equation}
  \nabla_\varphi E_2(\varphi)
  =
  J_F(\varphi)^\top X s(\xi)
  =
  R(\varphi) X s(\xi),
  \label{eq:E2-grad-appendix}
\end{equation}
which is Eq.~\eqref{eq:E2-grad} in the main text.

\subsection{Global Strict Convexity of \texorpdfstring{$E_2$}{E2} for S-Hop}
\label{subsec:E2-convexity}

The convergence analysis in
Sections~\ref{subsec:local-approx}--\ref{subsec:correct-storage}
(below) is based on contraction-mapping arguments and applies to the
entire kernel family.  In this and the next subsection we show that the
S-Hop kernel \eqref{eq:kernel-F} enjoys an additional, stronger
guarantee: the interaction term $E_2(\varphi)$ is globally strictly
convex in the phase variable~$\varphi$, so that the iteration
\eqref{eq:phi-update} is an instance of CCCP applied to the difference of convex decomposition
$E = E_1 - E_2$.  We begin by establishing the convexity of the two
scalar component functions of the S-Hop kernel.

\begin{lemma}[Strict convexity of $f_2$]
\label{lemma:f2-convex}
The function $f_2(t)=\sqrt{t^2+1}$ is strictly convex on~$\mathbb{R}$.
\end{lemma}

\begin{proof}
Differentiating twice,
\begin{equation*}
  f_2''(t)
  \;=\;
  \frac{1}{(t^2+1)^{3/2}}
  \;>\; 0
  \qquad \forall\, t\in\mathbb{R}.
  \qedhere
\end{equation*}
\end{proof}

\begin{lemma}[Global convexity of $f_1$]
\label{lemma:f1-convex}
The function $f_1(t)=\sqrt{t^2+\tanh(t)+1}$ satisfies
$f_1''(t)\ge 0$ for all $t\in\mathbb{R}$.
\end{lemma}

\begin{proof}
Let $g_1(t):=t^2+\tanh(t)+1$, so that $f_1(t)=\sqrt{g_1(t)}$.
By the chain rule, $f_1$ is convex if and only if
\begin{equation}
  N(t)
  \;:=\;
  2\,g_1(t)\,g_1''(t) - \bigl(g_1'(t)\bigr)^2
  \;\ge\; 0
  \qquad \forall\, t\in\mathbb{R}.
  \label{eq:N-def}
\end{equation}
We compute:
\begin{equation}
  g_1'(t) = 2t + \operatorname{sech}^2(t),
  \qquad
  g_1''(t) = 2 - 2\tanh(t)\operatorname{sech}^2(t).
  \label{eq:g1-derivs}
\end{equation}

\medskip \noindent\textbf{Global lower bound on $g_1''$.} From Lemma~\ref{lemma:strong-convexity},
\[
  \max_t \tanh(t)\operatorname{sech}^2(t)
  \;=\;
  \frac{2}{3\sqrt{3}}
  \;\approx\; 0.385,
\]
hence
\begin{equation}
  g_1''(t)
  \;\ge\;
  2 - 2\cdot\frac{2}{3\sqrt{3}}
  \;=\;
  2 - \frac{4}{3\sqrt{3}}
  \;\approx\; 1.230
  \;>\; 0
  \qquad \forall\, t\in\mathbb{R}.
  \label{eq:g1pp-lower}
\end{equation}
This implies that $g_1$ is itself \emph{strictly convex}, but this alone
does not directly imply the convexity of $\sqrt{g_1}$; we must verify
$N(t)\ge 0$.

\medskip \noindent\textbf{Asymptotic analysis.} \emph{Case $t\to+\infty$.}
Expanding $\tanh(t)$ and $\operatorname{sech}^2(t)$ gives
\[
\begin{aligned}
  g_1(t)&=t^2+2-2e^{-2t}+O(e^{-4t}),\\
  g_1'(t)&=2t+4e^{-2t}+O(e^{-4t}),\\
  g_1''(t)&=2-8e^{-2t}+O(e^{-4t}).
\end{aligned}
\]
Substituting,
\begin{equation}
  N(t) = 2(t^2+2)(2) - (2t)^2 + O(t^2e^{-2t})
       = 8 + O(t^2e^{-2t}) \;\to\; 8 > 0.
  \label{eq:N-plus-infty}
\end{equation}

\emph{Case $t\to-\infty$.}
Let $T=-t\to+\infty$. Expanding in $e^{-2T}$ gives
\[
\begin{aligned}
  g_1(-T)&=T^2+2e^{-2T}+O(e^{-4T}),\\
  g_1'(-T)&=-2T+4e^{-2T}+O(e^{-4T}),\\
  g_1''(-T)&=2+8e^{-2T}+O(e^{-4T}).
\end{aligned}
\]
A direct computation yields
\begin{align}
  N(-T)
  &= 2(T^2+2e^{-2T})(2+8e^{-2T}) - (-2T+4e^{-2T})^2+O(T^2e^{-4T})
  \notag\\
  &= e^{-2T}(16T^2+16T+8) + O(T^2e^{-4T})
  \;>\; 0.
  \label{eq:N-minus-infty}
\end{align}
Note that $N(-T)\to 0^+$, so $f_1$ becomes nearly linear as
$t\to-\infty$ but does not lose convexity.

\medskip \noindent\textbf{Finite-point verification and monotonicity of $N$.}
Let $t^*\approx 0.659$ be the point where
$\tanh(t)\operatorname{sech}^2(t)$ attains its maximum
($\tanh(t^*)=1/\sqrt{3}$):
\begin{equation}
  g_1(t^*)\approx 2.011,\qquad
  g_1'(t^*)\approx 1.984,\qquad
  g_1''(t^*)\approx 1.230.
  \label{eq:g1-at-tstar}
\end{equation}
Thus
\begin{equation}
  N(t^*)
  \approx 2(2.011)(1.230) - (1.984)^2
  \approx 4.947 - 3.936
  = 1.011 > 0.
  \label{eq:N-at-tstar}
\end{equation}

The sign can also be bounded without rounding: using
$0.65<t^*<0.66$ and $0.57<1/\sqrt3<0.58$ gives
$g_1(t^*)>1.9925$, $g_1''(t^*)>92/75$, and $g_1'(t^*)<149/75$.
Hence $N(t^*)>2(1.9925)(92/75)-(149/75)^2>0.93$.
It remains to justify that checking $t^*$ is enough. Differentiating
\eqref{eq:N-def} gives
\begin{equation}
  N'(t)
  =
  2\,g_1(t)\,g_1'''(t),
  \label{eq:N-prime}
\end{equation}
since the $2g_1'(t)g_1''(t)$ terms cancel. From \eqref{eq:g1-derivs},
\begin{equation}
  g_1'''(t)
  =
  2\,\operatorname{sech}^2(t)\bigl(3\tanh^2(t)-1\bigr),
  \label{eq:g1ppp}
\end{equation}
and $g_1(t)=t^2+\tanh(t)+1>0$ for all $t\in\mathbb{R}$. Hence the sign of
$N'(t)$ is exactly the sign of $3\tanh^2(t)-1$. Therefore $N$ is increasing on
$(-\infty,-t^*]$, decreasing on $[-t^*,t^*]$, and increasing on $[t^*,\infty)$,
where $t^*=\operatorname{arctanh}(1/\sqrt{3})$.

Now \eqref{eq:N-minus-infty} shows that $N(t)\to 0^+$ as $t\to-\infty$, so
$N(t)>0$ on $(-\infty,-t^*]$ because $N$ is increasing there. On
$[-t^*,\infty)$, the above monotonicity implies that the minimum is attained at
$t=t^*$, and \eqref{eq:N-at-tstar} gives $N(t^*)>0$. Therefore
$N(t)>0$ for all $t\in\mathbb{R}$, and thus $f_1''(t)\ge 0$ globally.
\end{proof}

\begin{lemma}[Strict positivity of S-Hop embeddings]
\label{lemma:embedding-positivity}
For all $\varphi_j\in\mathbb{R}$, every component of the S-Hop feature
vector $F(\varphi)$ defined in \eqref{eq:kernel-F} is strictly positive:
\[
  x_{i,k} = \bigl[F(\psi_i)\bigr]_k > 0
  \qquad \forall\, i,\, k.
\]
\end{lemma}

\begin{proof}
For each $\varphi_j\in\mathbb{R}$:
\begin{itemize}
  \item $f_2(\varphi_j) = \sqrt{\varphi_j^2+1} \ge 1 > 0$.
  \item $f_1(\varphi_j) = \sqrt{\varphi_j^2+\tanh(\varphi_j)+1}$:
  since $\tanh(t)>-1$ for all finite $t$, we have
  $\varphi_j^2+\tanh(\varphi_j)+1 > \varphi_j^2\ge 0$, and at
  $\varphi_j=0$ the radicand equals $0+0+1=1>0$.
  More precisely, from Lemma~\ref{lemma:regularity} the global minimum of
  the radicand is $\ge 3/4 > 0$.
\end{itemize}
Hence all embedded features $x_{i,k} = [F(\psi_i)]_k > 0$.
\end{proof}

We can now state the main convexity result.

\begin{theorem}[Global strict convexity of $E_2$ for S-Hop]
\label{thm:E2-convex}
Under the S-Hop kernel \eqref{eq:kernel-F}, the interaction term
\begin{equation}
  E_2(\varphi)
  \;=\;
  \frac{1}{\beta}\operatorname{lse}\bigl(\beta X^\top F(\varphi)\bigr)
\end{equation}
is globally strictly convex.
\end{theorem}

\begin{proof}
Denote $\xi = F(\varphi)$ and
$s = \mathrm{softmax}(\beta X^\top\xi)$.
The Hessian of $E_2$ is
\begin{equation}
  \nabla_\varphi^2 E_2(\varphi)
  \;=\;
  \underbrace{
    J_F(\varphi)^\top
    \cdot
    \beta X\bigl(\mathrm{diag}(s) - ss^\top\bigr)X^\top
    \cdot
    J_F(\varphi)
  }_{(A)}
  \;+\;
  \underbrace{
    \sum_{k=1}^{D} c_k \cdot \nabla_\varphi^2 F_k(\varphi)
  }_{(B)},
  \label{eq:E2-hessian}
\end{equation}
where $c_k := \sum_{i=1}^{N} x_{i,k}\,s_i(\xi)$.

\medskip \noindent\textbf{Term (A).} The matrix $\mathrm{diag}(s)-ss^\top$ is the Jacobian of the softmax
function and is positive semidefinite (PSD).
The quadratic form involving $J_F$, $X$, and $X^\top$ preserves this
property.  Hence $(A)\succcurlyeq 0$.

\medskip \noindent\looseness=-1 \textbf{Term (B).} By the separable structure of the S-Hop kernel \eqref{eq:kernel-F},
each $F_k$ depends on one phase coordinate~$\varphi_j$.
Specifically, $F_{2j-1}(\varphi)=f_1(\varphi_j)$ and
$F_{2j}(\varphi)=f_2(\varphi_j)$, with diagonal Hessians:
\begin{align}
  \nabla_\varphi^2 F_{2j-1}(\varphi)
  &= \mathrm{diag}(0,\ldots,f_1''(\varphi_j),\ldots,0),
  \notag\\
  \nabla_\varphi^2 F_{2j}(\varphi)
  &= \mathrm{diag}(0,\ldots,f_2''(\varphi_j),\ldots,0).
  \label{eq:Fk-hessians}
\end{align}
Hence $(B)$ is a diagonal matrix with the $j$-th diagonal entry
\begin{equation}
  (B)_{jj}
  \;=\;
  c_{2j-1}\cdot f_1''(\varphi_j)
  \;+\;
  c_{2j}\cdot f_2''(\varphi_j).
  \label{eq:B-diagonal}
\end{equation}
By Lemma~\ref{lemma:embedding-positivity}, $x_{i,k}>0$; since
$s_i > 0$ and $\beta>0$, we have $c_k>0$ for all~$k$.
By Lemma~\ref{lemma:f2-convex}, $f_2''(\varphi_j)>0$.
By Lemma~\ref{lemma:f1-convex}, $f_1''(\varphi_j)\ge 0$.
Therefore:
\begin{equation}
  (B)_{jj}
  \;\ge\;
  c_{2j}\cdot f_2''(\varphi_j)
  \;>\; 0
  \qquad \forall\, j.
  \label{eq:B-PD}
\end{equation}
Thus $(B)$ is strictly positive definite (PD).

\medskip \noindent\textbf{Combining.} $\nabla_\varphi^2 E_2(\varphi) = (A)+(B) \succ 0$.
Hence $E_2(\varphi)$ is globally strictly convex.
\end{proof}

\subsection{Strict CCCP Descent for S-Hop}
\label{subsec:cccp-descent}

With Theorem~\ref{thm:E2-convex} in hand, we upgrade the convergence
guarantee for S-Hop from the contraction-based analysis to a standard
CCCP descent argument.

\begin{theorem}[Energy monotone descent for S-Hop]
\label{thm:energy-descent}
Under the S-Hop kernel \eqref{eq:kernel-F}, the update rule
$\varphi_{t+1} = h^{-1}\!\bigl(R(\varphi_t)\,X\,s(\xi_t)\bigr)$
satisfies
\begin{equation}
  E(\varphi_{t+1}) \;\le\; E(\varphi_t)
  \qquad \forall\, t\ge 0.
  \label{eq:energy-descent}
\end{equation}
Equality holds if and only if $\varphi_t$ is a stationary point of~$E$.
\end{theorem}

\begin{proof}
By construction, $E = E_1 - E_2$, where $E_1$ is $\mu$-strongly convex
(Lemma~\ref{lemma:strong-convexity}) and $E_2$ is globally strictly
convex (Theorem~\ref{thm:E2-convex}).  Both functions are therefore
convex, so $E$ admits a difference of convex decomposition with both components convex.
The update rule \eqref{eq:cccp-update-phi} is precisely the CCCP
optimality condition:
\begin{equation}
  \nabla E_1(\varphi_{t+1})
  \;=\;
  \nabla E_2(\varphi_t).
  \label{eq:cccp-opt}
\end{equation}

\medskip \noindent\textbf{Energy difference.} \begin{equation}
  E(\varphi_{t+1}) - E(\varphi_t)
  \;=\;
  \bigl[E_1(\varphi_{t+1}) - E_1(\varphi_t)\bigr]
  \;-\;
  \bigl[E_2(\varphi_{t+1}) - E_2(\varphi_t)\bigr].
  \label{eq:energy-diff}
\end{equation}

\medskip \noindent\textbf{Bounding the $E_1$ difference.} Since $E_1$ is convex, the first-order convexity condition at
$\varphi_{t+1}$ gives
\[
  E_1(\varphi_t)
  \;\ge\;
  E_1(\varphi_{t+1})
  +
  \bigl\langle
    \nabla E_1(\varphi_{t+1}),\;
    \varphi_t - \varphi_{t+1}
  \bigr\rangle,
\]
which rearranges to
\begin{equation}
  E_1(\varphi_{t+1}) - E_1(\varphi_t)
  \;\le\;
  \bigl\langle
    \nabla E_1(\varphi_{t+1}),\;
    \varphi_{t+1} - \varphi_t
  \bigr\rangle.
  \label{eq:E1-upper}
\end{equation}

\medskip \noindent\textbf{Substituting the CCCP condition.} Using $\nabla E_1(\varphi_{t+1}) = \nabla E_2(\varphi_t)$ from
\eqref{eq:cccp-opt},
\begin{equation}
  E_1(\varphi_{t+1}) - E_1(\varphi_t)
  \;\le\;
  \bigl\langle
    \nabla E_2(\varphi_t),\;
    \varphi_{t+1} - \varphi_t
  \bigr\rangle.
  \label{eq:E1-via-cccp}
\end{equation}

\medskip \noindent\textbf{Bounding the $E_2$ difference.} By the convexity of $E_2$ (Theorem~\ref{thm:E2-convex}),
\begin{equation}
  E_2(\varphi_{t+1})
  \;\ge\;
  E_2(\varphi_t)
  +
  \bigl\langle
    \nabla E_2(\varphi_t),\;
    \varphi_{t+1} - \varphi_t
  \bigr\rangle.
  \label{eq:E2-lower}
\end{equation}
Rearranging:
$\bigl\langle \nabla E_2(\varphi_t),\,
\varphi_{t+1}-\varphi_t \bigr\rangle
\le E_2(\varphi_{t+1}) - E_2(\varphi_t)$.

\medskip \noindent\textbf{Combining.} Substituting \eqref{eq:E1-via-cccp} and \eqref{eq:E2-lower} into
\eqref{eq:energy-diff}:
\begin{equation}
  E(\varphi_{t+1}) - E(\varphi_t)
  \;\le\;
  \bigl\langle
    \nabla E_2(\varphi_t),\;
    \varphi_{t+1}-\varphi_t
  \bigr\rangle
  -
  \bigl[
    E_2(\varphi_{t+1}) - E_2(\varphi_t)
  \bigr]
  \;\le\; 0.
  \label{eq:descent-final}
\end{equation}
The last inequality follows directly from the convexity of~$E_2$.

\medskip \noindent\textbf{Equality condition.} Equality in \eqref{eq:E1-upper} requires $\varphi_{t+1}=\varphi_t$
(by the \emph{strict} convexity of~$E_1$), which is precisely the
stationary-point condition $\nabla E(\varphi_t)=0$.
\end{proof}

\begin{lemma}[Zangwill global convergence for S-Hop]
\label{cor:zangwill}
Let $\{\varphi^t\}_{t\ge 0}$ be the sequence generated by the S-Hop
update $\varphi^{t+1} = \Phi(\varphi^t)$ from any initialization
$\varphi^0 \in \mathbb{R}^d$.  Then the energy sequence
$\{E(\varphi^t)\}$ converges, and every accumulation point of
$\{\varphi^t\}$ is a stationary point of~$E(\varphi)$.
\end{lemma}

\begin{proof}
We verify the hypotheses of Zangwill's convergence theorem.

\medskip \noindent\looseness=-1 \textbf{Monotone descent.} Theorem~\ref{thm:energy-descent} gives
$E(\varphi^{t+1}) \le E(\varphi^t)$ for all $t$, with equality only at stationarity.

\medskip \noindent\textbf{Closedness of the iteration map.} The map $\Phi(\varphi) = h^{-1}\!\bigl(R(\varphi)\,X\,s(F(\varphi))\bigr)$ is
a composition of smooth functions ($F$, softmax, $R = J_F^\top$, and
the component-wise $h^{-1}$, whose smoothness follows from
Lemma~\ref{lemma:strong-convexity}).  Hence $\Phi$ is continuous and thus closed in Zangwill's sense.

\medskip \noindent\textbf{Coercivity and compactness of the trajectory.} The explicit potential gives $E_1(\varphi)\ge
\|\varphi\|_2^2-d/2\ge\tfrac{\mu}{2}\|\varphi\|_2^2-d/2$, so $E_1$ grows
quadratically.  For $E_2$, each component of the S-Hop kernel~\eqref{eq:kernel-F}
satisfies $|F_k(\varphi)| = O(|\varphi_{\lceil k/2\rceil}|)$ as
$|\varphi_{\lceil k/2\rceil}| \to \infty$ (since both
$\sqrt{t^2 + \tanh t + 1}$ and $\sqrt{t^2 + 1}$ grow linearly), so
$\|F(\varphi)\|_2 \le L_{\max}\|\varphi\|_2 + C$ for some $C > 0$.
Using $\mathrm{lse}(z) \le \max_i z_i + \log N$ and Cauchy--Schwarz,
\[
  E_2(\varphi)
  \;\le\;
  M L_{\max}\|\varphi\|_2 + C',
\]
i.e.\ $E_2$ grows at most linearly.  Therefore
\[
  E(\varphi)
  = E_1(\varphi) - E_2(\varphi)
  \;\ge\;
  \tfrac{\mu}{2}\|\varphi\|_2^2
  - M L_{\max}\|\varphi\|_2
  - C''
  \;\longrightarrow\; +\infty
  \quad\text{as } \|\varphi\|_2 \to \infty,
\]
confirming that $E$ is coercive.  By the descent property, the
trajectory lies in the sublevel set
$\{\varphi : E(\varphi) \le E(\varphi^0)\}$, which is closed and
bounded (hence compact) by coercivity.  In particular, $\{\varphi^t\}$
has at least one accumulation point.

\medskip \noindent\textbf{Conclusion.} All hypotheses hold: $E$ is a continuous descent function that decreases
strictly at non-stationary points, the map $\Phi$ is continuous (hence
closed), and the trajectory is contained in a compact set.  By
Zangwill's theorem, every accumulation point of $\{\varphi^t\}$ is a
stationary point of~$E$.  Since $\{E(\varphi^t)\}$ is monotonically
non-increasing and bounded below, it converges.
\end{proof}

\medskip \noindent\textbf{Remark.} The energy descent guarantee in Theorem~\ref{thm:energy-descent} is a
\emph{global} property that holds for any initialization,
complementing the \emph{local} contraction result of
Lemma~\ref{lemma:contraction} (below), which guarantees exponential
convergence within an invariant ball~$S_i$.
For the general P-Hop kernel, the convexity of~$E_2$ is not
established, and the convergence analysis relies entirely on the
contraction arguments in
Sections~\ref{subsec:local-approx}--\ref{subsec:Jacobian}.
This distinction represents a concrete theoretical advantage of S-Hop
over P-Hop within the framework of the general kernel family.

\section{Convergence, Storage Capacity, and Instantiation
Comparison}
\label{app:capacity-analysis}

The preceding part established the structural foundations of the
kernel family: the separable CCCP decoupling identity
(Appendix~\ref{subsec:cccp-verification}), uniform regularity of the kernel on
compact phase regions, and global
strong convexity of the local phase potential with condition number
$\kappa = L_Q/\mu$ (Appendix~\ref{subsec:local-potential}). 
For the
S-Hop instantiation we further obtained global strict convexity of
$E_2$ (Appendix~\ref{subsec:E2-convexity}) and the monotone CCCP energy
descent guarantee of Theorem~\ref{thm:energy-descent}
(Appendix~\ref{subsec:cccp-descent}). 

Building on these foundations, this part carries out the quantitative
convergence and capacity analysis. We first establish a local
approximation of the feature-space iteration near each stored pattern
(Appendix~\ref{subsec:local-approx}), identify invariant balls and
margin conditions that make the iterates well-separated
(Appendix~\ref{subsec:invariant-balls}), and derive the central Jacobian
bound that certifies local contraction on each invariant ball
(Appendix~\ref{subsec:Jacobian}). Combining these ingredients
with a pattern-separation argument
(Appendix~\ref{subsec:correct-storage}) yields the exponential
storage-capacity theorem (Theorem~\ref{thm:capacity},
Appendix~\ref{subsec:capacity}). We conclude with a quantitative
comparison of the S-Hop and P-Hop instantiations in terms of condition
number, radicand positivity, and convergence guarantees
(Appendix~\ref{subsec:phop-shop-comparison}).

For kernels in Appendix~\ref{subsec:kernel-family}, the phase-space
proofs below use $h=q'$ component-wise and the bounds of
Lemma~\ref{lemma:regularity}. These bounds follow from smoothness on
the compact region and the global derivative and inverse-field bounds.
The feature iterates are obtained by applying $F$ to the phase iterates;
when $F$ is injective, as for S-Hop, they also define the map $T$.

\subsection{Local Approximation at a Stored Pattern}
\label{subsec:local-approx}

The feature-space iteration $T$ in \eqref{eq:feature-iteration} factors through a
\emph{phase-space map}
\begin{equation}
  \Phi(\varphi)
  :=
  h^{-1}\!\bigl(R(\varphi)\,X\,s(F(\varphi))\bigr),
  \label{eq:Phi-def}
\end{equation}
so that $T(\xi) = F\!\bigl(\Phi(\varphi)\bigr)$ whenever
$\xi = F(\varphi)$.  Because $h^{-1}$, $R$, $X$, $s$, and $F$ are all
smooth maps defined on $\mathbb{R}^d$ (or $\mathbb{R}^D$),
$\Phi : \mathbb{R}^d \to \mathbb{R}^d$ is well-defined and smooth on
the \emph{full} phase space, with no restriction to an image
submanifold.  All subsequent local-convergence arguments are formulated
for $\Phi$ in phase space; the feature-space conclusions follow by
composing with $F$.

We show that a single application of $\Phi$ to a stored pattern
$\psi_i$ returns a point close to $\psi_i$ when the pattern has
sufficient margin.

\begin{lemma}[Local approximation at a stored pattern]
\label{lemma:local-approx}
Let the assumptions of Lemma~\ref{lemma:regularity} hold and $N\ge2$. For each
stored pattern $x_i = F(\psi_i)$, set
\begin{equation}
  \varepsilon_0
  =
  (N - 1)\exp\bigl(-\beta \Delta_i\bigr),
  \label{eq:eps0-def}
\end{equation}
where $\Delta_i = \min_{j \ne i}(x_i^\top x_i - x_i^\top x_j)$.
Then
\begin{equation}
  \|\Phi(\psi_i) - \psi_i\|_2
  \le
  \frac{2M}{\mu}\,\varepsilon_0.
  \label{eq:local-approx-bound}
\end{equation}
In particular, $\|T(x_i) - x_i\|_2 \le 2 L_F M \varepsilon_0$, since
$T(x_i) = F(\Phi(\psi_i))$ and $F \circ h^{-1}$ is Lipschitz with
constant $L_F$ (Lemma~\ref{lemma:regularity}).
\end{lemma}

\begin{proof}
Since $x_i = F(\psi_i)$, the phase-space map evaluated at $\psi_i$
reads
\begin{equation}
  \Phi(\psi_i)
  =
  h^{-1}\!\bigl(R(\psi_i)\, X\, s(x_i)\bigr).
  \label{eq:T-local}
\end{equation}

Define $y_j = R(\psi_i)\, x_j$ for $j = 1,\dots,N$ and
\begin{equation}
  S
  =
  \sum_{j=1}^N s_j(x_i)\, y_j.
  \label{eq:S-def}
\end{equation}
By the identity
$h(\psi_i) = \nabla Q(\psi_i) = J_F(\psi_i)^\top F(\psi_i) =
R(\psi_i)\, x_i = y_i$,
we have
\begin{equation}
  \psi_i
  =
  h^{-1}\!\bigl(h(\psi_i)\bigr)
  =
  h^{-1}(y_i).
  \label{eq:xi-identity}
\end{equation}
The logit $x_i^\top x_i$ dominates all others by at least $\Delta_i$
when evaluated at $\xi = x_i$, hence
\begin{equation}
  \sum_{j \ne i} s_j(x_i)
  \le
  (N - 1)\,\mathrm{e}^{-\beta \Delta_i}
  =
  \varepsilon_0.
  \label{eq:softmax-dom}
\end{equation}
We can write
\begin{equation}
  S - y_i
  =
  \sum_{j \ne i} s_j(x_i)\,(y_j - y_i).
  \label{eq:S-minus-yi}
\end{equation}
Using the triangle inequality and
$\|y_j\|_2 = \|R(\psi_i)\,x_j\|_2 \le M$ from \eqref{eq:M-bound},
we obtain
\begin{equation}
  \|S - y_i\|_2
  \le
  2M\varepsilon_0.
  \label{eq:S-minus-yi-bound}
\end{equation}
Since $\Phi(\psi_i) = h^{-1}(S)$ and
$\psi_i = h^{-1}(y_i)$ by \eqref{eq:xi-identity}, the
global Lipschitz continuity of $h^{-1}$ with constant $1/\mu$
(Lemma~\ref{lemma:strong-convexity}) yields
\begin{equation}
  \|\Phi(\psi_i) - \psi_i\|_2
  \le
  \frac{1}{\mu}\,\|S - y_i\|_2
  \le
  \frac{2 M}{\mu}\,\varepsilon_0,
\end{equation}
which is the claimed bound.
\end{proof}

\subsection{Invariant Balls and Margin Conditions}
\label{subsec:invariant-balls}

We now show that, under a separation condition on the patterns
$\{x_i\}$, the phase-space map $\Phi$ leaves a small neighbourhood of
each $\psi_i$ invariant.  For each stored pattern $\psi_i$ define the
\emph{phase-space ball}
\begin{equation}
  S_i
  =
  \left\{
    \varphi \in \mathbb{R}^d
    :
    \|\varphi - \psi_i\|_2
    \le
    \frac{1}{\beta N M L_{\max}}
  \right\},
  \label{eq:Si-def}
\end{equation}
with the compact region $U$ chosen to contain these balls. Define the margin
\begin{equation}
  \Delta_i
  =
  \min_{j \ne i}
  \bigl(
    x_i^\top x_i - x_i^\top x_j
  \bigr).
  \label{eq:Delta-i}
\end{equation}
Because $S_i$ is a ball in the full space $\mathbb{R}^d$ and $\Phi$ is
smooth on $\mathbb{R}^d$, the mean-value theorem can be applied along
any line segment in $S_i$ without domain restrictions---the key
advantage of working in phase space rather than in the image
$F(U) \subset \mathbb{R}^D$.

For $\varphi \in S_i$, define the \emph{effective margin} at
$\xi = F(\varphi)$ by
\begin{equation}
  \Delta_i^{\mathrm{eff}}(\xi)
  =
  \min_{j \ne i}
  \bigl(
    \xi^\top x_i - \xi^\top x_j
  \bigr).
  \label{eq:Delta-eff}
\end{equation}

\begin{lemma}[Invariant ball around a stored pattern]
\label{lemma:invariant-ball}
Assume the kernel and regularity conditions above, $N\ge2$, and
\[
  N \ge N_0^{\mathrm{inv}}
  :=
  \left\lfloor
    \frac{3C_r}{(1-\rho_0)\beta M L_{\max}}
  \right\rfloor + 1.
\]
If $\varphi \in S_i$ and
\begin{equation}
  \Delta_i
  \ge
  \max\!\left\{
    \frac{2}{\beta N}
    +
    \frac{1}{\beta}
    \log\!\Bigl(
      \frac{3 C_\varepsilon (N-1)}{1-\rho_0}
    \Bigr),
    \;
    \frac{1}{\beta}
    \log\!\Bigl(
      \frac{6\beta N M^2 L_{\max}(N-1)}{\mu(1-\rho_0)}
    \Bigr)
  \right\},
  \label{eq:Delta-cond}
\end{equation}
then $\Phi(\varphi)\in S_i$.
\end{lemma}

\begin{proof}
We bound $\|\Phi(\varphi)-\psi_i\|_2$ via the triangle inequality:
\begin{equation}
  \|\Phi(\varphi)-\psi_i\|_2
  \le
  \|\Phi(\varphi)-\Phi(\psi_i)\|_2
  +
  \|\Phi(\psi_i)-\psi_i\|_2.
  \label{eq:triangle}
\end{equation}

\medskip \noindent\textbf{Bound on the second term.}
By Lemma~\ref{lemma:local-approx},
\begin{equation}
  \|\Phi(\psi_i)-\psi_i\|_2
  \le
  \frac{2M}{\mu}(N-1)e^{-\beta\Delta_i}.
  \label{eq:term2-bound}
\end{equation}

\medskip \noindent\textbf{Bound on the first term.}
Since $\psi_i\in S_i$ and $S_i$ is convex, every point
$\tilde\varphi$ on the line segment $[\psi_i,\varphi]$ lies in $S_i$.
Let $\tilde\xi=F(\tilde\varphi)$. By the Lipschitz continuity of $F$,
\begin{equation}
  \|\tilde\xi-x_i\|_2
  \le
  \frac{1}{\beta N M}.
  \label{eq:feat-displacement}
\end{equation}
By Cauchy--Schwarz, this implies
\begin{align}
  \Delta_i^{\mathrm{eff}}(\tilde\xi)
  &=
  \min_{j \ne i}
  \bigl(
    \tilde\xi^\top x_i - \tilde\xi^\top x_j
  \bigr)
  \notag\\
  &\ge
  \Delta_i - 2\|\tilde\xi-x_i\|_2 M
  \notag\\
  &\ge
  \Delta_i-\frac{2}{\beta N}.
  \label{eq:Delta-eff-4}
\end{align}
Hence
\begin{equation}
  s_i(\tilde\xi)\ge 1-\hat\varepsilon,
  \qquad
  \hat\varepsilon
  :=
  (N-1)\exp\!\left(
    -\beta\Bigl(\Delta_i-\frac{2}{\beta N}\Bigr)
  \right).
\end{equation}
Applying Lemma~\ref{lemma:Jacobian-bound} to $\tilde\varphi\in S_i$ gives
\begin{equation}
  \|J_\Phi(\tilde\varphi)\|_2
  \le
  \rho_0 + C_\varepsilon \hat\varepsilon + C_r \tilde r,
  \qquad
  \tilde r:=\|\tilde\varphi-\psi_i\|_2
  \le
  \frac{1}{\beta N M L_{\max}}.
  \label{eq:J-hat-bound}
\end{equation}
Therefore
\begin{align}
  \|\Phi(\varphi)-\Phi(\psi_i)\|_2
  &\le
  \sup_{\tilde\varphi\in[\psi_i,\varphi]}
  \|J_\Phi(\tilde\varphi)\|_2\,
  \|\varphi-\psi_i\|_2
  \notag\\
  &\le
  \left(
    \rho_0
    +
    C_\varepsilon \hat\varepsilon
    +
    \frac{C_r}{\beta N M L_{\max}}
  \right)
  \frac{1}{\beta N M L_{\max}}.
  \label{eq:term1-bound}
\end{align}

\medskip \noindent\textbf{Combining.}
To ensure $\Phi(\varphi)\in S_i$, it suffices that
\[
  \|\Phi(\varphi)-\psi_i\|_2
  \le
  \frac{1}{\beta N M L_{\max}}.
\]
By \eqref{eq:triangle}, \eqref{eq:term2-bound}, and \eqref{eq:term1-bound},
it is sufficient that
\begin{equation}
  C_\varepsilon \hat\varepsilon
  \le
  \frac{1-\rho_0}{3},
  \qquad
  \frac{C_r}{\beta N M L_{\max}}
  \le
  \frac{1-\rho_0}{3},
  \qquad
  \frac{2M}{\mu}(N-1)e^{-\beta\Delta_i}
  \le
  \frac{1-\rho_0}{3\beta N M L_{\max}}.
\end{equation}
The second inequality follows from $N\ge N_0^{\mathrm{inv}}$.
The first inequality follows from the first term in
\eqref{eq:Delta-cond}, since
\[
  \hat\varepsilon
  =
  (N-1)\exp\!\left(
    -\beta\Bigl(\Delta_i-\frac{2}{\beta N}\Bigr)
  \right).
\]
The third inequality follows from the second term in
\eqref{eq:Delta-cond}. Hence $\Phi(\varphi)\in S_i$.
\end{proof}

\subsection{Jacobian Bounds and Contraction}
\label{subsec:Jacobian}

We bound the Jacobian of $\Phi$ to establish contraction on each invariant ball $S_i$.

\begin{lemma}[Jacobian bound]
\label{lemma:Jacobian-bound}
Let $\varphi \in S_i$ for some $i$, let $\xi = F(\varphi)$, let
\[
  r := \|\varphi-\psi_i\|_2,
\]
and let $\varepsilon \in (0,1/2]$ satisfy $s_i(\xi)\ge 1-\varepsilon$.
Define
\begin{equation}
  \rho_0
  :=
  \sup_{t \in \mathbb{R}}
  \left|\frac{f_1(t)\,f_1''(t) + f_2(t)\,f_2''(t)}{q''(t)}\right|
  \;<\; 1.
  \label{eq:rho0-def}
\end{equation}
Under Lemma~\ref{lemma:regularity} there exist finite constants
$C_\varepsilon,C_r>0$, depending only on kernel and regularity
parameters (and independent of $N$), such that
\begin{equation}
  \|J_\Phi(\varphi)\|_2
  \le
  \rho_0 + C_\varepsilon\,\varepsilon + C_r\,r.
  \label{eq:J-bound}
\end{equation}
Moreover, since $\varphi\in S_i$ implies
$r\le 1/(\beta N M L_{\max})$, we have
\begin{equation}
  \|J_\Phi(\varphi)\|_2
  \le
  \rho_0
  +
  C_\varepsilon\,\varepsilon
  +
  \frac{C_r}{\beta N M L_{\max}}.
  \label{eq:J-bound-eps}
\end{equation}
Let $L_2\ge1$ bound $|f_1''|$ and $|f_2''|$ on the scalar coordinates of $U$.
One may take
\begin{align}
 C_\varepsilon&=\max\left\{1,
 \frac{8\beta M^2L_{\max}^2}{\mu}
 +\frac{4L_2M}{\mu}
 +\frac{2\rho_0L_{q''}L_{\max}M}{\mu^2}\right\},\notag\\
 C_r&=\frac{2L_2L_{\max}}{\mu}
       +\frac{\rho_0L_{q''}L_{\max}^2}{\mu^2}.
 \label{eq:jacobian-constants}
\end{align}
\end{lemma}

\begin{proof}

\medskip \noindent\textbf{Step~1: $\rho_0 < 1$.}
For S-Hop the numerator in~\eqref{eq:rho0-def} is nonnegative
by Lemmas~\ref{lemma:f2-convex} and~\ref{lemma:f1-convex}. Since
\[
  q(t)=\frac12\bigl[f_1(t)^2+f_2(t)^2\bigr],
\]
differentiating twice gives
\[
  q''(t)
  =
  [f_1'(t)]^2 + f_1(t)\,f_1''(t)
  +
  [f_2'(t)]^2 + f_2(t)\,f_2''(t).
\]
Hence
\[
  \rho_0
  =
  1
  -
  \inf_{t\in\mathbb{R}}
  \frac{[f_1'(t)]^2+[f_2'(t)]^2}{q''(t)}.
\]
By Lemma~\ref{lemma:strong-convexity}, $q''(t)\ge \mu>0$ for all $t$.
For S-Hop, $f_2'(t)=t/\sqrt{t^2+1}$ vanishes only at $t=0$, while
$f_1'(0)=1/2\neq 0$. Therefore
\[
  [f_1'(t)]^2+[f_2'(t)]^2>0
  \qquad \forall\,t\in\mathbb{R}.
\]
The ratio in the infimum tends to $1$ as $|t|\to\infty$.
Continuity and positivity on each compact interval therefore give a
strictly positive global infimum, so $\rho_0<1$.
Other kernels in the family satisfy assumption~\eqref{eq:rho0-family} by hypothesis.

\medskip \noindent\textbf{Step~2: Decomposition of $J_\Phi$.}
Let
\[
  w(\varphi) := R(\varphi)\,X\,s(F(\varphi)),
\]
so that $\Phi(\varphi)=h^{-1}(w(\varphi))$. By the chain rule,
\begin{equation}
  J_\Phi(\varphi)
  =
  \left(
    \frac{\partial h}{\partial \varphi}
  \right)^{-1}
  \frac{\partial w}{\partial \varphi},
  \label{eq:J-chain}
\end{equation}
where $(\partial h/\partial\varphi)^{-1}$ is diagonal with entries
$1/q''(\Phi(\varphi)_k)$ and
\[
  \left\|
    \left(\frac{\partial h}{\partial\varphi}\right)^{-1}
  \right\|_2
  \le
  \frac1\mu.
\]
The second factor expands as
\begin{equation}
  \frac{\partial w}{\partial \varphi}
  =
  R(\varphi)\, X
  \frac{\partial s(F(\varphi))}{\partial \varphi}
  +
  \Delta R(\varphi),
  \label{eq:J-split}
\end{equation}
where
\begin{equation}
  \Delta R(\varphi)
  =
  \frac{\partial R(\varphi)}{\partial \varphi}\,
  X s(F(\varphi)).
  \label{eq:DeltaR-def}
\end{equation}

\medskip \noindent\textbf{Step~3: Softmax-dependent term.}
Using the chain rule through $\xi = F(\varphi)$ and the standard
softmax Jacobian identity,
\begin{equation}
  R(\varphi)\,X\,
  \frac{\partial s(F(\varphi))}{\partial \varphi}
  =
  R(\varphi)\,X\,J_s\,X^\top\,J_F(\varphi),
  \label{eq:Js-def}
\end{equation}
where
\[
  J_s = \beta(\mathrm{diag}(s)-ss^\top).
\]
Let
\[
  A:=\mathrm{diag}(s)-ss^\top,
  \qquad
  m := Xs(\xi)=\sum_{j=1}^N s_j(\xi)\,x_j.
\]
Then
\begin{equation}
  XAX^\top
  =
  \sum_{j=1}^N s_j(\xi)\,(x_j-m)(x_j-m)^\top.
  \label{eq:covariance-identity}
\end{equation}
We first control $m-x_i$. Since $\sum_{j\neq i}s_j(\xi)=1-s_i(\xi)\le\varepsilon$,
\[
  m-x_i
  =
  \sum_{j=1}^N s_j(\xi)x_j - x_i
  =
  \sum_{j\neq i}s_j(\xi)x_j + (s_i(\xi)-1)x_i
\]
which simplifies to
\[
  m-x_i = \sum_{j\neq i}s_j(\xi)(x_j-x_i).
\]
Hence, using $\|x_j\|_2\le M$,
\begin{equation}
  \|m-x_i\|_2
  \le
  \sum_{j\neq i}s_j(\xi)\,\|x_j-x_i\|_2
  \le
  2M\sum_{j\neq i}s_j(\xi)
  \le
  2M\varepsilon.
  \label{eq:m-minus-xi}
\end{equation}
Therefore
\[
  \|x_i-m\|_2 \le 2M\varepsilon,
\]
and for $j\neq i$,
\[
  \|x_j-m\|_2
  \le
  \|x_j\|_2+\|m\|_2
  \le
  \|x_j\|_2+\sum_{k=1}^N s_k(\xi)\|x_k\|_2
  \le
  2M.
\]
Applying the triangle inequality to \eqref{eq:covariance-identity},
\begin{align}
  \|XAX^\top\|_2
  &\le
  \sum_{j=1}^N s_j(\xi)\,\|(x_j-m)(x_j-m)^\top\|_2
  \notag\\
  &=
  \sum_{j=1}^N s_j(\xi)\,\|x_j-m\|_2^2
  \notag\\
  &\le
  s_i(\xi)\,(2M\varepsilon)^2
  +
  \sum_{j\neq i}s_j(\xi)\,(2M)^2
  \notag\\
  &\le
  4M^2\varepsilon^2 + 4M^2\varepsilon
  \notag\\
  &\le
  8M^2\varepsilon,
  \label{eq:XAX-bound}
\end{align}
where the last step uses $\varepsilon\in(0,1/2]$. Since $J_s=\beta A$,
\begin{equation}
  \|XJ_sX^\top\|_2
  \le
  8\beta M^2\varepsilon.
  \label{eq:XJsX-bound}
\end{equation}
Combined with
\[
  \left\|
    \left(\frac{\partial h}{\partial\varphi}\right)^{-1}
  \right\|_2 \le \frac1\mu,
  \qquad
  \|R(\varphi)\|_2\le L_{\max},
  \qquad
  \|J_F(\varphi)\|_2\le L_{\max},
\]
from Lemma~\ref{lemma:regularity}, the softmax-dependent contribution
to $\|J_\Phi(\varphi)\|_2$ is bounded by
\begin{equation}
  \frac{8\beta M^2 L_{\max}^2}{\mu}\,\varepsilon.
  \label{eq:softmax-contribution}
\end{equation}
This gives the softmax part of the bound, with
\[
  C_\varepsilon^{(\mathrm{soft})}
  :=
  \frac{8\beta M^2 L_{\max}^2}{\mu}.
\]

\medskip \noindent\textbf{Step~4: Structural term $\Delta R$.}
By the coordinate-wise kernel structure, $\Delta R(\varphi)$ is
diagonal with entries
\[
  \Delta R_{aa}
  =
  f_1''(\varphi_a)\,[Xs]_{2a-1}
  +
  f_2''(\varphi_a)\,[Xs]_{2a}.
\]
Write
\[
  [Xs]_{2a-1} = f_1(\varphi_a) + \eta_{2a-1},
  \qquad
  [Xs]_{2a} = f_2(\varphi_a) + \eta_{2a},
\]
where
\[
  \eta := Xs(F(\varphi)) - F(\varphi).
\]
Since
\[
  \eta
  =
  s_i(\xi)\,(x_i-F(\varphi))
  +
  \sum_{j\neq i} s_j(\xi)\,(x_j-F(\varphi)),
\]
we obtain
\[
  \|\eta\|_2
  \le
  s_i(\xi)\,\|x_i-F(\varphi)\|_2
  +
  \sum_{j\neq i}s_j(\xi)\,\|x_j-F(\varphi)\|_2.
\]
Now $\varphi\in S_i$ and Lemma~\ref{lemma:regularity} imply
\[
  \|x_i-F(\varphi)\|_2
  =
  \|F(\psi_i)-F(\varphi)\|_2
  \le
  L_{\max}\,\|\psi_i-\varphi\|_2
  =
  L_{\max} r,
\]
while
\[
  \sum_{j\neq i}s_j(\xi)
  =
  1-s_i(\xi)
  \le
  \varepsilon,
\]
and for every $j$,
\[
  \|x_j-F(\varphi)\|_2
  \le
  \|x_j\|_2+\|F(\varphi)\|_2
  \le
  2M.
\]
Therefore
\[
  \|\eta\|_2
  \le
  L_{\max}r + 2M\varepsilon,
\]
and hence, componentwise,
\begin{equation}
  |\eta_k|
  \le
  2M\varepsilon + L_{\max}r.
  \label{eq:eta-bound}
\end{equation}

Let
\[
  L_2 := \max\Bigl\{1,\ \sup_U |f_1''|,\ \sup_U |f_2''|\Bigr\}<\infty.
\]
Then
\begin{align}
  \Delta R_{aa}
  &=
  f_1''(\varphi_a)f_1(\varphi_a)
  +
  f_2''(\varphi_a)f_2(\varphi_a)
  +
  f_1''(\varphi_a)\eta_{2a-1}
  +
  f_2''(\varphi_a)\eta_{2a}
  \notag\\
  &=
  \bigl[
    f_1''(\varphi_a)f_1(\varphi_a)
    +
    f_2''(\varphi_a)f_2(\varphi_a)
  \bigr]
  +
  \zeta_a,
  \label{eq:DeltaR-split}
\end{align}
with
\[
  |\zeta_a|
  \le
  2L_2\,(2M\varepsilon+L_{\max}r).
\]
Dividing by $q''(\Phi(\varphi)_a)\ge \mu$ gives
\begin{equation}
  \frac{|\zeta_a|}{q''(\Phi(\varphi)_a)}
  \le
  \frac{4L_2M}{\mu}\,\varepsilon
  +
  \frac{2L_2L_{\max}}{\mu}\,r.
  \label{eq:zeta-divide}
\end{equation}

It remains to control the principal term
\[
  \frac{\left|f_1''(\varphi_a)f_1(\varphi_a)
  +f_2''(\varphi_a)f_2(\varphi_a)\right|}
  {q''(\Phi(\varphi)_a)}.
\]
First, by the definition of $\rho_0$,
\[
  \frac{\left|f_1''(\varphi_a)f_1(\varphi_a)
  +f_2''(\varphi_a)f_2(\varphi_a)\right|}
  {q''(\varphi_a)}
  \le
  \rho_0.
\]
Second, we compare $q''(\varphi_a)$ and $q''(\Phi(\varphi)_a)$.
By~\eqref{eq:Q-family-third}, $q''$ is globally Lipschitz with
constant $L_{q''}$. Also,
\[
  w(\varphi)-h(\varphi)
  =
  R(\varphi)\bigl[Xs(F(\varphi))-F(\varphi)\bigr]
  =
  R(\varphi)\eta,
\]
so
\[
  \|w(\varphi)-h(\varphi)\|_2
  \le
  \|R(\varphi)\|_2\,\|\eta\|_2
  \le
  L_{\max}(2M\varepsilon+L_{\max}r).
\]
Using the global Lipschitz continuity of $h^{-1}$ with constant
$1/\mu$ (Lemma~\ref{lemma:strong-convexity}),
\[
  \|\Phi(\varphi)-\varphi\|_2
  \le
  \frac{L_{\max}}{\mu}\,(2M\varepsilon+L_{\max}r).
\]
Hence
\[
  |q''(\varphi_a)-q''(\Phi(\varphi)_a)|
  \le
  L_{q''}\,|\varphi_a-\Phi(\varphi)_a|
  \le
  \frac{L_{q''}L_{\max}}{\mu}\,(2M\varepsilon+L_{\max}r).
\]
Since $q''(\Phi(\varphi)_a)\ge\mu$, it follows that
\begin{align}
  \frac{q''(\varphi_a)}{q''(\Phi(\varphi)_a)}
  &=
  1+
  \frac{q''(\varphi_a)-q''(\Phi(\varphi)_a)}
       {q''(\Phi(\varphi)_a)}
  \notag\\
  &\le
  1
  +
  \frac{L_{q''}L_{\max}}{\mu^2}\,(2M\varepsilon+L_{\max}r).
  \label{eq:qpp-ratio}
\end{align}
Therefore
\begin{align}
  \frac{\left|f_1''(\varphi_a)f_1(\varphi_a)
  +f_2''(\varphi_a)f_2(\varphi_a)\right|}
  {q''(\Phi(\varphi)_a)}
  &\le
  \rho_0\,
  \frac{q''(\varphi_a)}{q''(\Phi(\varphi)_a)}
  \notag\\
  &\le
  \rho_0
  +
  \frac{2\rho_0 L_{q''}L_{\max}M}{\mu^2}\,\varepsilon
  +
  \frac{\rho_0 L_{q''}L_{\max}^2}{\mu^2}\,r.
  \label{eq:principal-divide}
\end{align}

Combining \eqref{eq:DeltaR-split}, \eqref{eq:zeta-divide}, and
\eqref{eq:principal-divide}, we obtain
\[
  \left|
  \frac{\Delta R_{aa}}{q''(\Phi(\varphi)_a)}
  \right|
  \le
  \rho_0 + C_\varepsilon' \varepsilon + C_r' r
\]
for finite constants $C_\varepsilon',C_r'>0$ depending only on kernel
and regularity parameters.
Since the structural term is diagonal, this yields
\[
  \left\|
    \left(\frac{\partial h}{\partial\varphi}\right)^{-1}
    \Delta R(\varphi)
  \right\|_2
  \le
  \rho_0 + C_\varepsilon' \varepsilon + C_r' r.
\]

\medskip \noindent\textbf{Combining.}
Adding the bounds from Steps~3 and~4 gives
\[
  \|J_\Phi(\varphi)\|_2
  \le
  \rho_0
  +
  \bigl(C_\varepsilon^{(\mathrm{soft})}+C_\varepsilon'\bigr)\varepsilon
  +
  C_r' r.
\]
Using the constants in~\eqref{eq:jacobian-constants},
\[
  \|J_\Phi(\varphi)\|_2
  \le
  \rho_0 + C_\varepsilon\,\varepsilon + C_r\,r,
\]
which is \eqref{eq:J-bound}. Since
$r\le 1/(\beta N M L_{\max})$ for $\varphi\in S_i$, \eqref{eq:J-bound-eps}
follows immediately.
\end{proof}

\begin{lemma}[Contraction on the invariant ball]
\label{lemma:contraction}
Assume the hypotheses of Lemma~\ref{lemma:invariant-ball}. In particular,
$\Delta_i$ satisfies the margin condition
\begin{equation}
  \Delta_i
  >
  \frac{2}{\beta N}
  +
  \frac{1}{\beta}
  \log
  \Bigl(
    \frac{2C_\varepsilon (N-1)}{1-\rho_0}
  \Bigr),
  \label{eq:Delta-cond-strong}
\end{equation}
where $C_\varepsilon>0$ is the $N$-independent constant from
Lemma~\ref{lemma:Jacobian-bound}. Assume also that
\[
  N \ge N_0
  :=
  \left\lfloor
  \frac{2C_r}{(1-\rho_0)\beta M L_{\max}}
  \right\rfloor + 1.
\]
If $\varphi, \psi \in S_i$, then there exists $0 < \rho < 1$ such
that
\begin{equation}
  \|\Phi(\varphi) - \Phi(\psi)\|_2
  \le
  \rho \|\varphi - \psi\|_2.
  \label{eq:contraction}
\end{equation}
In particular, $\Phi$ is a contraction mapping on $S_i$.
\end{lemma}

\begin{proof}
By Lemma~\ref{lemma:invariant-ball}, $S_i$ is $\Phi$-invariant. Since
$S_i$ is a convex subset of $\mathbb{R}^d$, every point on the line
segment
\[
  \{\lambda \varphi + (1 - \lambda)\psi : \lambda\in[0,1]\}
\]
lies in $S_i$. Let $\tilde{\varphi}$ be any point on this segment, and
set $\tilde{\xi}=F(\tilde{\varphi})$.

By the same argument as in
\eqref{eq:feat-displacement}--\eqref{eq:Delta-eff-4}, we have
\[
  \|\tilde{\xi}-x_i\|_2 \le \frac{1}{\beta N M},
  \qquad
  \Delta_i^{\mathrm{eff}}(\tilde{\xi})
  \ge
  \Delta_i-\frac{2}{\beta N}.
\]
Hence
\[
  s_i(\tilde{\xi})
  \ge
  1-\varepsilon,
  \qquad
  \varepsilon
  :=
  (N-1)\exp\!\left(
    -\beta\Bigl(\Delta_i-\frac{2}{\beta N}\Bigr)
  \right).
\]
Applying Lemma~\ref{lemma:Jacobian-bound} to $\tilde{\varphi}\in S_i$
gives
\begin{equation}
  \|J_\Phi(\tilde{\varphi})\|_2
  \le
  \rho_0 + C_\varepsilon\,\varepsilon + C_r\,\tilde r,
  \qquad
  \tilde r := \|\tilde{\varphi}-\psi_i\|_2.
  \label{eq:J-tilde-bound}
\end{equation}
Since $\tilde{\varphi}\in S_i$,
\[
  \tilde r \le \frac{1}{\beta N M L_{\max}}.
\]

\medskip \noindent\textbf{$\varepsilon$-term.}
By \eqref{eq:Delta-cond-strong},
\[
  \varepsilon
  <
  \frac{1-\rho_0}{2C_\varepsilon},
\]
hence
\begin{equation}
  C_\varepsilon\,\varepsilon < \frac{1-\rho_0}{2}.
  \label{eq:eps-half-gap}
\end{equation}

\medskip \noindent\textbf{$r$-term.}
Since $N\ge N_0$,
\[
  N
  >
  \frac{2C_r}{(1-\rho_0)\beta M L_{\max}},
\]
and therefore
\[
  C_r\,\tilde r
  \le
  \frac{C_r}{\beta N M L_{\max}}
  <
  \frac{1-\rho_0}{2}.
\]
Thus
\begin{equation}
  C_r\,\tilde r < \frac{1-\rho_0}{2}.
  \label{eq:r-half-gap}
\end{equation}

Combining \eqref{eq:J-tilde-bound}, \eqref{eq:eps-half-gap}, and
\eqref{eq:r-half-gap}, we obtain
\[
  \|J_\Phi(\tilde{\varphi})\|_2 < 1
  \qquad
  \forall\,\tilde{\varphi}\in S_i.
\]
Now $J_\Phi$ is continuous on the compact set $S_i$
(Lemma~\ref{lemma:regularity}), so the function
$\tilde{\varphi}\mapsto \|J_\Phi(\tilde{\varphi})\|_2$ attains its
maximum on $S_i$. Since it is pointwise strictly less than $1$, there
exists a uniform constant $0<\rho<1$ such that
\[
  \|J_\Phi(\tilde{\varphi})\|_2 \le \rho
  \qquad
  \forall\,\tilde{\varphi}\in S_i.
\]
Since $\Phi$ is smooth on $\mathbb{R}^d$ and the line segment joining
$\psi$ and $\varphi$ lies entirely in $S_i \subset \mathbb{R}^d$, we may write
\begin{equation}
  \Phi(\varphi)-\Phi(\psi)
  =
  \int_0^1
  J_\Phi\!\bigl(\psi+t(\varphi-\psi)\bigr)\,(\varphi-\psi)\,dt .
  \label{eq:MVT}
\end{equation}
Hence, taking norms and using the uniform Jacobian bound on $S_i$,
\[
  \|\Phi(\varphi)-\Phi(\psi)\|_2
  \le
  \int_0^1
  \bigl\|J_\Phi\!\bigl(\psi+t(\varphi-\psi)\bigr)\bigr\|_2
  \,\|\varphi-\psi\|_2\,dt
  \le
  \rho\,\|\varphi-\psi\|_2,
\]
which proves \eqref{eq:contraction}.

\end{proof}

\subsection{Correct Storage and Separation of Basins}
\label{subsec:correct-storage}

We now combine the previous lemmas to show that each pattern has its own
attracting fixed point and basin of attraction.

\begin{lemma}[Correct storage and disjoint basins]
\label{lemma:correct-storage}
Assume the kernel and regularity conditions above, $N\ge2$, and
\[
  N > N_0^{\mathrm{store}}
  :=
  \max\{N_0,\; N_0^{\mathrm{inv}}\},
\]
where
\[
  N_0
  :=
  \left\lfloor
  \frac{2C_r}{(1-\rho_0)\beta M L_{\max}}
  \right\rfloor + 1,
  \qquad
  N_0^{\mathrm{inv}}
  :=
  \left\lfloor
    \frac{3C_r}{(1-\rho_0)\beta M L_{\max}}
  \right\rfloor + 1.
\]
Assume moreover that
\begin{equation}
  \Delta_{\min}
  :=
  \min_i \Delta_i
  \ge
  \Delta_{\mathrm{store}}(N),
  \label{eq:Delta-min-cond}
\end{equation}
where
\begin{equation}
\Delta_{\mathrm{store}}(N)
:=
\max\!\left\{
\begin{aligned}
  &\frac{2}{\beta N}
  +
  \frac{1}{\beta}
  \log\!\Bigl(
    \frac{3 C_\varepsilon (N-1)}{1-\rho_0}
  \Bigr), \\
  &\frac{1}{\beta}
  \log\!\Bigl(
    \frac{6\beta N M^2 L_{\max}(N-1)}{\mu(1-\rho_0)}
  \Bigr), \\
  &\frac{2}{\beta N}
  +
  \frac{1}{\beta}
  \log\!\Bigl(
    \frac{2C_\varepsilon (N-1)}{1-\rho_0}
  \Bigr)
\end{aligned}
\right\}.
\end{equation}

Then:
\begin{enumerate}
  \item For each $i$, there exists a unique fixed point
  $\varphi_i^\star$ of $\Phi$ in $S_i$, and every phase-space
  trajectory starting in $S_i$ converges to $\varphi_i^\star$.
  When $F$ is injective, $\xi_i^\star := F(\varphi_i^\star)$ is a fixed point
  of $T$, and every feature-space trajectory starting at any
  $\xi = F(\varphi)$ with $\varphi \in S_i$ converges to $\xi_i^\star$.
  \item The balls $S_i$ are pairwise disjoint, i.e.,
  $S_i \cap S_j = \emptyset$ for $i \ne j$. In particular, each stored
  pattern $x_i$ has its own basin of attraction, and the network
  retrieves the correct memory when initialized within $S_i$.
\end{enumerate}
\end{lemma}

\begin{proof}
If \eqref{eq:Delta-min-cond} holds, then for every $i$ the margin
$\Delta_i$ satisfies both the invariant-ball condition
\eqref{eq:Delta-cond} and the contraction condition
\eqref{eq:Delta-cond-strong}. Since $N>N_0^{\mathrm{store}}$, the
corresponding lower bounds on $N$ also hold. Therefore,
Lemma~\ref{lemma:invariant-ball} gives $\Phi(S_i)\subset S_i$,
and Lemma~\ref{lemma:contraction} establishes contraction on $S_i$.

Because $S_i$ is a closed ball in $\mathbb{R}^d$, it is complete.
Banach's fixed-point theorem therefore yields a unique fixed point
$\varphi_i^\star\in S_i$, and every trajectory starting in $S_i$
converges to it. Composing with the continuous map $F$, the feature-space
iterates $\xi^t=F(\varphi^t)$ converge to
$\xi_i^\star:=F(\varphi_i^\star)$. When $F$ is injective, this point satisfies
\[
  T(\xi_i^\star)
  =
  F(\Phi(\varphi_i^\star))
  =
  F(\varphi_i^\star)
  =
  \xi_i^\star.
\]

To show that the $S_i$ are disjoint, suppose there exists
$\varphi \in S_i \cap S_j$ with $i\ne j$. Then
\[
  \|\psi_i-\psi_j\|_2
  \le
  \|\psi_i-\varphi\|_2+\|\varphi-\psi_j\|_2
  \le
  \frac{2}{\beta N M L_{\max}}.
\]
Applying the Lipschitz bound for $F$ gives
\[
  \|x_i-x_j\|_2
  \le
  L_{\max}\|\psi_i-\psi_j\|_2
  \le
  \frac{2}{\beta N M}.
\]
Hence, by Cauchy--Schwarz and $\|x_i\|_2\le M$,
\[
  \Delta_i
  \le
  x_i^\top(x_i-x_j)
  \le
  \|x_i\|_2\,\|x_i-x_j\|_2
  \le
  \frac{2}{\beta N},
\]
which contradicts \eqref{eq:Delta-min-cond}, since
$\Delta_{\mathrm{store}}(N)>\frac{2}{\beta N}$ for $N\ge2$ and $C_\varepsilon\ge1$.
Therefore $S_i\cap S_j=\emptyset$ for $i\ne j$.
\end{proof}

\subsection{Exponential Storage Capacity}
\label{subsec:capacity}

We apply the deterministic margin condition of
Lemma~\ref{lemma:correct-storage} to memories generated by random phases.
Let $\psi_{i,k}$ be independent and uniformly distributed on $[-1,1]$
over both $i$ and $k$, and set $x_i=F(\psi_i)$. Throughout this
subsection, the scalar kernel functions and $\beta>0$ are fixed as $d$ grows.
For a scalar $u$ uniform on $[-1,1]$, define
\begin{align}
  v &:= \sum_{\ell=1}^{2}\operatorname{Var}\bigl(f_\ell(u)\bigr),
  \label{eq:phase-variance}\\
  B &:= \max_{u,w\in[-1,1]}
  \left|\sum_{\ell=1}^{2}f_\ell(u)
                   \bigl(f_\ell(u)-f_\ell(w)\bigr)\right|.
  \label{eq:phase-increment-bound}
\end{align}
Both constants depend only on the scalar kernel. Smoothness on $[-1,1]$
gives $B<\infty$. The derivative lower bound following
\eqref{eq:rho0-family} ensures that the pair $(f_1,f_2)$ is nonconstant
on this interval, so $v>0$ and hence $B>0$.

\begin{lemma}[Margin concentration for random phases]
\label{lemma:phase-margin}
For the random memories above and any $N\ge2$,
\begin{equation}
  \mathbb{P}\!\left(\Delta_{\min}<\frac{vd}{2}\right)
  \le N(N-1)\exp\!\left(-\frac{dv^2}{8B^2}\right).
  \label{eq:phase-margin-tail}
\end{equation}
\end{lemma}

\begin{proof}
For each ordered pair $i\ne j$, write
\[
  \Delta_{ij}:=x_i^\top(x_i-x_j)
  =\sum_{k=1}^{d}Y_{ij,k},\qquad
  Y_{ij,k}:=\sum_{\ell=1}^{2}f_\ell(\psi_{i,k})
       \bigl(f_\ell(\psi_{i,k})-f_\ell(\psi_{j,k})\bigr).
\]
The summands are independent over $k$, satisfy $|Y_{ij,k}|\le B$,
and have mean
\[
  \mathbb{E}Y_{ij,k}
  =\sum_{\ell=1}^{2}\left[
     \mathbb{E}f_\ell(u)^2-\bigl(\mathbb{E}f_\ell(u)\bigr)^2
   \right]=v.
\]
Hoeffding's inequality~\cite{hoeffding1963probability} therefore gives
$\mathbb{P}(\Delta_{ij}<vd/2)\le\exp(-dv^2/(8B^2))$.
Taking a union bound over the $N(N-1)$ ordered pairs proves the claim.
\end{proof}

\begin{theorem}[Exponential storage capacity]
\label{thm:capacity}
Let $F$ be a fixed coordinate-wise kernel satisfying the conditions in
Appendix~\ref{subsec:kernel-family}, and let the memories be generated
by the random phases above. Fix $\beta>0$, $\delta\in(0,1)$, and
\begin{equation}
  0<\gamma<\min\left\{\frac{v^2}{16B^2},\frac{\beta v}{4}\right\}.
  \label{eq:capacity-rate}
\end{equation}
For all sufficiently large $d$ and every integer
\begin{equation}
  2\le N\le\exp(\gamma d)=\exp(\gamma D/2),
  \label{eq:capacity-bound}
\end{equation}
the margin condition~\eqref{eq:Delta-min-cond} holds with probability
at least $1-\delta$. Each phase-space ball $S_i$ then contains a unique
attracting fixed point of $\Phi$, and the balls are pairwise disjoint.
\end{theorem}

\begin{proof}
\medskip\noindent\textbf{Uniform regularity constants.}
Use the compact phase region $U=[-2,2]^d$, which contains all sampled
memories and their balls $S_i$ for sufficiently large $d$.
The coordinate-wise structure gives dimension-independent bounds on
$J_F$ and the second derivatives of its scalar components on this region.
Choose $L_{\max}\ge1$ to dominate these bounds, and choose a fixed
$m\ge1$ large enough that
\[
  \|F(\varphi)\|_2\le m\sqrt d,\qquad
  \|R(\varphi)x_i\|_2\le m\sqrt d
  \quad (\varphi\in U).
\]
Thus $M=m\sqrt d$ satisfies~\eqref{eq:M-bound}. The constants in
Lemma~\ref{lemma:Jacobian-bound} obey
$C_\varepsilon=O(d)$ and $C_r=O(1)$, with constants independent of
$N$ and $d$. We may take $C_\varepsilon\ge1$ as specified there.
The thresholds $N_0$ and $N_0^{\mathrm{inv}}$ are consequently equal
to $1$ for all sufficiently large $d$.

\medskip\noindent\textbf{Deterministic threshold.}
Substituting these bounds into the full definition of
$\Delta_{\mathrm{store}}(N)$ in Lemma~\ref{lemma:correct-storage}
shows that there is a constant $C_0>0$, independent of $N$ and $d$, such that
\begin{equation}
  \Delta_{\mathrm{store}}(N)
  \le \frac{2}{\beta}\log N
       +\frac{1}{\beta}\log d+C_0
  \qquad (N\ge2).
  \label{eq:correct-storage-condition}
\end{equation}
Indeed, the first and third logarithms in that definition are at most
$\log N+\log d$ plus a constant, and the second is at most
$2\log N+\log d$ plus a constant. The remaining term
$2/(\beta N)$ is bounded by $1/\beta$.
If $N\le\exp(\gamma d)$, then
\[
  \Delta_{\mathrm{store}}(N)
  \le\frac{2\gamma}{\beta}d+\frac{1}{\beta}\log d+C_0
  <\frac{v}{2}d
\]
for all sufficiently large $d$, since $\gamma<\beta v/4$.

\medskip\noindent\textbf{Probability and storage.}
By Lemma~\ref{lemma:phase-margin}, the failure probability is at most
\[
  N(N-1)e^{-dv^2/(8B^2)}
  \le\exp\!\left[-d\left(\frac{v^2}{8B^2}-2\gamma\right)\right],
\]
which is at most $\delta$ for sufficiently large $d$ by
\eqref{eq:capacity-rate}. On the complementary event,
$\Delta_{\min}\ge vd/2>\Delta_{\mathrm{store}}(N)$ and
$N>N_0^{\mathrm{store}}$. Lemma~\ref{lemma:correct-storage} gives
the asserted fixed points, convergence, and disjointness.
\end{proof}

Theorem~\ref{thm:capacity} establishes exponential-order storage for
memories in the image of $F$. The constants $v$ and $B$ describe the
geometry induced by the kernel on the stated phase distribution.
Section~\ref{sec:results} evaluates noisy retrieval on natural images.

\subsection{Theoretical Comparison of P-Hop and S-Hop Instantiations}
\label{subsec:phop-shop-comparison}

Both P-Hop and S-Hop are concrete instantiations of the general kernel
family specified in Section~\ref{subsec:surrogate-energy} and
verified in this appendix.  They share the same feature-space dimension
$D = 2d$, the same surrogate energy decomposition $E = E_1 - E_2$, the
same CCCP structural identity (Lemma~\ref{lemma:cccp-identity} and
its P-Hop counterpart, both with constant $C = d$), and the same
phase-space update form $\Phi$ in \eqref{eq:Phi-def}.
The phase-space conclusions of
Lemmas~\ref{lemma:local-approx}--\ref{lemma:correct-storage} and
Theorem~\ref{thm:capacity} apply to both kernels under their stated
conditions. In addition to the potential bounds, they use the absolute
structural bound~\eqref{eq:rho0-family} verified below.
The two kernels have different scalar features, potentials, and
regularity constants.  We now present a systematic comparison.

\medskip \noindent\textbf{Kernel definitions.} The two maps $F:\mathbb{R}^d\to\mathbb{R}^{2d}$ are defined
component-wise for each $k=1,\dots,d$:
\begin{equation}
\renewcommand{\arraystretch}{1.6}
\begin{array}{lll}
\toprule
 & \textbf{P-Hop} & \textbf{S-Hop} \\
\midrule
f_1(\varphi_k)
  & \sqrt{\varphi_k^2+\sin\varphi_k+0.5}
  & \sqrt{\varphi_k^2+\tanh\varphi_k+1} \\
f_2(\varphi_k)
  & \sqrt{\varphi_k^2+\cos\varphi_k+0.5}
  & \sqrt{\varphi_k^2+1} \\
\bottomrule
\end{array}
\label{eq:kernel-comparison-table}
\end{equation}

\medskip \noindent\textbf{Scalar potential and local field.} Both kernels yield a separable phase potential
$Q(\varphi)=\sum_k q(\varphi_k)$ with local field $h=q'$:
\begin{equation}
\renewcommand{\arraystretch}{1.6}
\begin{array}{lll}
\toprule
 & \textbf{P-Hop} & \textbf{S-Hop} \\
\midrule
q(t)
  & t^2+\tfrac{1}{2}(\sin t+\cos t+1)
  & t^2+\tfrac{1}{2}\tanh t+1 \\[2pt]
h(t)=q'(t)
  & 2t+\tfrac{1}{2}(\cos t-\sin t)
  & 2t+\tfrac{1}{2}\operatorname{sech}^{2}(t) \\[2pt]
q''(t)
  & 2-\tfrac{1}{2}(\cos t+\sin t)
  & 2-\tanh(t)\operatorname{sech}^{2}(t) \\
\bottomrule
\end{array}
\label{eq:potential-comparison}
\end{equation}

\medskip \noindent\textbf{Regularity constants and condition number.} The global Hessian bounds $\mu \le q''(t) \le L_Q$ are determined by
bounding the perturbation term $\|q''(t) - 2\|_\infty$:
\begin{itemize}
  \item \textbf{P-Hop (tight bound):}
  $\tfrac{1}{2}(\cos t+\sin t) = \tfrac{\sqrt{2}}{2}\sin(t+\tfrac{\pi}{4})$,
  so
  $\bigl\|\tfrac{1}{2}(\cos t+\sin t)\bigr\|_\infty
  = \frac{\sqrt{2}}{2} \approx 0.707.$
  \item \textbf{S-Hop (tight bound):}
  $\bigl\|\tanh(t)\operatorname{sech}^{2}(t)\bigr\|_\infty
  = \frac{2}{3\sqrt{3}} = \frac{2\sqrt{3}}{9} \approx 0.385$
  (attained at $\tanh^2 t = 1/3$; see Lemma~\ref{lemma:strong-convexity}).
\end{itemize}

\noindent
\textbf{Remark.}
The original P-Hop analysis uses the looser elementwise bound
$\cos t + \sin t \in [-2, 2]$, yielding
$\mu_P^{\mathrm{loose}}=1$ and $L_{Q,P}^{\mathrm{loose}}=3$
($\kappa_P^{\mathrm{loose}}=3$).
This is sufficient for all theoretical conclusions but overestimates
$\kappa$.
For a fair comparison with S-Hop, we give tight bounds for both kernels.

The resulting constants are:
\begin{equation*}
\renewcommand{\arraystretch}{1.4}
\begin{array}{lcccc}
\toprule
 & \|q''-2\|_\infty & \mu & L_Q & \kappa := L_Q/\mu \\
\midrule
\textbf{P-Hop (tight)}
  & \frac{\sqrt{2}}{2} \approx 0.707
  & 2 - \frac{\sqrt{2}}{2} \approx 1.293
  & 2 + \frac{\sqrt{2}}{2} \approx 2.707
  & \approx 2.09 \\[4pt]
\textbf{P-Hop (loose, as stated)}
  & 1
  & 1
  & 3
  & 3.00 \\[4pt]
\textbf{S-Hop (tight)}
  & \frac{2}{3\sqrt{3}} \approx 0.385
  & 2 - \frac{2}{3\sqrt{3}} \approx 1.615
  & 2 + \frac{2}{3\sqrt{3}} \approx 2.385
  & \approx 1.477 \\
\bottomrule
\end{array}
\label{eq:constants-comparison}
\end{equation*}

The condition number $\kappa=L_Q/\mu$ describes the curvature of the local potential and bounds the inverse-field sensitivity through $1/\mu$. The complete contraction bound combines the embedding, Jacobian, and memory-margin constants.

\medskip\noindent\textbf{Structural contraction bound.}
For S-Hop, Lemma~\ref{lemma:Jacobian-bound}, Step~1, proves $\rho_0<1$.
For P-Hop, use the identity
\[
 \frac{f_1f_1''+f_2f_2''}{q''}
 =1-\frac{(f_1')^2+(f_2')^2}{q''}.
\]
For $|t|\ge10$, the explicit derivatives and radicand bounds give
\[
 \frac{2(|t|-1/2)^2}{t^2+3/2}
 \le (f_1'(t))^2+(f_2'(t))^2
 \le \frac{2(|t|+1/2)^2}{t^2-1/2}.
\]
Together with $2-\sqrt2/2\le q''\le2+\sqrt2/2$, these bounds place
the structural ratio in $[-0.715,0.344]$ on both tails.
On $[-10,10]$, outward-rounded interval evaluation on $20{,}000$
intervals of width $10^{-3}$ encloses the ratio in $[-0.53,0.983]$.
The supplied \texttt{verify\_theory.py} reproduces this enclosure
using interval arithmetic. Thus P-Hop also satisfies $\rho_0\le0.983<1$.

\medskip\noindent\textbf{Scale and asymptotic order.}
Theorem~\ref{thm:capacity} permits any exponent satisfying
\eqref{eq:capacity-rate}. The ratio $v^2/B^2$ controls the probability
of memory separation, while $\beta v$ controls the deterministic storage
threshold. Multiplying all scalar features by a positive constant multiplies
both $v$ and $B$ by its square, preserving the former ratio and scaling
$\beta v$ accordingly. The regularity constants enter the dimension
threshold through~\eqref{eq:correct-storage-condition}.
The values $\kappa_P\approx2.09$ and $\kappa_S\approx1.477$ describe
the local-potential curvature; Figures~\ref{fig:appendix_phop_capacity}
and~\ref{fig:appendix_phop_curves} compare finite-dimensional retrieval.

\medskip\noindent\textbf{Radicands and descent.}
The stated P-Hop radicand $t^2+\sin t+0.5$ has a positive minimum of approximately $0.268$, while $t^2+\cos t+0.5\geq1.5$. A clamp at $10^{-8}$ is therefore inactive for these formulas.
Both kernels have separable strongly convex local potentials and satisfy the local convergence and storage results under the stated assumptions. For S-Hop, Theorem~\ref{thm:energy-descent} additionally establishes strict CCCP monotone descent.

\FloatBarrier

\section{Additional Numerical Results}
\label{sec:appendix-experiment}

This appendix reports supplementary experiments that complement the main results
in Section~\ref{sec:results}.
We focus on  (i) cross-dataset verification on MNIST,
(ii) framework-level validation through a second admissible kernel instance,
(iii) consistency under the aligned U-Hop SSD metric,
(iv) an alternative evaluation protocol inherited from the U-Hop
codebase, (v) comparison with a centered dot-product control,
(vi) extended ablations and diagnostics,
(vii) a plug-in gating wrapper for S-Hop,
(viii) additional generalization evidence on an MIL
benchmark, and (ix) downstream CIFAR-10 classification results.
Unless otherwise stated, we follow the same MSE-based retrieval criterion as in
the main text, with successful retrieval defined by $\mathrm{MSE}<0.1$.

\subsection{MNIST Retrieval and Capacity}
\label{subsec:appendix-mnist}

This subsection extends the CIFAR-10 and ImageNet capacity comparisons
to lower-dimensional MNIST images.
We repeat the capacity protocol of Section~\ref{sec:results} on MNIST
($d=784$) under identical settings: one-step retrieval with $\beta=1$,
additive Gaussian noise of standard deviation $\sigma=0.3$, and the
$\mathrm{MSE}<0.1$ perfect-retrieval criterion.
The four methods are S-Hop, U-Hop (\emph{max}), MHN,
and the invalid-kernel baseline C-Hop.
This cross-dataset check is the most direct test of whether the observed
advantage generalizes beyond a single image resolution.

\begin{figure}[htbp]
    \centering
    \includegraphics[width=0.48\linewidth]{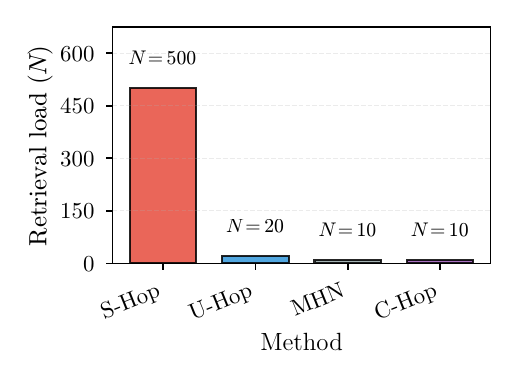}
    \caption{MNIST retrieval-load summary with U-Hop (\emph{max}; $d=784$). S-Hop attains
    $N_c\geq500$, a passing sweep endpoint matching that reached on CIFAR-10.}
    \label{fig:appendix_mnist_capacity}
\end{figure}

S-Hop reaches $100\%$ success at $N=500$ on MNIST. U-Hop (\emph{max}) first falls below $90\%$ at $N=20$, then reaches $90\%$ at $30$ and $40$; its largest passing load is therefore $40$. The public \emph{avg} objective instead passes through $500$, with $92.93\pm1.17\%$ success. Figure~\ref{fig:appendix_mnist_capacity} retains the original first-failure annotation of $20$ for U-Hop; Table~\ref{tab:main} reports the largest passing loads.

\begin{figure}[htbp]
    \centering
    \begin{tikzpicture}
        \node[anchor=north west, inner sep=0] (imga) at (0,0)
            {\includegraphics[width=0.48\linewidth]{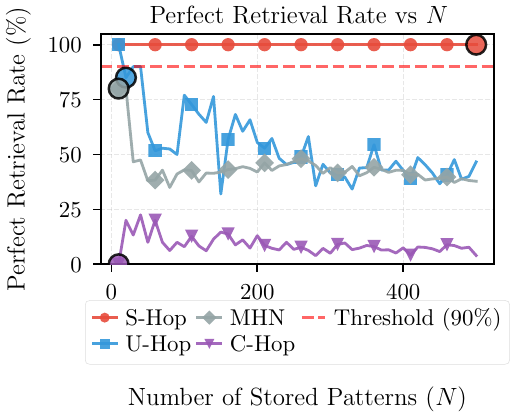}};
        \node[anchor=north west, inner sep=0] (imgb)
            at ([xshift=4pt]imga.north east)
            {\includegraphics[width=0.47\linewidth]{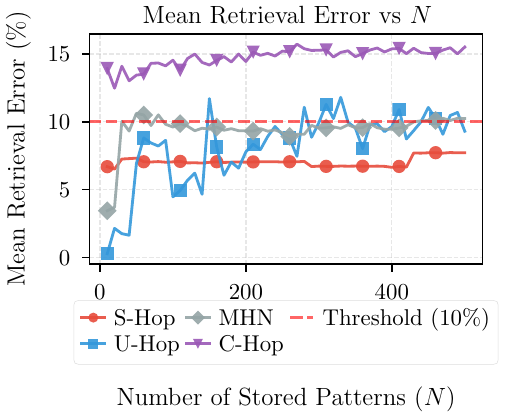}};
        \node[font=\small] at ([xshift=1pt, yshift=1pt]imga.north west) {(a)};
        \node[font=\small] at ([xshift=1pt, yshift=1pt]imgb.north west) {(b)};
    \end{tikzpicture}

    \begin{subfigure}[b]{0.48\linewidth}
        \phantomsubcaption\label{fig:mnist_curves_a}
    \end{subfigure}\hfill
    \begin{subfigure}[b]{0.48\linewidth}
        \phantomsubcaption\label{fig:mnist_curves_b}
    \end{subfigure}
    \caption{\looseness=-1 MNIST retrieval curves with U-Hop (\emph{max}; $d=784$).
    (a)~Perfect retrieval rate under the $\mathrm{MSE}<0.1$ criterion.
    (b)~Mean retrieval error versus the number of stored patterns.
    S-Hop sustains above-threshold retrieval and low error across the full
    range, while baselines saturate early.}
    \label{fig:appendix_mnist_curves}
\end{figure}

Figure~\ref{fig:appendix_mnist_curves} further confirms that S-Hop maintains
consistently higher perfect retrieval rates and lower mean retrieval error
across the full capacity range on MNIST.
The qualitative separation between S-Hop and all baselines is essentially
similar to that on CIFAR-10, supporting retrieval gains across these two datasets.

\FloatBarrier

\subsection{Framework-Level Validation: A Second Kernel Instance}
\label{subsec:appendix-phop}

\looseness=-1 The theoretical framework of Section~\ref{sec:method} and
Appendix~\ref{sec:appendix-theory} defines an entire family of admissible
kernels, of which S-Hop is only one representative.
To examine the behavior of the framework beyond the
tanh-based embedding of Eq.~\eqref{eq:sqrt-kernel}, we evaluate a second
admissible kernel, P-Hop, with coordinate map
\begin{equation}
F_P(\varphi_k) \;=\; \Bigl[\sqrt{\varphi_k^2 + \sin\varphi_k + 0.5},\; \sqrt{\varphi_k^2 + \cos\varphi_k + 0.5}\Bigr].
\label{eq:phop-kernel}
\end{equation}
Both kernels share the same feature-space dimension $D=2d$, the same CCCP
structural identity with constant $C=d$ (Lemma~\ref{lemma:cccp-identity}), and
the same conditional local-convergence and storage analysis (Lemmas~\ref{lemma:local-approx}--\ref{lemma:correct-storage}
and Theorem~\ref{thm:capacity}); their detailed theoretical comparison is given
in Appendix~\ref{subsec:phop-shop-comparison}.

\begin{figure}[htbp]
    \centering
    \includegraphics[width=0.48\linewidth]{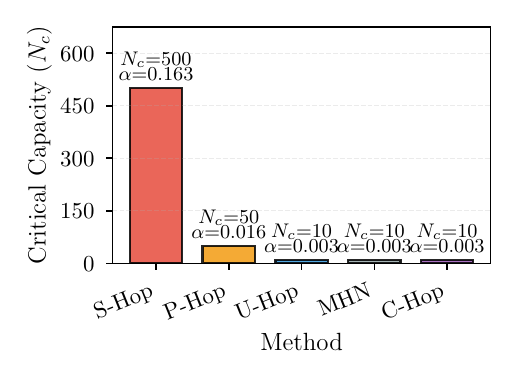}
    \caption{Second admissible kernel on CIFAR-10: P-Hop and S-Hop retrieval loads. P-Hop is plotted at $N=50$, the plotted threshold-crossing load, while S-Hop reaches the $N=500$ endpoint. The bars mark a passing endpoint for S-Hop and a first failure for P-Hop.}
    \label{fig:appendix_phop_capacity}
\end{figure}

\looseness=-1 Figure~\ref{fig:appendix_phop_capacity} reports the retrieval-load
comparison with P-Hop added as a fifth method, under the protocol
of Figure~\ref{fig:cifar-original}.
Appendix~\ref{subsec:phop-shop-comparison} compares the two kernels' local-potential geometry and capacity bounds.

\begin{figure}[htbp]
    \centering
    \begin{tikzpicture}
        \node[anchor=north west, inner sep=0] (imga) at (0,0)
            {\includegraphics[width=0.48\linewidth]{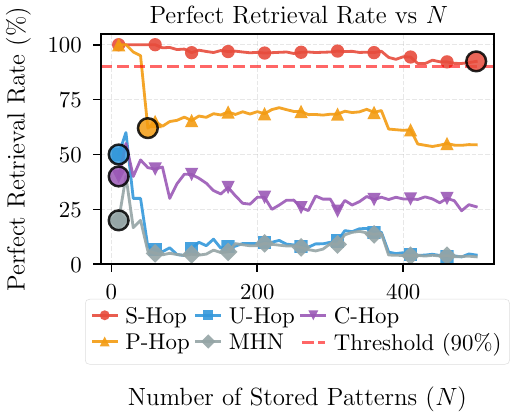}};
        \node[anchor=north west, inner sep=0] (imgb)
            at ([xshift=4pt]imga.north east)
            {\includegraphics[width=0.47\linewidth]{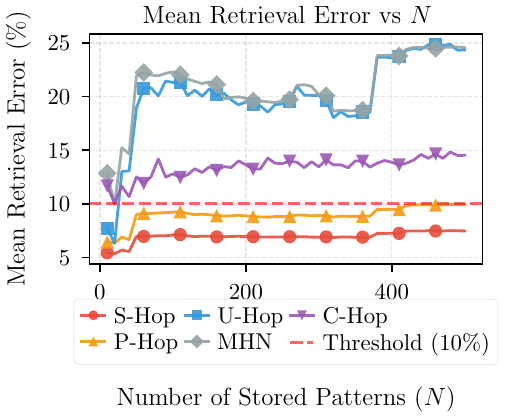}};
        \node[font=\small] at ([xshift=1pt, yshift=1pt]imga.north west) {(a)};
        \node[font=\small] at ([xshift=1pt, yshift=1pt]imgb.north west) {(b)};
    \end{tikzpicture}

    \begin{subfigure}[b]{0.48\linewidth}
        \phantomsubcaption\label{fig:phop_curves_a}
    \end{subfigure}\hfill
    \begin{subfigure}[b]{0.48\linewidth}
        \phantomsubcaption\label{fig:phop_curves_b}
    \end{subfigure}
    \caption{Second admissible kernel on CIFAR-10: P-Hop and S-Hop retrieval curves.
    (a)~Perfect retrieval rate versus the number of stored patterns.
    (b)~Mean retrieval error versus the number of stored patterns.
    P-Hop improves on the baselines, while S-Hop achieves higher retrieval
    success and lower error over the tested loads.}
    \label{fig:appendix_phop_curves}
\end{figure}

Figure~\ref{fig:appendix_phop_curves} compares the two retrieval maps: S-Hop sustains higher retrieval success and lower error over the tested CIFAR-10 loads.

\FloatBarrier

\subsection{Robustness Under the U-Hop SSD Metric}
\label{subsec:appendix-ssd}

The main-text critical capacity is defined via the mean-squared-error criterion
$\mathrm{MSE}<0.1$, which differs from the SSD
metric reported in the U-Hop public codebase~\cite{pmlr-v235-wu24i} and in
related prior work.
To check consistency under aligned thresholds, we
re-evaluate all methods under the U-Hop SSD protocol, defined as
\begin{equation}
\mathrm{SSD}(\psi, \hat\psi)
\;=\; \sum_{k=1}^{d}\bigl(\mathrm{clamp}(\psi_k, 0, 1) - \mathrm{clamp}(\hat\psi_k, 0, 1)\bigr)^2,
\label{eq:ssd-metric}
\end{equation}
where $\psi$ is the target image, $\hat\psi$ the retrieved output, and $k$ the pixel-coordinate index. We evaluate this metric on CIFAR-10 ($d=3072$) with additive Gaussian noise $\sigma=0.3$ and
the same one-step retrieval dynamics as the main protocol.
Both metrics clamp targets and outputs to $[0,1]$. We define SSD-based perfect retrieval as
$\mathrm{SSD}<0.1\cdot d \approx 307$, which by construction is equivalent at
the pattern level to $\mathrm{MSE}<0.1$.
We note that the U-Hop evaluation~\cite{pmlr-v235-wu24i} reports only
mean-error curves without a defined critical capacity; the load-level $90\%$ criterion is introduced here. The main finding of this appendix,
summarized in Figure~\ref{fig:appendix_ssd_capacity}, is consistency with the MSE-based sweep in Section~\ref{sec:results}. The plotted loads follow the original endpoint/crossing convention.

\begin{figure}[htbp]
    \centering
    \includegraphics[width=0.48\linewidth]{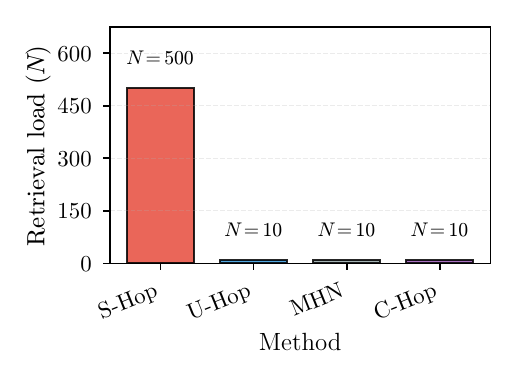}
    \caption{Aligned SSD criterion on CIFAR-10: retrieval-load summary. S-Hop passes through $N=500$ under the aligned threshold $\mathrm{SSD}<0.1d$. Baseline bars mark threshold-crossing loads, following the original plotting convention.}
    \label{fig:appendix_ssd_capacity}
\end{figure}

\begin{figure}[htbp]
    \centering
    \begin{tikzpicture}
        \node[anchor=north west, inner sep=0] (imga) at (0,0)
            {\includegraphics[width=0.47\linewidth]{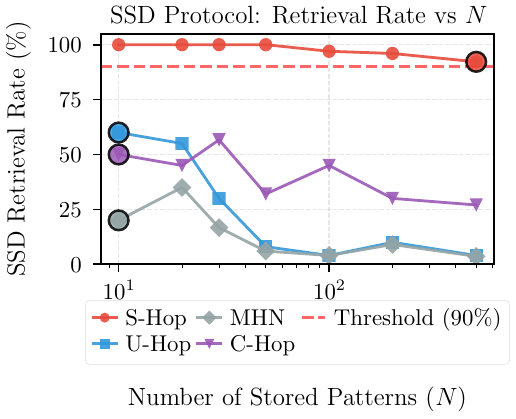}};
        \node[anchor=north west, inner sep=0] (imgb)
            at ([xshift=4pt]imga.north east)
            {\includegraphics[width=0.5\linewidth]{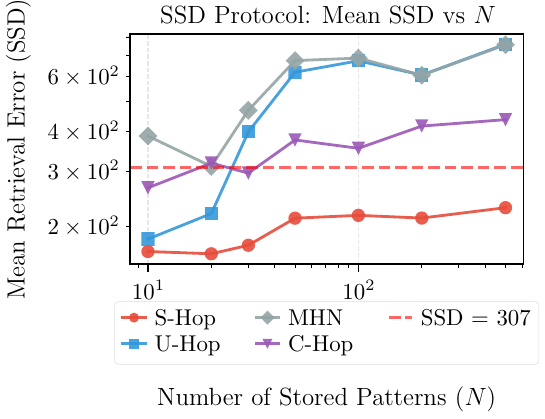}};
        \node[font=\small] at ([xshift=1pt, yshift=3pt]imga.north west) {(a)};
        \node[font=\small] at ([xshift=1pt, yshift=3pt]imgb.north west) {(b)};
    \end{tikzpicture}

    \begin{subfigure}[b]{0.48\linewidth}
        \phantomsubcaption\label{fig:ssd_curves_a}
    \end{subfigure}\hfill
    \begin{subfigure}[b]{0.48\linewidth}
        \phantomsubcaption\label{fig:ssd_curves_b}
    \end{subfigure}
    \caption{Aligned SSD criterion on CIFAR-10: retrieval-rate and error curves.
    (a)~SSD-based perfect retrieval rate versus the number of stored patterns
    under the aligned threshold $\mathrm{SSD}<0.1\cdot d \approx 307$.
    (b)~Mean retrieval SSD versus the number of stored patterns.
    S-Hop remains below the threshold across the full range
    $N\in[10,500]$, while all baselines cross it by $N\approx 30$.}
    \label{fig:appendix_ssd_curves}
\end{figure}

Figure~\ref{fig:appendix_ssd_curves} further illustrates that S-Hop achieves
the lowest mean SSD at every $N$ in the tested range, whereas U-Hop, MHN, and
C-Hop all cross the aligned threshold well before the end of the range.
The qualitative separation mirrors that of the main MSE protocol, providing
a check of the aligned MSE and SSD implementations. Other retrieval criteria are evaluated in Appendix~\ref{subsec:cr-metrics}.

\FloatBarrier

\subsection{Alternative U-Hop Evaluation Protocol (Avg vs. Max)}
\label{subsec:appendix-avg-max}
For each memory set, U-Hop trains its Stage-I projection for $100$ epochs and then fixes it for all queries. With an identity projection, its linear-kernel retrieval reduces to MHN. The alternative \emph{avg} experiment applies PCA spectral denoising ($\sigma^2_{\mathrm{PCA}}=0.01$) to all methods.
The legacy \emph{avg} implementation averages $\log\exp(-t\|u_i-u_j\|^2)$ over memory pairs with $t=2$ and appears in Table~\ref{tab:ablation}. The public \emph{avg} objective instead takes the logarithm after averaging the exponentials; its primary comparison is given in Appendix~\ref{subsec:cr-public-avg}.

\begin{figure}[htbp]
    \centering
    \includegraphics[width=0.48\linewidth]{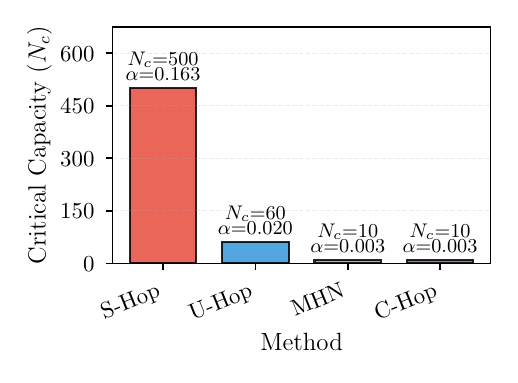}
    \caption{Legacy pairwise-average U-Hop protocol on CIFAR-10: retrieval-load summary. With MSE${}<0.1$, S-Hop passes through $500$; U-Hop fails at the plotted load $60$ ($81.7\%$ success).}
    \label{fig:appendix_avg_capacity}
\end{figure}

Figure~\ref{fig:appendix_avg_capacity} shows that S-Hop continues to achieve
the highest critical capacity under the \emph{avg} protocol.
Under this protocol, U-Hop achieves $81.7\%$ success at $N=60$. S-Hop remains above threshold through the $N=500$ endpoint.

\begin{figure}[htbp]
    \centering
    \begin{tikzpicture}
        \node[anchor=north west, inner sep=0] (imga) at (0,0)
            {\includegraphics[width=0.48\linewidth]{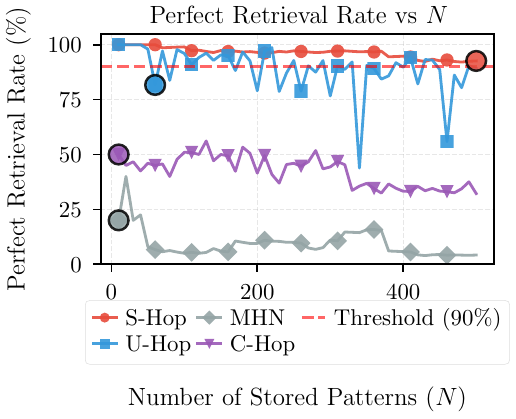}};
        \node[anchor=north west, inner sep=0] (imgb) 
            at ([xshift=4pt]imga.north east)
            {\includegraphics[width=0.47\linewidth]{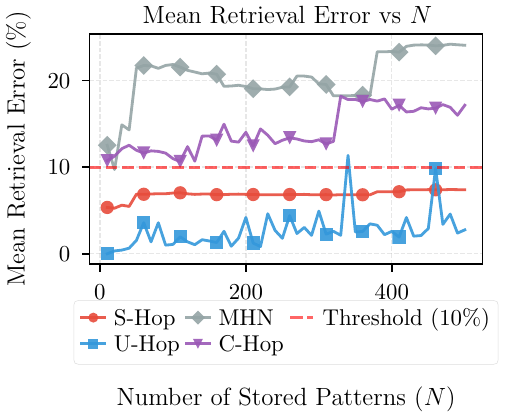}};
        \node[font=\small] at ([xshift=1pt, yshift=1pt]imga.north west) {(a)};
        \node[font=\small] at ([xshift=1pt, yshift=1pt]imgb.north west) {(b)};
    \end{tikzpicture}

    \begin{subfigure}[b]{0.48\linewidth}
        \phantomsubcaption\label{fig:avg_curves_a}
    \end{subfigure}\hfill
    \begin{subfigure}[b]{0.48\linewidth}
        \phantomsubcaption\label{fig:avg_curves_b}
    \end{subfigure}
    \caption{Legacy pairwise-average U-Hop protocol on CIFAR-10: retrieval curves.
    (a)~Perfect retrieval rate versus the number of stored patterns.
    S-Hop maintains above-threshold retrieval up to $N \approx 500$,
    while U-Hop is below $90\%$ at $N=60$.
    (b)~Mean retrieval error versus the number of stored patterns.
    U-Hop attains lower error at small $N$ due to its learned separation
    loss that uniformizes memory distribution under low load~\cite{pmlr-v235-wu24i}.
    Its error rises steeply as $N$ grows;
    S-Hop sustains consistently low error across the full capacity range,
    retaining its advantage at larger loads.}
    \label{fig:appendix_avg_curves}
\end{figure}

Figure~\ref{fig:appendix_avg_curves} further confirms that S-Hop maintains
above-threshold retrieval through the full tested range and lower error at larger loads.
The qualitative separation between S-Hop and all baseline variants closely
matches that observed under the \emph{max} protocol, supporting the main empirical conclusions.

\FloatBarrier

\subsection{Comparison with a Centered Dot-Product Control}
\label{subsec:appendix-orthogonality}

Recent distance-overlap dense associative memory variants~\cite{hoover2024dense, hoover2025epanechnikov}
achieve capacity gains through an energy of the form
\begin{equation}
E_{\text{H}}(x) \;=\; -\frac{1}{\beta}\log\sum_{j=1}^{N}\exp\!\left(-\tfrac{\beta}{2}\,\|x - \xi_j\|^2\right),
\label{eq:hoover-energy}
\end{equation}
where $x$ is the query and $\xi_j$ are memory patterns. The per-pattern norm penalty $-\tfrac{\beta}{2}\|\xi_j\|^2$ reshapes basin
volumes in a manner structurally distinct from the dot-product framework
developed in this work.
Our phase kernel modifies the feature embedding while the construction
of~\cite{hoover2024dense} modifies the overlap form itself, so the two design
axes are orthogonal: any admissible $F$ from our family can in principle be
composed with a distance-based overlap, and vice versa.
Here we compare the centered dot-product variant H-Hop$^{-}$ in Figures~\ref{fig:appendix_orthogonal_capacity}--\ref{fig:appendix_orthogonal_curves}.

\begin{figure}[htbp]
    \centering
    \includegraphics[width=0.48\linewidth]{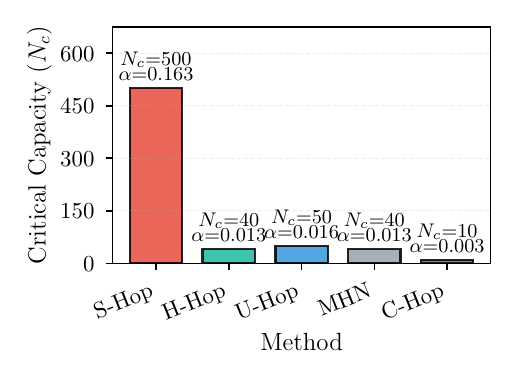}
    \caption{Centered dot-product control on CIFAR-10: H-Hop$^{-}$ and S-Hop retrieval loads. All methods use pattern-space centering $\tilde{\xi}_j = \xi_j - \bar{\xi}$,
    $\tilde{x} = x - \bar{\xi}$, where $\bar{\xi}=N^{-1}\sum_j\xi_j$.
    H-Hop$^{-}$ denotes the dot-product degenerate variant
    of~\cite{hoover2024dense}, obtained by removing the $\|\xi_j\|^2$ penalty
    from Eq.~\eqref{eq:hoover-energy} while retaining the pattern-space
    centering.
    S-Hop passes through $N=500$; the H-Hop$^{-}$ bar marks $N=40$ under the original plotting convention.}
    \label{fig:appendix_orthogonal_capacity}
\end{figure}

Figure~\ref{fig:appendix_orthogonal_capacity} compares H-Hop$^{-}$ with the four main methods under a centered one-step retrieval protocol. S-Hop supports larger tested loads than this dot-product variant. Figure~\ref{fig:appendix_orthogonal_curves} gives the corresponding retrieval and error curves.

\begin{figure}[htbp]
    \centering
    \begin{tikzpicture}
        \node[anchor=north west, inner sep=0] (imga) at (0,0)
            {\includegraphics[width=0.48\linewidth]{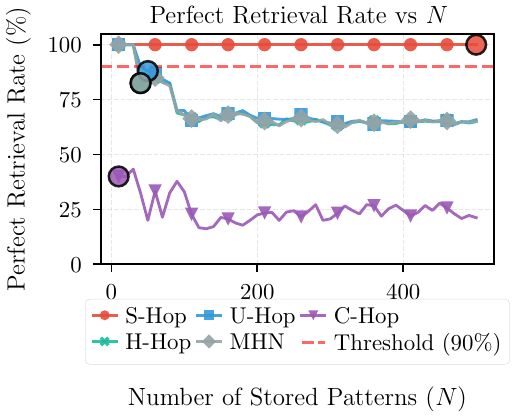}};
        \node[anchor=north west, inner sep=0] (imgb)
            at ([xshift=4pt]imga.north east)
            {\includegraphics[width=0.47\linewidth]{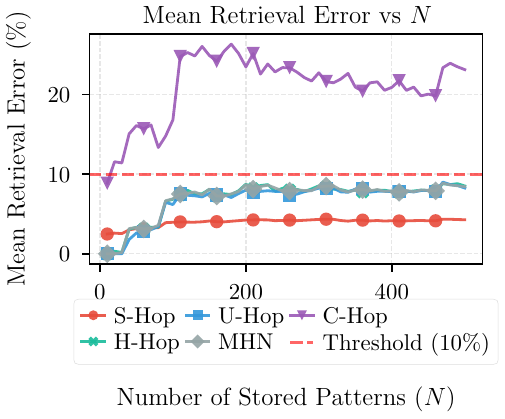}};
        \node[font=\small] at ([xshift=1pt, yshift=1pt]imga.north west) {(a)};
        \node[font=\small] at ([xshift=1pt, yshift=1pt]imgb.north west) {(b)};
    \end{tikzpicture}

    \begin{subfigure}[b]{0.48\linewidth}
        \phantomsubcaption\label{fig:orth_curves_a}
    \end{subfigure}\hfill
    \begin{subfigure}[b]{0.48\linewidth}
        \phantomsubcaption\label{fig:orth_curves_b}
    \end{subfigure}
    \caption{Centered dot-product control on CIFAR-10: H-Hop$^{-}$ and S-Hop retrieval curves.
    (a)~Perfect retrieval rate versus the number of stored patterns.
    (b)~Mean retrieval error versus the number of stored patterns.
    With the norm-penalty removed, H-Hop$^{-}$ degrades toward the linear
    baseline, whereas S-Hop preserves low error throughout the full range,
    showing the retrieval benefit of S-Hop relative to this centered dot-product control.}
    \label{fig:appendix_orthogonal_curves}
\end{figure}

Figure~\ref{fig:appendix_orthogonal_curves} further illustrates that removing
the $\|\xi_j\|^2$ penalty degrades H-Hop$^{-}$ toward the linear baseline,
while S-Hop preserves low error throughout the full range.
This pattern is compatible with the strong performance of the full
distance-based method of~\cite{hoover2024dense}, which benefits from the
norm penalty that H-Hop$^{-}$ discards by construction, and positions our
contribution as a structurally complementary axis: composing an admissible
phase kernel $F$ with a distance-based overlap is a natural direction for
future work.

\FloatBarrier

\subsection{Extended Ablations and Diagnostics}

\begin{figure}[htbp]
    \centering
    \begin{tikzpicture}
        \node[anchor=north west, inner sep=0] (imga) at (0,0)
            {\includegraphics[width=0.48\linewidth]{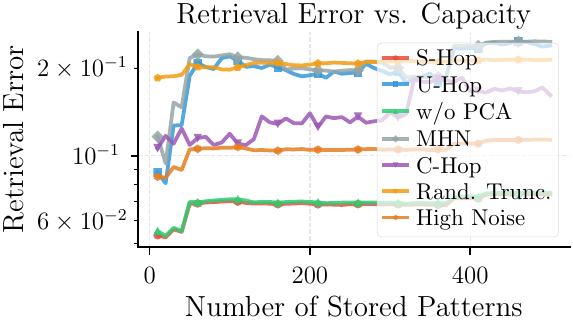}};
        \node[anchor=north west, inner sep=0] (imgb) 
            at ([xshift=4pt]imga.north east)
            {\includegraphics[width=0.482\linewidth]{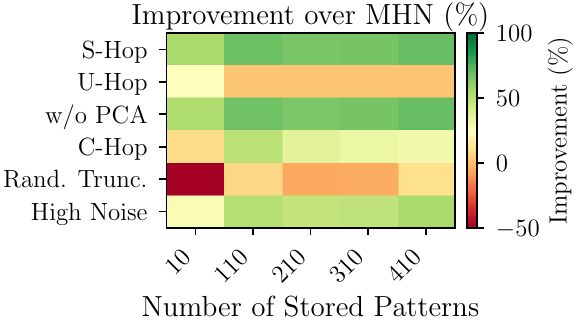}};
        \node[font=\small] at ([xshift=2pt, yshift=-3pt]imga.north west) {(a)};
        \node[font=\small] at ([xshift=2pt, yshift=-3pt]imgb.north west) {(b)};
    \end{tikzpicture}
    \begin{subfigure}[b]{0.48\linewidth}
        \phantomsubcaption\label{fig:ablation_a}
    \end{subfigure}\hfill
    \begin{subfigure}[b]{0.48\linewidth}
        \phantomsubcaption\label{fig:ablation_b}
    \end{subfigure}
    \caption{CIFAR-10 kernel and preprocessing ablations: retrieval error and gains over MHN. 
(a) Retrieval error under increasing memory load for ablated variants 
and stress-test conditions, including removal of the nonlinear kernel, 
PCA preprocessing, cosine-kernel substitution, random truncation, and 
high-noise evaluation ($\sigma = 0.5$). 
(b) Percentage improvement in retrieval accuracy over the MHN baseline 
as a function of the number of stored patterns, isolating 
the contributions of kernel choice and preprocessing.}
    \label{fig:appendix_ablation}
\end{figure}

Figure~\ref{fig:appendix_ablation} aggregates retrieval error versus memory load
for multiple ablated variants and stress tests, together with an
improvement-over-baseline view.
Removing the nonlinear kernel induces the largest degradation across the entire
capacity range, while PCA primarily stabilizes retrieval under moderate noise
and reduces error variation under noise.
The random-truncation and cosine-kernel variants yield higher retrieval error than S-Hop in these comparisons, consistent with Section~\ref{sec:results}.

\FloatBarrier

\subsection{Gating Wrapper: Additional Retrieval Gains}

This section reports the full results for the optional gating wrapper described
in the main text.
The wrapper proposes a one-step MHN, U-Hop, or S-Hop update, mixes it with the current state through a gate, and updates an auxiliary gated state. Figures~\ref{fig:appendix_gating_capacity}--\ref{fig:appendix_gating_curves} show the results.

\begin{figure}[htbp]
    \centering
    \includegraphics[width=0.48\linewidth]{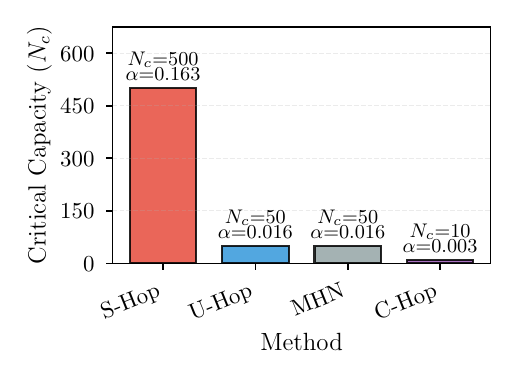}
    \caption{Gating control on CIFAR-10: retrieval-load summary.
    S-Hop passes through $N=500$; baseline bars show the plotted
    threshold-crossing loads.}
    \label{fig:appendix_gating_capacity}
\end{figure}

Figure~\ref{fig:appendix_gating_capacity} compares retrieval-load summaries
for the gated MHN, U-Hop, and S-Hop variants.
S-Hop retains the largest passing load among the gated variants, supporting the use of the phase-kernel update within this wrapper.

\begin{figure}[htbp]
    \centering
    \begin{tikzpicture}
        \node[anchor=north west, inner sep=0] (imga) at (0,0)
            {\includegraphics[width=0.48\linewidth]{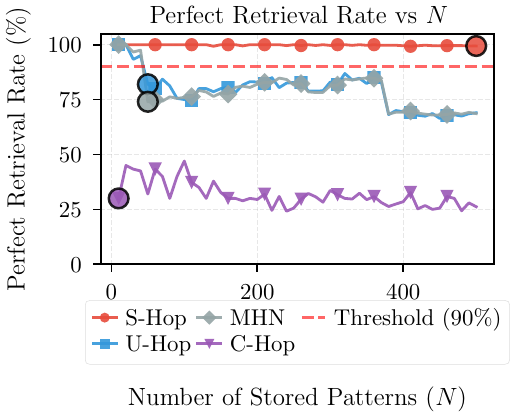}};
        \node[anchor=north west, inner sep=0] (imgb) 
            at ([xshift=4pt]imga.north east)
            {\includegraphics[width=0.47\linewidth]{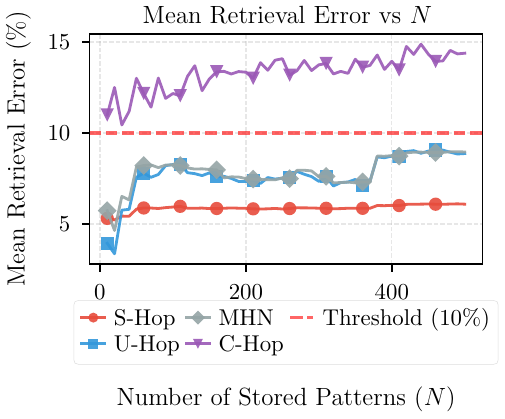}};
        \node[font=\small] at ([xshift=1pt, yshift=1pt]imga.north west) {(a)};
        \node[font=\small] at ([xshift=1pt, yshift=1pt]imgb.north west) {(b)};
    \end{tikzpicture}
    \begin{subfigure}[b]{0.48\linewidth}
        \phantomsubcaption\label{fig:gating_curves_a}
    \end{subfigure}\hfill
    \begin{subfigure}[b]{0.48\linewidth}
        \phantomsubcaption\label{fig:gating_curves_b}
    \end{subfigure}
    \caption{Gating control on CIFAR-10: retrieval-rate and error curves. (a) Perfect retrieval rate and (b) mean retrieval error versus the number of stored patterns.}
    \label{fig:appendix_gating_curves}
\end{figure}

Figure~\ref{fig:appendix_gating_curves} further illustrates that gating improves
retrieval robustness under increasing load, as reflected in both higher perfect
retrieval rates and lower mean MSE.
These curves support combining the phase-kernel update with the gating wrapper for improved retrieval in the tested settings.

\FloatBarrier

\subsection{Additional Generalization: MIL}
\label{subsec:mil}
This section evaluates whether the proposed design choices transfer to a
different learning setting beyond the CIFAR-10 classifier in the main paper.
Figure~\ref{fig:appendix_mil} reports results on a multiple-instance learning
benchmark across random seeds.
The learning curves show the evolution of test accuracy, while the final comparison reports the best test accuracy across the repeated runs.

\begin{figure}[htbp]
    \centering
    \begin{tikzpicture}
        \node[anchor=north west, inner sep=0] (imga) at (0,0)
            {\includegraphics[width=0.48\linewidth]{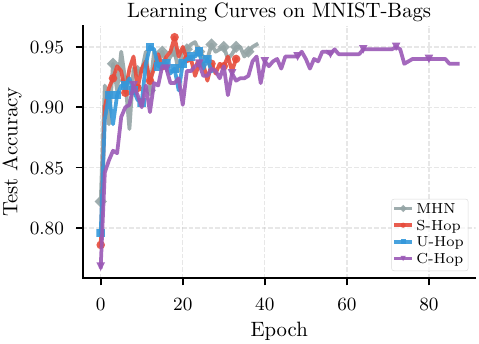}};
        \node[anchor=north west, inner sep=0] (imgb) 
            at ([xshift=4pt]imga.north east)
            {\includegraphics[width=0.475\linewidth]{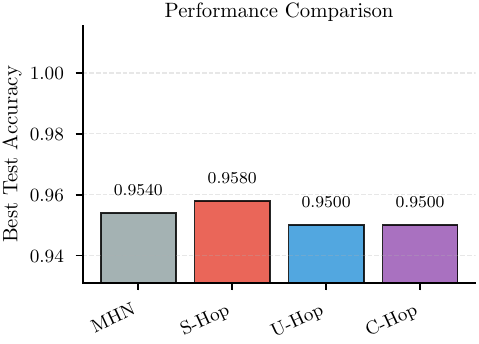}};
        \node[font=\small] at ([xshift=2pt, yshift=-3pt]imga.north west) {(a)};
        \node[font=\small] at ([xshift=2pt, yshift=-3pt]imgb.north west) {(b)};
    \end{tikzpicture}
    \begin{subfigure}[b]{0.48\linewidth}
        \phantomsubcaption\label{fig:mil_a}
    \end{subfigure}\hfill
    \begin{subfigure}[b]{0.48\linewidth}
        \phantomsubcaption\label{fig:mil_b}
    \end{subfigure}
    \caption{Downstream generalization to multiple-instance learning: training and final performance.
    (a) Learning curves with mean $\pm$ standard deviation.
    (b) Best test accuracy over repeated runs.}
    \label{fig:appendix_mil}
\end{figure}

\FloatBarrier

Across alternative U-Hop evaluation modes, extended ablations, the optional
gating wrapper, and MIL generalization, the empirical evidence remains
consistent with improved retrieval under the stated protocols, with further gains from the tested preprocessing and gating choices.

\subsection{Energy-Shaped Classification on CIFAR-10}

The classifier uses six blocks with $4\times4$ patches, embedding width $384$, and six attention heads. The hybrid places two S-Hop blocks before four MHN blocks. Training uses AdamW with learning rate $3\times10^{-4}$, weight decay $0.05$, batch size $128$, and label smoothing $0.1$. The training budget is $100$ epochs, with a $10$-epoch warm-up followed by cosine decay and seed $784$. We record test accuracy at each epoch and the highest observed value.

\begin{figure}[htbp]
    \centering
    \includegraphics[width=0.7\linewidth]{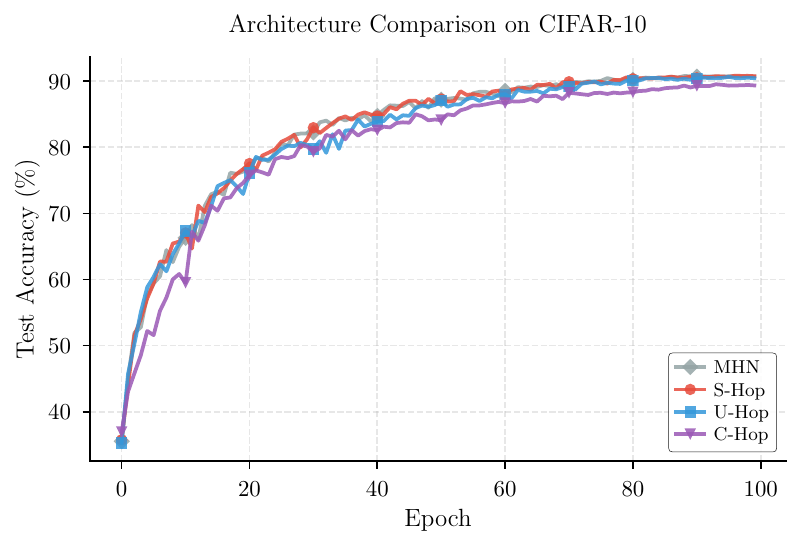}
    \caption{Downstream CIFAR-10 classification: hybrid S-Hop/MHN architectures. Hybrid architectures combining
    two early S-Hop layers with four later MHN layers reach approximately $90\%$ peak test accuracy, about $0.2$ percentage points above the best baseline peak (seed $784$).}
    \label{fig:cifar_main}
\end{figure}

To evaluate the phase-kernel update in supervised learning, we integrate S-Hop
layers into a CIFAR-10 classifier by replacing the attention-based aggregation
modules.
A hybrid architecture with early S-Hop layers followed by standard MHN
layers achieves the highest test accuracy among the compared architectures.
As shown in Figure~\ref{fig:cifar_main}, this hybrid achieves test accuracy
around $90\%$, approximately $0.2$ percentage points above the best baseline peak in the $100$-epoch run.
The test-accuracy trajectories for all four architectures are shown in Figure~\ref{fig:cifar_main}.
Figures~\ref{fig:appendix_gating_curves}--\ref{fig:appendix_mil} cover gated retrieval and MIL.

\FloatBarrier


\section{Additional Retrieval Analyses}
\label{sec:camera-ready-controls}

This appendix examines the effects of feature scale, retrieval criterion, update count, and kernel components on the empirical retrieval load. All experiments use the stated protocol; passing sweep endpoints are reported as lower bounds.

\subsection{Scale, Centering, and Minimum-Margin Controls}
\label{subsec:cr-scale}

\paragraph{Protocol.}
The scale controls use CIFAR-10 at $32\times32$ resolution ($d=3072$), Gaussian query noise $\sigma=0.3$, one retrieval update, base $\beta=1$, no PCA, and MSE${}<0.1$ success with a $90\%$ load-level criterion. U-Hop uses the implemented \emph{max} separation objective, trained for $100$ epochs on the first half of each memory set and then fixed for all queries.

\paragraph{Three-seed comparison.}
The sweep over $N\in\{10,20,50,100,200,300,400,500\}$ uses seeds $42,43,44$ for every method, with a fixed memory set and shared noisy queries within each seed. Noise is added before taking the elementwise absolute value and clipping to $[0,1]$. At $N=500$, normalized logits yield $100\%$ success for S-Hop and $90.60\%$ for each baseline. After logit-scale matching, success is $91.47\pm0.76\%$ for S-Hop, $3.40\pm0.00\%$ for MHN, and $4.40\pm0.69\%$ for U-Hop (mean $\pm$ sample standard deviation over seeds). Figure~\ref{fig:cr-scale-three-seed} shows the corresponding retrieval curves.

\begin{figure}[htbp]
\centering
\includegraphics[width=0.49\linewidth]{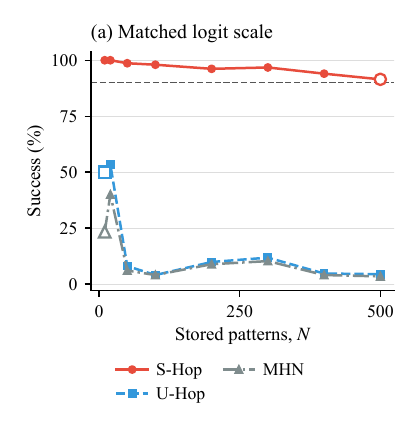}\hfill
\includegraphics[width=0.49\linewidth]{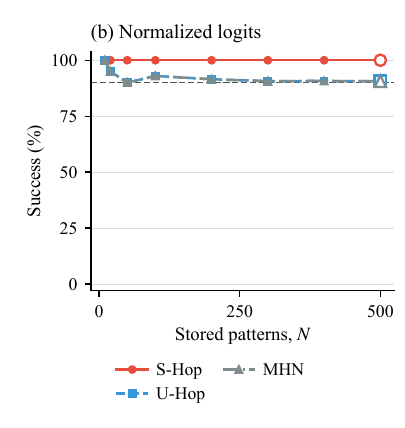}
\caption{Three-seed CIFAR-10 scale controls with the U-Hop \emph{max} objective. Left: matched logit scale; right: normalized logits. Curves average seeds $42,43,44$, with shared memories and noisy queries within each seed. Success means MSE${}<0.1$ after one update at $\sigma=0.3$; dashed lines mark $90\%$. Outlined points mark first failures or passing sweep endpoints.}
\label{fig:cr-scale-three-seed}
\end{figure}

\paragraph{Scale diagnostics.}
Table~\ref{tab:cr-diagnostics} reports statistics from the same three-seed experiment at $N=500$. Logit spread and entropy are averaged over $64$ calibration queries per seed; feature norms are averaged over stored memories. Normalized logits have entropy near $\log 500\approx6.215$ for all methods. The associated retrieval curves are shown in Figure~\ref{fig:cr-scale-three-seed}.

U-Hop's Stage-I optimization uses stochastic gradient descent (SGD) with learning rate $10$. At $N=500$, the mean separation loss over the three seeds decreases from $-0.10573$ at initialization to $-0.12867$ at epoch index $90$; the last objective evaluated during the $100$-epoch training is $-0.13142$. The learned projection is then fixed for retrieval. Figure~\ref{fig:cr-stage1} shows the training trajectory.

\begin{table}[htbp]
\centering
\small
\caption{Feature and logit diagnostics at $N=500$, averaged over three seeds. $H$ is softmax entropy in nats at base $\beta=1$; $\sigma_z$ is the mean within-query logit standard deviation.}
\label{tab:cr-diagnostics}
\begin{tabular}{lrrrrr}
\toprule
Method & Mean feature norm & Raw $\sigma_z$ & Normalized $\sigma_z$ & Raw $H$ & Normalized $H$ \\
\midrule
MHN & $28.84$ & $180.38$ & $0.05683$ & $0.06556$ & $6.21306$ \\
U-Hop & $1508.50$ & $4.8213\times10^5$ & $0.06093$ & $0$ & $6.21283$ \\
S-Hop & $95.63$ & $506.83$ & $0.00282$ & $4.05\times10^{-19}$ & $6.21460$ \\
\bottomrule
\end{tabular}
\end{table}

\paragraph{Extended sweeps.}
Table~\ref{tab:cr-scale} summarizes the extended scans in Figures~\ref{fig:cr-normalization} and~\ref{fig:cr-centering-margin}. MHN and U-Hop results are averaged over seeds $42,43,44$; S-Hop uses seed $42$.

\begin{table}[htbp]
\centering
\caption{CIFAR-10 scale controls: last passing load / first failing load under one-step retrieval with $\sigma=0.3$, base $\beta=1$, MSE${}<0.1$, and $90\%$ success. A bound $\geq$ denotes a passing sweep endpoint; $<10$ denotes failure at the smallest tested load. U-Hop uses the implemented \emph{max} objective. S-Hop uses one seed and the baselines use three.}
\label{tab:cr-scale}
\small
\begin{tabular}{lccc}
\toprule
Control & MHN & U-Hop & S-Hop \\
\midrule
Normalized logits & $870/880$ & $880/890$ & $\geq5000$ \\
Matched logit scale & $<10$ & $<10$ & $790/800$ \\
Input centering & $500/510$ & $\geq1000$ & $\geq1000$ \\
Matched $\beta\Delta_{\min}$ & $<10$ & $<10$ & $790/800$ \\
\bottomrule
\end{tabular}
\end{table}

\paragraph{Normalized logits.}
For method $m$, the normalized score is
\begin{equation}
z_{m,i}^{\mathrm{norm}}(q)
=\frac{x_{m,i}^{\top}q_m}{\|x_{m,i}\|_2\|q_m\|_2}.
\end{equation}
Here $x_{m,i}$ and $q_m$ are method-specific memory and query features, with raw score $z_{m,i}(q)=x_{m,i}^{\top}q_m$. We normalize the scoring vectors and retain each method's value aggregation and state update. The sweep tests $N=10,20,\ldots,1000$ for all methods and extends S-Hop in increments of $100$ to $5000$. The last passing loads are $870$ for MHN and $880$ for U-Hop, with first failures at $880$ and $890$, respectively. S-Hop passes through $5000$, yielding $N_c\geq5000$ and tested-load ratios of approximately $5.75$ and $5.68$, respectively (Figure~\ref{fig:cr-normalization}, left).

\paragraph{Empirical logit-scale matching.}
For each memory set, $64$ calibration queries estimate
\begin{equation}
\sigma_m=\mathbb{E}_q\!\left[\operatorname{Std}_i z_{m,i}(q)\right],
\qquad a_m=\sigma_{\mathrm{MHN}}/\sigma_m.
\end{equation}
The multiplier $a_m$ is applied \emph{inside} softmax; the base update coefficient remains unchanged. At $\beta=1$, this matches the empirical spread of softmax inputs. S-Hop last passes at $790$ and first fails at $800$; both baselines fail at the smallest tested load, $10$. Both baseline loads are left-censored at $10$, so the sweep resolves no finite capacity ratio (Figure~\ref{fig:cr-normalization}, right). Together with the normalized-logit bounds of $5.75$ and $5.68$, these results quantify the controlled comparisons at their measured resolution.

\begin{figure}[htbp]
\centering
\includegraphics[width=0.49\linewidth]{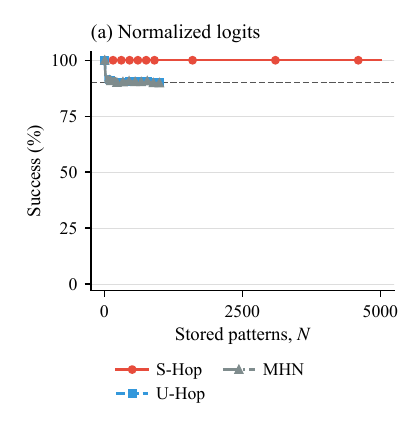}\hfill
\includegraphics[width=0.49\linewidth]{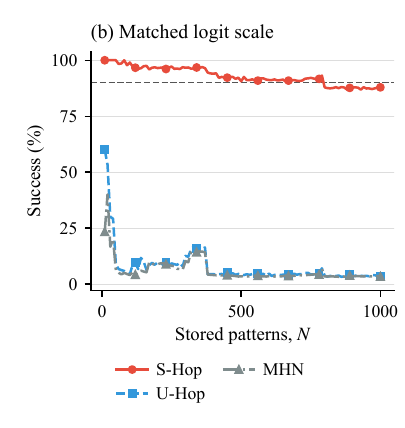}
\caption{Extended CIFAR-10 scale sweeps with the U-Hop \emph{max} objective. Left: normalized logits; right: matched logit scale. Both use one update, $\sigma=0.3$, and MSE${}<0.1$. S-Hop uses seed $42$; baselines average seeds $42,43,44$. Dashed lines mark $90\%$ success. With normalized logits, S-Hop passes through $N=5000$.}
\label{fig:cr-normalization}
\end{figure}

\paragraph{Input centering.}
The centering control uses $F(\psi-\bar\psi)$, subtracting the memory-set mean in input/phase space before embedding and applying the same shift to queries. MHN last passes at $500$ and fails at $510$; S-Hop and U-Hop both pass through $1000$ (Figure~\ref{fig:cr-centering-margin}, left). Full nonlinear-feature centering, $F(\psi)-\overline{F(\psi)}$, is a further extension involving queries, values, and dynamics. Centering only stored keys shifts all logits for each query by the same scalar, preserving the softmax weights.

\paragraph{Signed minimum-margin matching.}
The additional control computes
\begin{equation}
\Delta_{\min,m}=\min_i\min_{j\ne i}
\left(x_{m,i}^{\top}x_{m,i}-x_{m,i}^{\top}x_{m,j}\right)
\end{equation}
and uses a positive multiplier $a_m=\Delta_{\min,\mathrm{MHN}}/\Delta_{\min,m}$ inside softmax to match $\beta\Delta_{\min}$. This calibration requires nonzero margins of the same sign so that the multiplier is positive. With the same one-step MSE criterion and seed allocation as above, S-Hop passes at $790$ and fails at $800$, while both baselines fail at $10$. Figure~\ref{fig:cr-centering-margin} (right) shows retrieval under this signed-margin calibration.

\begin{figure}[htbp]
\centering
\includegraphics[width=0.49\linewidth]{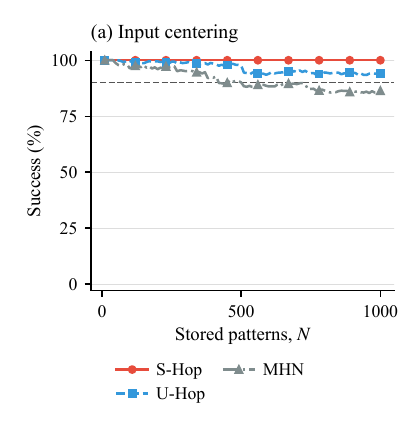}\hfill
\includegraphics[width=0.49\linewidth]{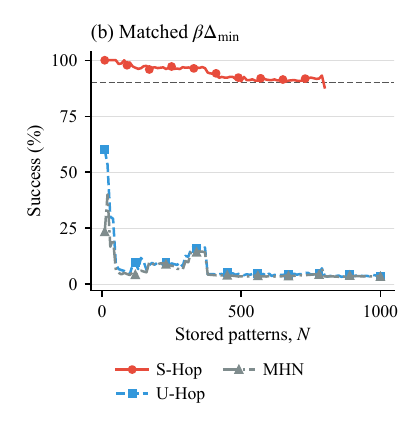}
\caption{CIFAR-10 geometry controls: input centering and signed minimum-margin matching. Left: $F(\psi-\bar\psi)$; right: matched $\beta\Delta_{\min}$. Success means MSE${}<0.1$ after one update at $\sigma=0.3$; dashed lines mark $90\%$. S-Hop uses seed $42$; baselines average seeds $42,43,44$.}
\label{fig:cr-centering-margin}
\end{figure}
\FloatBarrier

\subsection{Retrieval Criteria and Multiple Updates}
\label{subsec:cr-metrics}

\paragraph{MSE thresholds and sweep boundaries.}
A separate CIFAR-10 sweep uses seed $42$, no PCA, one update, and $\sigma=0.3$ to evaluate thresholds $0.01$, $0.03$, $0.05$, and $0.1$. It stops after two consecutive failures of the $90\%$ criterion at MSE${}<0.1$. In this sweep, S-Hop last passes at $750$ and first fails at $800$ under MSE${}<0.1$; the final tested load is $850$. At the three stricter thresholds, S-Hop already fails the $90\%$ success criterion at $N=10$. Figure~\ref{fig:cr-thresholds} shows the success curves for all four thresholds.

\paragraph{Output-based identification.}
For output $\hat\psi$ and target $\psi_i$, cosine-threshold success means
$\cos(\hat\psi,\psi_i)\geq0.95$. Output-cosine top-1 means
$\arg\max_j\cos(\hat\psi,\psi_j)=i$ over the stored memories.
At $N=500$, S-Hop has $91.53\%$ MSE success, $13.4\%$ cosine-threshold success, and $1.53\%$ output-cosine top-1 accuracy. Figure~\ref{fig:cr-cosine} presents the two output-based identification curves; Appendix~\ref{subsec:appendix-ssd} gives the corresponding SSD evaluation.

\begin{figure}[htbp]
\centering
\includegraphics[width=\linewidth]{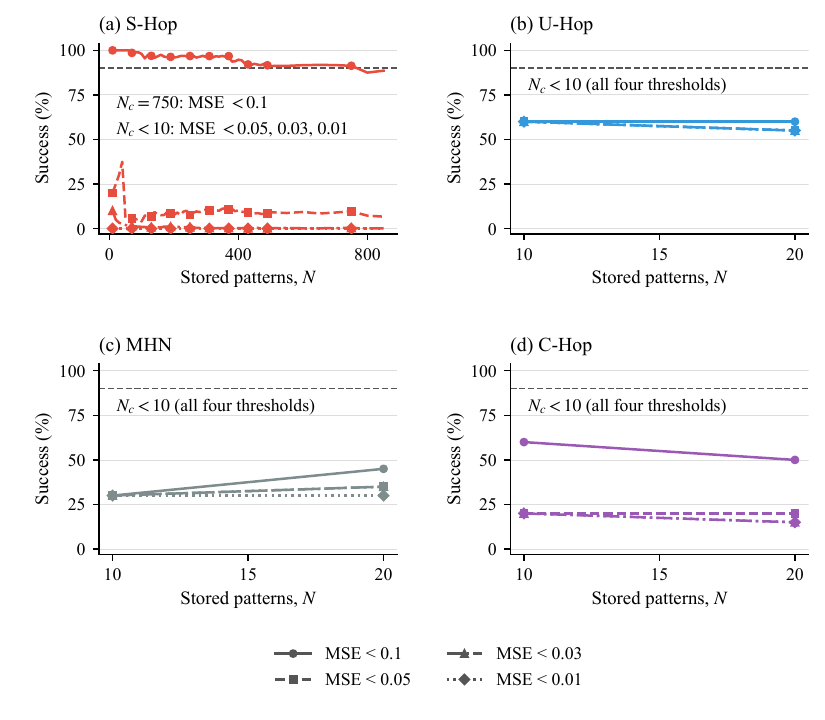}
\caption{CIFAR-10 criterion sensitivity: four MSE thresholds. One-step retrieval uses thresholds $0.01$, $0.03$, $0.05$, and $0.1$, with $\sigma=0.3$, no PCA, and seed $42$. Dashed lines mark $90\%$ success. At threshold $0.1$, S-Hop last passes at $750$ and first fails at $800$; stricter thresholds fail at the smallest tested load, $10$.}
\label{fig:cr-thresholds}
\end{figure}
\begin{figure}[htbp]
\centering
\includegraphics[width=0.49\linewidth]{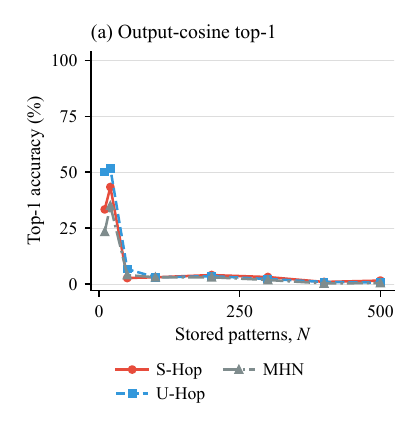}\hfill
\includegraphics[width=0.49\linewidth]{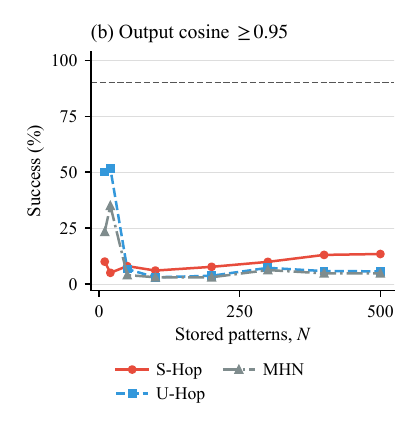}
\caption{CIFAR-10 output identification: cosine top-1 and cosine-threshold criteria. Retrieval uses one update at $\sigma=0.3$. Left: top-1 accuracy selects the stored memory with highest cosine similarity to the retrieved output. Right: success requires output--target cosine similarity $\geq0.95$; the dashed line marks $90\%$ success.}
\label{fig:cr-cosine}
\end{figure}

\paragraph{Multiple updates.}
The common-norm experiment rescales stored inputs to a sphere whose radius is their original mean norm and iterates retrieval at $t\in\{0,1,2,5,10,20\}$. At $N=10$, S-Hop's mean MSE decreases from $0.04177$ after one update to $0.01653$ after $20$, while nearest-memory accuracy changes from $100\%$ to $90\%$. At $N=60$, S-Hop's mean MSE decreases from $0.05026$ after one update to $0.04497$ after two, then rises to $0.05836$ at step $20$. Nearest-memory accuracy drops from $100\%$ to $75\%$ and then to $15\%$ by step $5$, remaining at $15\%$ through step $20$. At step $20$, the successive-update RMSE is $0.00250$. The MSE and identification trajectories are shown in Figures~\ref{fig:cr-multistep-lowload} and~\ref{fig:cr-multistep}, with numerical diagnostics in Table~\ref{tab:cr-multistep}.

\begin{figure}[htbp]
\centering
\includegraphics[width=\linewidth]{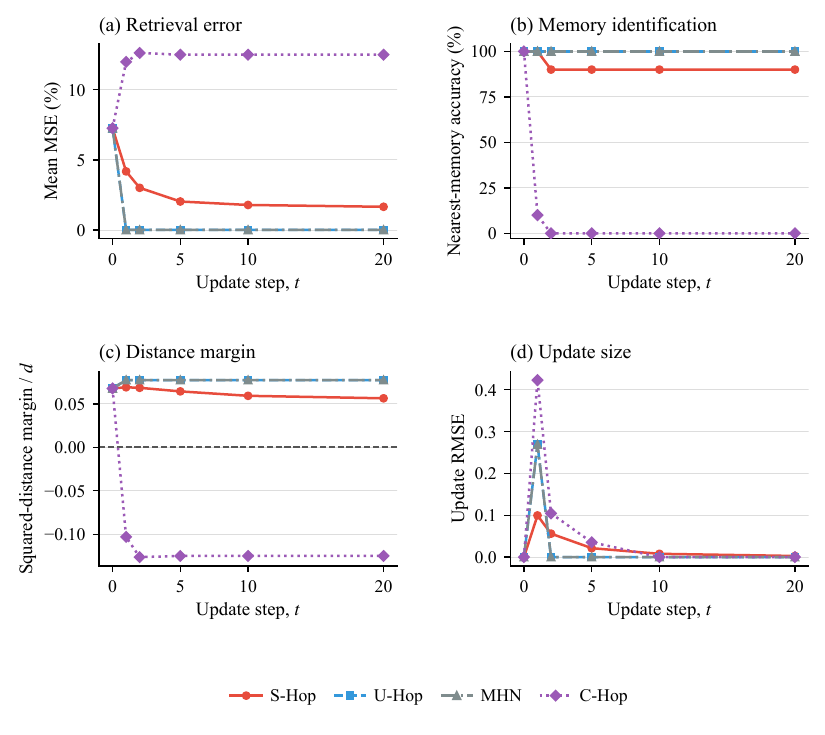}
\caption{Common-norm CIFAR-10 dynamics at low load ($N=10$). Initial query noise is $\sigma=0.3$. Panels show mean MSE, Euclidean nearest-memory accuracy, the nearest-competitor minus target squared-distance margin divided by $d$, and successive-update RMSE at $t\in\{0,1,2,5,10,20\}$.}
\label{fig:cr-multistep-lowload}
\end{figure}

\begin{table}[htbp]
\centering
\caption{S-Hop multiple-update diagnostics at common norm, CIFAR-10, $N=60$ (60 queries). Nearest-memory accuracy uses Euclidean distance. The margin is the mean nearest-competitor squared distance minus target squared distance, divided by $d$.}
\label{tab:cr-multistep}
\small
\begin{tabular}{rrrrr}
\toprule
Step & Mean MSE & Accuracy (\%) & Mean margin & Update RMSE \\
\midrule
0 & 0.07218 & 100 & 0.04342 & --- \\
1 & 0.05026 & 100 & 0.02206 & 0.11148 \\
2 & 0.04497 & 75 & 0.00600 & 0.06249 \\
5 & 0.04729 & 15 & $-0.02372$ & 0.02959 \\
10 & 0.05517 & 15 & $-0.03981$ & 0.01019 \\
20 & 0.05836 & 15 & $-0.04347$ & 0.00250 \\
\bottomrule
\end{tabular}
\end{table}
\begin{figure}[htbp]
\centering
\includegraphics[width=\linewidth]{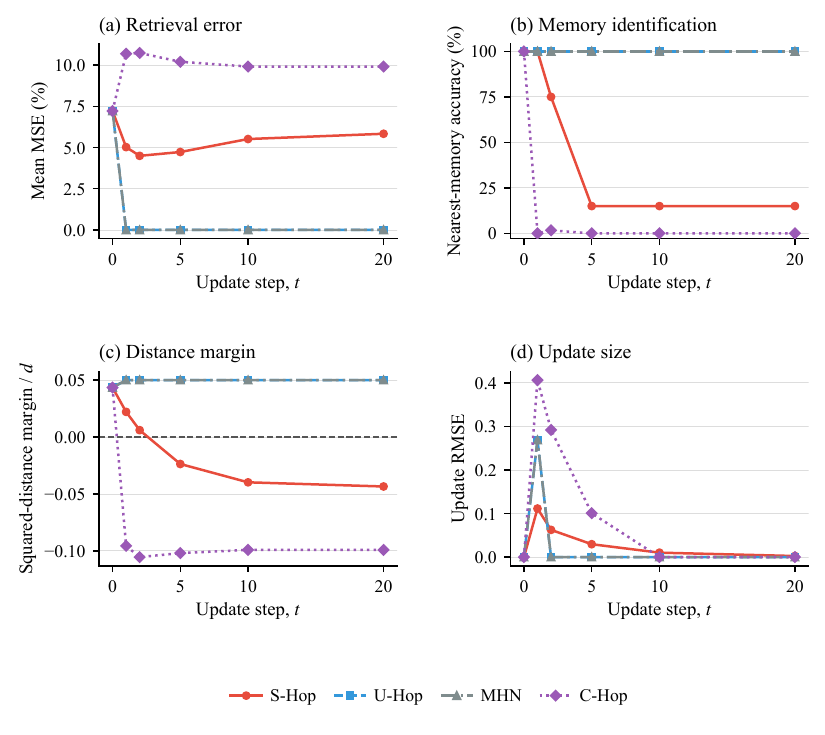}
\caption{Common-norm CIFAR-10 dynamics at higher load ($N=60$). Panels use the four diagnostics in Figure~\ref{fig:cr-multistep-lowload}, with initial noise $\sigma=0.3$. The dashed line marks zero margin; S-Hop values appear in Table~\ref{tab:cr-multistep}.}
\label{fig:cr-multistep}
\end{figure}
\FloatBarrier

\subsection{Similarity, Dynamics, Dimension, and Preprocessing}
\label{subsec:cr-components}

Let $\Psi=[\psi_1,\ldots,\psi_N]$ and $s=\operatorname{softmax}(\beta X^\top F(\varphi))$. The similarity-only variant uses $\varphi^+=\Psi s$. The joint variant uses
$\varphi^+=h^{-1}(J_F(\varphi)^\top F(\Psi s))$, retaining raw-memory aggregation before the Jacobian/local-field transformation. Full S-Hop instead uses the embedded-memory aggregate $Xs$ in Eq.~\eqref{eq:phi-update}.
In the CIFAR-10 sweep through $N=500$, similarity alone falls below $90\%$ success at $N=10$, whereas the joint variant tracks full S-Hop and remains above threshold through $500$. Figure~\ref{fig:cr-components} (left) shows the effect of jointly restoring the Jacobian and inverse-field operations. These arise together from the interaction-energy gradient and the CCCP minimization, respectively; the ablation measures their combined effect.

The dimension control equips U-Hop with a learned $2d\times d$ projection. S-Hop achieves higher retrieval success than this $2d$ linear baseline (Figure~\ref{fig:cr-components}, right). P-Hop and the alternative U-Hop objective are reported in Appendices~\ref{subsec:appendix-phop} and~\ref{subsec:appendix-avg-max} with their respective protocols.

\begin{figure}[htbp]
\centering
\includegraphics[width=0.49\linewidth]{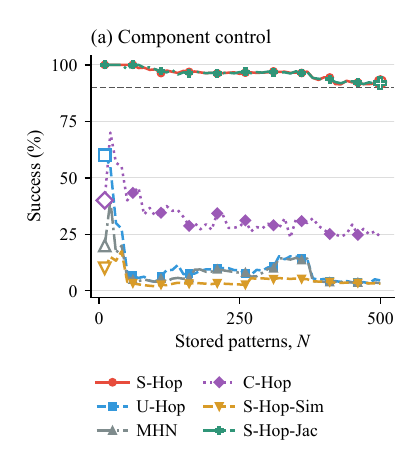}\hfill
\includegraphics[width=0.49\linewidth]{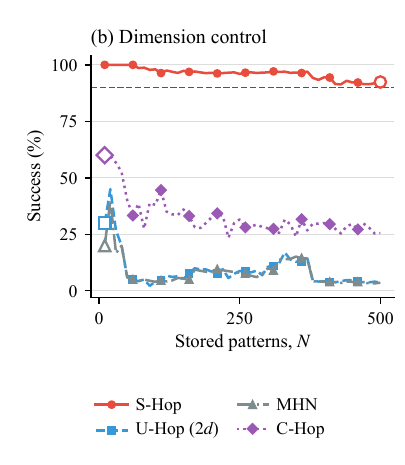}
\caption{CIFAR-10 mechanism controls: state-update components and embedding dimension. Both use one update, $\beta=1$, $\sigma=0.3$, no PCA, and seed $42$. S-Hop-Sim uses phase-kernel scores with raw-memory averaging; S-Hop-Jac restores the Jacobian and inverse local field jointly. Right: U-Hop learns a $2d\times d$ projection. Success means MSE${}<0.1$; dashed lines mark $90\%$. Outlined points mark first failures or passing sweep endpoints.}
\label{fig:cr-components}
\end{figure}

\looseness=-1 Figure~\ref{fig:cr-pca} compares S-Hop with and without PCA alongside kernel, truncation, and higher-noise controls across memory loads. PCA fits the memory covariance, retains $\min(N,d)$ components, scales each centered component by $\lambda_k/(\lambda_k+\sigma^2_{\mathrm{PCA}})$, and reconstructs the memories with their mean restored. Here $\lambda_k$ is the corresponding covariance eigenvalue and $\sigma^2_{\mathrm{PCA}}=0.01$ is the shrinkage regularizer.
\begin{figure}[htbp]
\centering
\includegraphics[width=\linewidth]{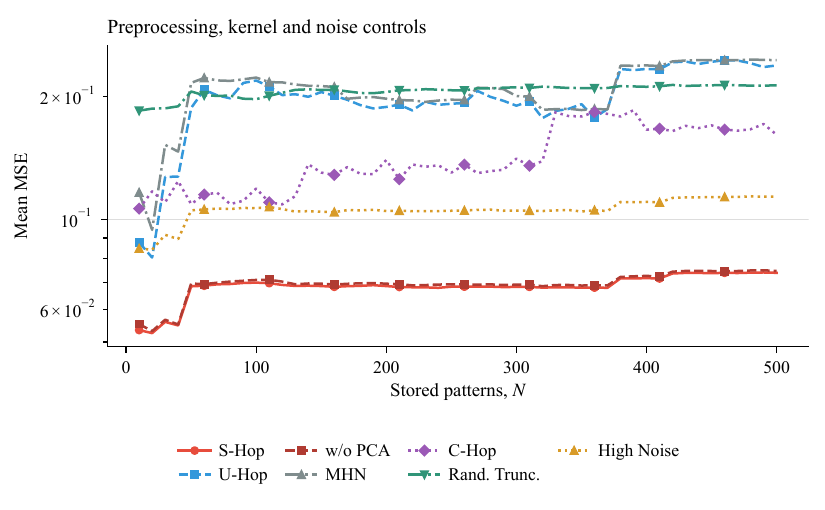}
\caption{CIFAR-10 preprocessing controls: PCA, random truncation, and higher query noise. ``w/o PCA'' removes S-Hop's PCA preprocessing; ``Rand. Trunc.'' retains $1024$ random input coordinates. ``High Noise'' uses $\sigma=0.5$; other curves use $\sigma=0.3$. Retrieval uses one update at $\beta=1$; the error axis is logarithmic.}
\label{fig:cr-pca}
\end{figure}
\FloatBarrier

\subsection{Cross-Dataset Separation and ImageNet}
\label{subsec:cr-datasets}

The cross-dataset diagnostic uses CIFAR-10 and TinyImageNet resized to $32\times32$. It measures median nearest-neighbor squared distance per input dimension and rescales inputs to match this statistic. Figure~\ref{fig:cr-datasets} presents the retrieval curves and distance statistics before and after matching. This run uses a resampled TinyImageNet subset; the fixed-subset control below isolates the resizing effect.

\begin{figure}[htbp]
\centering
\includegraphics[width=0.49\linewidth]{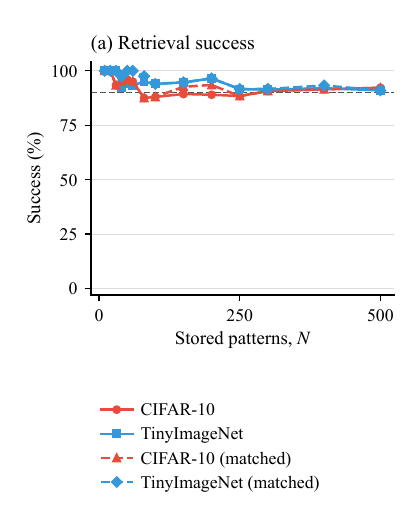}\hfill
\includegraphics[width=0.49\linewidth]{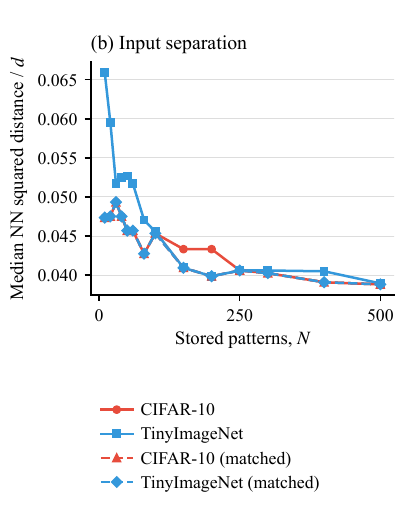}
\caption{Cross-dataset separation control: CIFAR-10 and $32\times32$ TinyImageNet. Curves compare S-Hop before and after input-distance matching. Left: one-step MSE${}<0.1$ success at $\sigma=0.3$, with the $90\%$ threshold dashed. Right: median nearest-neighbor squared distance per input dimension.}
\label{fig:cr-datasets}
\end{figure}

\paragraph{Dimension and retrieval load.}
The original TinyImageNet subset has a retrieval-load threshold of $40$ for S-Hop and $10$ for each baseline, a $4\times$ gain (Table~\ref{tab:tiny-original}). The following controls examine memory geometry at fixed dimension, complementing the theorem's dimension scaling.

\begin{table}[htbp]
\centering
\caption{Original TinyImageNet subset: $64\times64$ images, one update, $\sigma=0.3$, and a $90\%$ MSE${}<0.1$ threshold. Loads mark the first threshold crossing. Public \emph{avg} uses three-seed mean success.}
\label{tab:tiny-original}
\small
\begin{tabular}{lc}
\toprule
Method & Retrieval load \\
\midrule
MHN & $10$ \\
U-Hop (\emph{max}) & $10$ \\
U-Hop (public \emph{avg}) & $10$ \\
C-Hop & $10$ \\
S-Hop & $40$ \\
\bottomrule
\end{tabular}
\end{table}

\begin{figure}[htbp]
\centering
\includegraphics[width=0.48\linewidth]{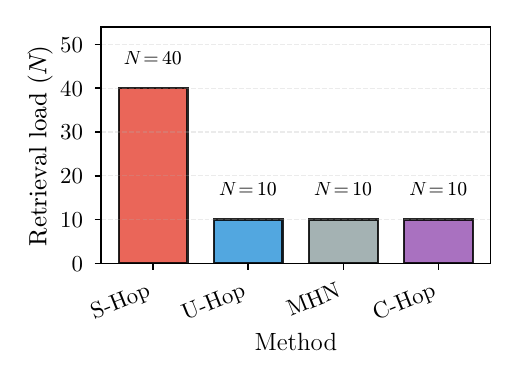}
\caption{Original TinyImageNet retrieval-load comparison. S-Hop reaches $40$, versus $10$ for the baselines, a $4\times$ gain. U-Hop uses the original \emph{max} objective.}
\label{fig:tiny-original-capacity}
\end{figure}

\paragraph{TinyImageNet subset composition and separation.}
The original TinyImageNet experiment uses the first $500$ images in class-sorted order, all from the goldfish class. These natural images have correlated content and varying feature norms, rather than the independent phase coordinates assumed in the capacity theorem. Both image similarity and norm differences affect the signed separation $\Delta_i=\min_{j\ne i}x_i^\top(x_i-x_j)$: at $N=500$, only one of the original memories has positive $\Delta_i$. Keeping these $500$ images fixed and resizing them from $64\times64$ to $32\times32$ raises mean MSE${}<0.1$ success from $50.67\%$ to $54.40\%$ over noise seeds $42,43,44$. This modest improvement motivates examining subset geometry at the original resolution.

The geometry-selected $500$-image subset used in Figure~\ref{fig:tiny-selected} is constructed at $64\times64$ using clean-memory geometry. A candidate pool contains $50$ randomly sampled images from each of the $200$ classes (sampling seed $20261009$). Within the $1000$-image band with the smallest range of squared embedding norms, greedy selection maximizes the minimum of the two directed pairwise margins. The resulting subset spans $183$ classes and has $\Delta_{\min}=80.83$, with all $500$ margins positive. Selection precedes noisy-query evaluation and leaves the images, kernel, and update unchanged.

\looseness=-1 The fixed-subset comparison uses $\beta=1$, $\sigma=0.3$, one update, no PCA or normalization, and the same absolute-value and clipping transforms. Over noise seeds $20261010$, $20261011$, and $20261012$, the original subset achieves $50.53\%$ MSE${}<0.1$ success, mean MSE $0.10121$, and $7.60\%$ nearest-memory accuracy at $N=500$. Under these same seeds, the selected subset achieves $100\%$ MSE-based success and $95.27\pm0.46\%$ nearest-memory accuracy (Figure~\ref{fig:tiny-selected}). One update reduces query MSE from $0.06274$ to $0.02697$ ($57.02\%$); removing the target memory raises MSE to $0.04476$ and lowers nearest-memory accuracy to $69.00\%$. Nearest-memory accuracy identifies the output's closest stored image by Euclidean distance; error bars are sample standard deviations across the three noise seeds on each fixed subset. These results isolate the effect of memory composition at the same input dimension.

\begin{figure}[htbp]
\centering
\includegraphics[width=\linewidth]{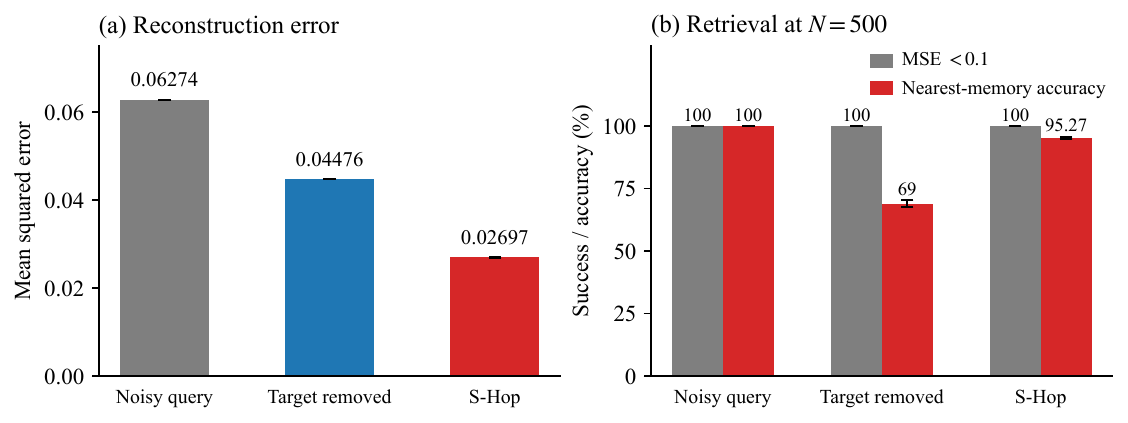}
\caption{\looseness=-1 Geometry-selected TinyImageNet memories ($N=500$, $64\times64$). (a) Reconstruction error. (b) MSE${}<0.1$ success and Euclidean nearest-memory accuracy. Bars show means and sample standard deviations across three noise seeds: one update, $\beta=1$, $\sigma=0.3$, no PCA or normalization.}
\label{fig:tiny-selected}
\end{figure}

\paragraph{ImageNet protocol.}
We further extend the evaluation to ImageNet-1K, using the first $500$ training images in class-sorted directory order, all from the tench class. Images are resized and center-cropped to $64\times64$, or resized to $256$ pixels on the shorter side and center-cropped to $224\times224$, then converted to RGB values in $[0,1]$. The original \emph{max} runs use seed $42$, no PCA, $\beta=1$, one update, and Gaussian noise $\sigma=0.3$ followed by the absolute-value and clipping transform. The sweep uses $N=10,20,\ldots,500$, with one noisy query per memory. U-Hop uses the \emph{max} objective for $100$ epochs; the $224\times224$ run uses its block-diagonal projection implementation with block size $64$. The high-resolution U-Hop sweep was completed using an equivalent closest-pair loss calculation that avoids the original pairwise-gradient memory allocation.

Figure~\ref{fig:cr-imagenet} shows the resulting retrieval curves at $64\times64$ and $224\times224$. S-Hop achieves $91.4\%$ and $98.2\%$ MSE${}<0.1$ success at $N=500$, respectively, giving $N_c\geq500$ in both settings, compared with U-Hop's largest passing load of $10$. Both resolutions include all four methods. At $224\times224$, U-Hop (block) achieves $2.0\%$ success at $N=500$. These MSE-based results at $d=12288$ and $d=150528$, together with the fixed-resolution TinyImageNet subset control, show that input dimension alone does not explain the TinyImageNet degradation; image distribution and separation also affect retrieval.

\begin{figure}[htbp]
\centering
\includegraphics[width=0.49\linewidth]{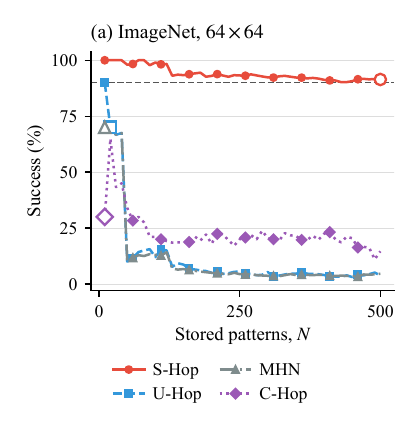}\hfill
\includegraphics[width=0.49\linewidth]{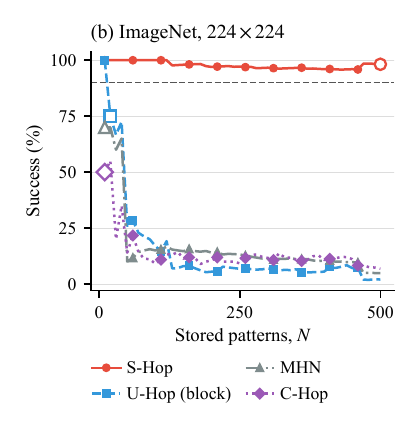}
\caption{ImageNet resolution comparison with U-Hop (\emph{max}): one-step retrieval at $64\times64$ and $224\times224$. Left and right use the same first $500$ training images, seed $42$, no PCA, $\beta=1$, $\sigma=0.3$, and MSE${}<0.1$. Dashed lines mark $90\%$ success; outlined points mark first failures or passing endpoints. S-Hop passes through $N=500$ in both runs. U-Hop uses a dense projection at $64\times64$ and a block-diagonal projection (block size $64$) at $224\times224$.}
\label{fig:cr-imagenet}
\end{figure}
\FloatBarrier

\FloatBarrier
\subsection{Public Average-Separation Baseline and Training Diagnostics}
\label{subsec:cr-public-avg}
The public U-Hop \emph{avg} objective is $\log\operatorname{mean}_{i<j}\exp(-2\|u_i-u_j\|_2^2)$, where $u_i=W\psi_i/\|W\psi_i\|_2$ and $W\in\mathbb{R}^{d\times d}$. We train on the first half of each memory set using SGD with learning rate $10$ for $100$ epochs, starting from a standard-normal projection, and then fix $W$. With $W=I$, the implementation reproduces MHN exactly in the checked runs.

The experiment uses the first $500$ CIFAR-10 training images, no PCA, seeds $42,43,44$, and $N\in\{10,20,50,60,100,200,500\}$. Memories and noisy queries are shared across methods within each seed; the noise transform, $\beta=1$, one-update budget, and MSE${}<0.1$ criterion follow Appendix~\ref{subsec:cr-scale}. Logit matching uses the first $\min(64,N)$ queries for calibration. Runs use float32 on an RTX 5060 Laptop graphics processing unit (GPU), with TensorFloat-32 (TF32) disabled.

At $N=500$, raw-logit success is $92.07\pm0.42\%$ for S-Hop, $48.40\pm3.83\%$ for U-Hop, and $3.33\pm0.12\%$ for MHN. Normalization yields $100\%$, $90.80\%$, and $90.60\%$, respectively, with zero between-seed standard deviation at this load. Logit matching yields $92.07\pm0.42\%$, $48.87\pm3.93\%$, and $3.33\pm0.12\%$. Thus normalization substantially improves the linear baselines, while S-Hop retains higher success in the three comparisons (Figure~\ref{fig:cr-public-avg}). The raw-logit U-Hop curve last passes at $60$ and first fails at $100$ on this grid; S-Hop passes through $500$.

The new MNIST, TinyImageNet, and ImageNet public \emph{avg} sweeps use seeds $42,43,44$ and $N=10,20,\ldots,500$, with fixed memories and independent projection and query generators initialized by multiplying each seed by $100000$ and adding $N$. MNIST uses the first $500$ training images; TinyImageNet and ImageNet use the subsets in Appendix~\ref{subsec:cr-datasets}. ImageNet compares all four methods on shared queries, and repeats U-Hop (\emph{max}) on the same grid. Runs use float32 on an RTX 5090, with TF32 disabled and a block size of $64$ for the $224\times224$ projection. Pairwise squared distances are evaluated through their equivalent Gram-matrix expression. Load bounds use mean success over the three seeds.

\looseness=-1 On MNIST, public \emph{avg} reaches $92.93\pm1.17\%$ success at $N=500$. On the original TinyImageNet subset, it fails at $N=10$ and reaches $3.20\%$ at $N=500$. At $N=500$ on ImageNet, public \emph{avg} reaches $7.00\pm1.06\%$ at $64\times64$ and $4.67\pm0.12\%$ at $224\times224$; S-Hop reaches $91.47\pm0.12\%$ and $98.53\pm0.12\%$, respectively. The warmed-up MNIST timing at $N=300$, averaged over the three seeds, is $1.24$ ms per query after amortizing projection training over $256$ queries, including $0.17$ ms for retrieval.

\raggedbottom
\begingroup
\makeatletter
\setlength{\@fptop}{0pt}
\setlength{\@fpsep}{18pt}
\makeatother
\par\noindent\begin{minipage}{\linewidth}
\captionsetup{type=figure}
\centering
\includegraphics[width=\linewidth,trim=0 5pt 0 3pt,clip]{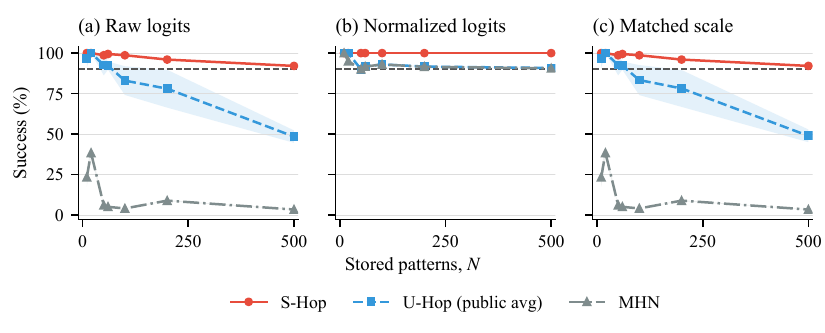}
\caption{CIFAR-10 scale controls with the public U-Hop \emph{avg} objective. (a) Raw, (b) normalized, and (c) matched-scale logits complement the \emph{max} controls in Figure~\ref{fig:cr-scale-three-seed}. Lines and shading show means and sample standard deviations over seeds $42,43,44$. Methods share memories and queries: no PCA, one update, $\beta=1$, and $\sigma=0.3$. Success means MSE${}<0.1$; dashed lines mark $90\%$.}
\label{fig:cr-public-avg}
\end{minipage}\par

\par\noindent\begin{minipage}{\linewidth}
\captionsetup{type=figure}
\centering
\includegraphics[width=\linewidth,trim=0 10pt 0 5pt,clip]{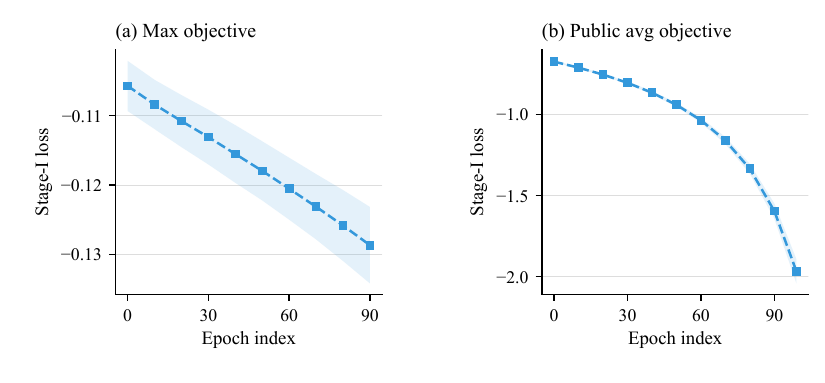}
\caption{U-Hop training diagnostics: \emph{max} and public \emph{avg} objectives at $N=500$. (a) The \emph{max} objective used for Figure~\ref{fig:cr-scale-three-seed}, logged every $10$ epochs through index $90$. (b) The public \emph{avg} objective used for Figure~\ref{fig:cr-public-avg}, including index $99$. Both runs use $100$ epochs and seeds $42,43,44$; lines and bands show means and sample standard deviations on separate loss scales.}
\label{fig:cr-stage1}
\end{minipage}\par
\endgroup

\begin{figure}[htbp]
    \centering
    \begin{tikzpicture}
        \node[anchor=north west, inner sep=0] (imga) at (0,0)
            {\includegraphics[width=0.48\linewidth]{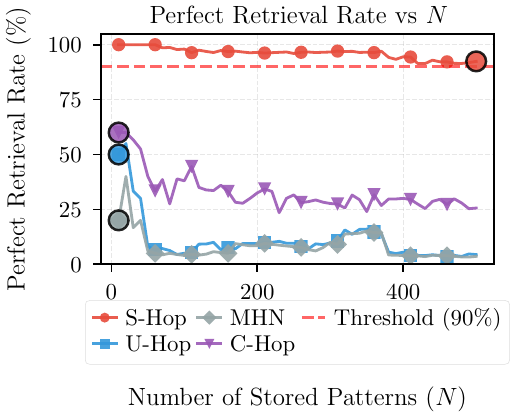}};
        \node[anchor=north west, inner sep=0] (imgb) 
            at ([xshift=4pt]imga.north east)
            {\includegraphics[width=0.47\linewidth]{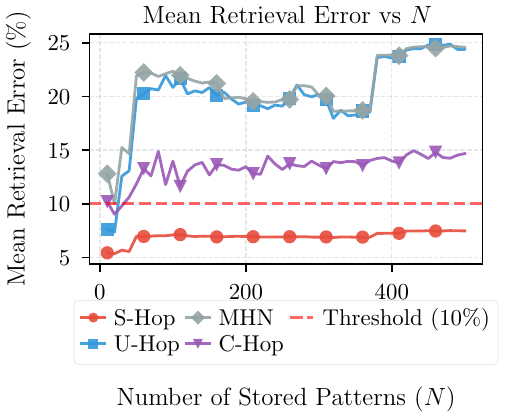}};
        \node[font=\small] at ([xshift=1pt, yshift=1pt]imga.north west) {(a)};
        \node[font=\small] at ([xshift=1pt, yshift=1pt]imgb.north west) {(b)};
    \end{tikzpicture}

    \begin{subfigure}[b]{0.48\linewidth}
        \phantomsubcaption\label{fig:cifar-original-a}
    \end{subfigure}\hfill
    \begin{subfigure}[b]{0.48\linewidth}
        \phantomsubcaption\label{fig:cifar-original-b}
    \end{subfigure}
    \caption{CIFAR-10 retrieval curves under the original one-step protocol with U-Hop (\emph{max}).
    (a) Perfect retrieval rate under the $\mathrm{MSE} < 0.1$ criterion (dots mark the first drop below $90\%$).
    (b) Mean retrieval error as a function of the number of stored patterns.}
    \label{fig:cifar-original}
\end{figure}

\subsection{Empirical Findings and Future Directions}
\label{subsec:cr-limitations}
The scale controls in Figures~\ref{fig:cr-scale-three-seed} and~\ref{fig:cr-normalization} show S-Hop's retrieval advantage under normalization and logit matching. The component comparison in Figure~\ref{fig:cr-components} highlights the role of the combined similarity and state update. Figures~\ref{fig:cr-thresholds}--\ref{fig:cr-multistep} characterize retrieval across success criteria and update counts, and Figures~\ref{fig:cr-datasets}--\ref{fig:cr-imagenet} extend the evaluation across datasets and resolutions.

\paragraph{Future directions.}
First, extending the capacity analysis to correlated memory distributions would help explain the dataset and subset differences. Learned or structured embeddings that improve memory separation offer a related practical direction.
Second, formally incorporating the gating wrapper into the surrogate
Hopfield energy would extend the Lyapunov analysis to the gated dynamics.
Third, composing admissible phase kernels with the distance-based overlaps
of~\cite{hoover2024dense,hoover2025epanechnikov} is a natural next step,
combining changes to memory features with the norm penalty in the retrieval scores.
Fourth, approximating the coupled square-root maps with fast polynomials could improve scalability and reduce computational overhead in applications requiring low latency.

\FloatBarrier
\newpage
\flushbottom
\section*{NeurIPS Paper Checklist}

The checklist is designed to encourage best practices for responsible machine learning research, addressing issues of reproducibility, transparency, research ethics, and societal impact. Do not remove the checklist: {\bf The papers not including the checklist will be desk rejected.} The checklist should follow the references and follow the (optional) supplemental material.  The checklist does NOT count towards the page
limit. 

Please read the checklist guidelines carefully for information on how to answer these questions. For each question in the checklist:
\begin{itemize}
    \item You should answer \answerYes{}, \answerNo{}, or \answerNA{}.
    \item \answerNA{} means either that the question is Not Applicable for that particular paper or the relevant information is Not Available.
    \item Please provide a short (1--2 sentence) justification right after your answer (even for \answerNA). 
\end{itemize}

{\bf The checklist answers are an integral part of your paper submission.} They are visible to the reviewers, area chairs, senior area chairs, and ethics reviewers. You will also be asked to include it (after eventual revisions) with the final version of your paper, and its final version will be published with the paper.

The reviewers of your paper will be asked to use the checklist as one of the factors in their evaluation. While \answerYes{} is generally preferable to \answerNo{}, it is perfectly acceptable to answer \answerNo{} provided a proper justification is given (e.g., error bars are not reported because it would be too computationally expensive'' or ``we were unable to find the license for the dataset we used''). In general, answering \answerNo{} or \answerNA{} is not grounds for rejection. While the questions are phrased in a binary way, we acknowledge that the true answer is often more nuanced, so please just use your best judgment and write a justification to elaborate. All supporting evidence can appear either in the main paper or the supplemental material, provided in appendix. If you answer \answerYes{} to a question, in the justification please point to the section(s) where related material for the question can be found.

IMPORTANT, please:
\begin{itemize}
    \item {\bf Delete this instruction block, but keep the section heading ``NeurIPS Paper Checklist"},
    \item  {\bf Keep the checklist subsection headings, questions/answers and guidelines below.}
    \item {\bf Do not modify the questions and only use the provided macros for your answers}.
\end{itemize}


\begin{enumerate}

\item {\bf Claims}
    \item[] Question: Do the main claims made in the abstract and introduction accurately reflect the paper's contributions and scope?
    \item[] Answer: \answerYes{} 
    \item[] Justification: {The abstract and introduction explicitly state the theoretical guarantees (exponential storage capacity, fixed-point convergence) and empirical evaluations (superior critical capacity and retrieval robustness), which are fully addressed in Sections \ref{sec:method} and \ref{sec:results}.}

\item {\bf Limitations}
    \item[] Question: Does the paper discuss the limitations of the work performed by the authors?
    \item[] Answer: \answerYes{} 
    \item[] Justification: {Section~\ref{sec:limitations} states the theoretical and empirical limitations, including the sampling and margin assumptions, criterion and scale dependence, and joint component ablations. Appendix~\ref{sec:camera-ready-controls} details the controls.}

\item {\bf Theory assumptions and proofs}
    \item[] Question: For each theoretical result, does the paper provide the full set of assumptions and a complete (and correct) proof?
    \item[] Answer: \answerYes{} 
    \item[] Justification: {Complete formal proofs, including all regularity assumptions and the contraction argument, are provided in Appendix \ref{sec:appendix-theory} and \ref{app:capacity-analysis}.}

    \item {\bf Experimental result reproducibility}
    \item[] Question: Does the paper fully disclose all the information needed to reproduce the main experimental results of the paper to the extent that it affects the main claims and/or conclusions of the paper (regardless of whether the code and data are provided or not)?
    \item[] Answer: \answerYes{} 
    \item[] Justification: {Section~\ref{sec:results} and Appendix~\ref{sec:appendix-experiment} give the retrieval equations and baseline protocols. Appendix~\ref{sec:camera-ready-controls} specifies the additional control settings, seed counts, and evaluation criteria; code is supplied with the submission.}

\item {\bf Open access to data and code}
    \item[] Question: Does the paper provide open access to the data and code, with sufficient instructions to faithfully reproduce the main experimental results, as described in supplemental material?
    \item[] Answer: \answerYes{} 
    \item[] Justification: {Anonymized source code, together with a README, a dependency list, and configuration files reproducing the main capacity, ablation, and MIL results in Section~\ref{sec:results} and Appendix~\ref{sec:appendix-experiment}, is provided in the supplemental material.}

\item {\bf Experimental setting/details}
    \item[] Question: Does the paper specify all the training and test details (e.g., data splits, hyperparameters, how they were chosen, type of optimizer) necessary to understand the results?
    \item[] Answer: \answerYes{} 
    \item[] Justification: {Experimental settings, including datasets, the alternative U-Hop evaluation protocol (max vs. avg), and retrieval criteria, are described in Section \ref{sec:results} and Appendix \ref{sec:appendix-experiment}.}

\item {\bf Experiment statistical significance}
    \item[] Question: Does the paper report error bars suitably and correctly defined or other appropriate information about the statistical significance of the experiments?
    \item[] Answer: \answerYes{} 
    \item[] Justification: {The MIL experiment reports mean and standard deviation across repeated runs. Appendix~\ref{subsec:cr-scale} reports mean and sample standard deviation over three seeds for the scale controls through $N=500$ (Figure~\ref{fig:cr-scale-three-seed}); the extended scans specify their method-specific seed counts. Appendix~\ref{subsec:cr-public-avg} adds a three-seed comparison using the public average-separation objective.}

\item {\bf Experiments compute resources}
    \item[] Question: For each experiment, does the paper provide sufficient information on the computer resources (type of compute workers, memory, time of execution) needed to reproduce the experiments?
    \item[] Answer: \answerYes{} 
    \item[] Justification: {The original capacity, retrieval, ablation, and MIL experiments were run on a workstation equipped with two NVIDIA Tesla T4 GPUs (16 GB each). The single-T4 timing in Table~\ref{tab:main} and Figure~\ref{fig:efficiency} includes training amortized over 256 queries at $N=300$, as described in Section~\ref{app:efficiency}. Compute for the original experiment set, including preliminary tuning, was estimated at fewer than 50 GPU-hours; this estimate excludes the extended control sweeps. The public-average comparison in Appendix~\ref{subsec:cr-public-avg} uses an NVIDIA RTX 5060 Laptop GPU (8 GB).}
    
\item {\bf Code of ethics}
    \item[] Question: Does the research conducted in the paper conform, in every respect, with the NeurIPS Code of Ethics \url{https://neurips.cc/public/EthicsGuidelines}?
    \item[] Answer: \answerYes{} 
    \item[] Justification: {The research involves foundational theoretical and empirical evaluations in associative memory and strictly conforms to the NeurIPS Code of Ethics.}

\item {\bf Broader impacts}
    \item[] Question: Does the paper discuss both potential positive societal impacts and negative societal impacts of the work performed?
    \item[] Answer: \answerYes{} 
    \item[] Justification: {\emph{Positive impacts.} The nonlinear phase-kernel framework provides a principled route to associative memory and attention modules with provably exponential capacity and well-separated attractor basins, supporting more reliable long-context retrieval and reduced memory confusion in attention-based systems. \emph{Potential negative impacts.} Because dense associative memory is mathematically equivalent to self-attention~\cite{ramsauer2021hopfield}, capacity gains at this level can in principle propagate into Transformer-based applications with dual-use concerns (e.g., large-scale surveillance or automated disinformation). Our work is foundational and not tied to any specific deployment; we release no pretrained weights or scraped data, and the accompanying code is restricted to standard associative-memory benchmarks (MNIST, CIFAR-10, TinyImageNet, MNIST-Bags), which limits direct misuse.}
    
\item {\bf Safeguards}
    \item[] Question: Does the paper describe safeguards that have been put in place for responsible release of data or models that have a high risk for misuse (e.g., pre-trained language models, image generators, or scraped datasets)?
    \item[] Answer: \answerNA{} 
    \item[] Justification: {The work focuses on fundamental associative memory algorithms and does not involve releasing high-risk models, generators, or scraped datasets.}

\item {\bf Licenses for existing assets}
    \item[] Question: Are the creators or original owners of assets (e.g., code, data, models), used in the paper, properly credited and are the license and terms of use explicitly mentioned and properly respected?
    \item[] Answer: \answerYes{} 
    \item[] Justification: {Standard benchmark datasets (MNIST, CIFAR-10, TinyImageNet, MNIST-Bags) and the public U-Hop codebase used in our experiments are properly cited in the text and references in Section~\ref{sec:results} and Appendix~\ref{sec:appendix-experiment}.}

\item {\bf New assets}
    \item[] Question: Are new assets introduced in the paper well documented and is the documentation provided alongside the assets?
    \item[] Answer: \answerYes{} 
    \item[] Justification: {We release anonymized source code for the S-Hop and P-Hop kernels and the capacity-evaluation pipeline as supplemental material. The submission includes a README documenting environment setup and exact commands to reproduce Tables~\ref{tab:main}--\ref{tab:ablation} and Figures~\ref{fig:capacity_main}--\ref{fig:cifar_main}, a \texttt{requirements.txt} pinning dependency versions, and configuration files for each evaluation protocol (max/avg, MSE/SSD, with/without PCA, with/without gating). The code is released under the MIT license.}

\item {\bf Crowdsourcing and research with human subjects}
    \item[] Question: For crowdsourcing experiments and research with human subjects, does the paper include the full text of instructions given to participants and screenshots, if applicable, as well as details about compensation (if any)? 
    \item[] Answer: \answerNA{} 
    \item[] Justification: {The research does not involve crowdsourcing or human subjects.}

\item {\bf Institutional review board (IRB) approvals or equivalent for research with human subjects}
    \item[] Question: Does the paper describe potential risks incurred by study participants, whether such risks were disclosed to the subjects, and whether Institutional Review Board (IRB) approvals (or an equivalent approval/review based on the requirements of your country or institution) were obtained?
    \item[] Answer: \answerNA{} 
    \item[] Justification: {The research does not involve human subjects.}

\item {\bf Declaration of LLM usage}
    \item[] Question: Does the paper describe the usage of LLMs if it is an important, original, or non-standard component of the core methods in this research? Note that if the LLM is used only for writing, editing, or formatting purposes and does \emph{not} impact the core methodology, scientific rigor, or originality of the research, declaration is not required.
    \item[] Answer: \answerNA{} 
    \item[] Justification: {The core method development and theoretical derivations in this research do not involve LLMs as any important, original, or non-standard components.}

\end{enumerate}

\end{document}